\documentclass[onecolumn]{IEEEtran}

\usepackage[utf8]{inputenc}
\usepackage{setspace}
\usepackage{color}
\usepackage{soul}
\usepackage{booktabs}
\usepackage{amsmath}
\usepackage{amssymb}
\usepackage{amsthm}
\usepackage{amsfonts}
\usepackage{mathtools}
\usepackage{stmaryrd}
\usepackage{bbm}
\usepackage[mathscr]{euscript}
\usepackage{mathrsfs}

\usepackage{graphicx}
\usepackage{epstopdf}
\usepackage{float}
\usepackage{caption}
\usepackage{subcaption} 

\usepackage{algorithm}
\usepackage{algpseudocode}

\usepackage{multirow}

\usepackage[round,authoryear]{natbib}
\usepackage[colorlinks=true,citecolor=blue,linkcolor=blue]{hyperref}

\newtheorem{theorem}{Theorem}

\newtheorem{remark}{Remark}
\newtheorem{lemma}{Lemma}

\newtheorem{defn}{Definition}

\newtheorem{example}{Example}
\newcommand{\remove}[1]{}

\title{\Large \bf Cascading Bandits with Noisy Observations: \\ Algorithms and Regret Bounds}

\author{
  Yashaswini Kaluguri$^*$
  \and
  \hspace{2cm} 
  Nikhil Karamchandani$^+$
}

\date{\vspace{-5ex}}
\begin{document}

\maketitle
\let\thefootnote\relax
\footnotetext{$^*$Department of Electrical Engineering, Indian Institute of Technology Bombay, Mumbai, India (\texttt{yashkaluguri2005@gmail.com}).}

\let\thefootnote\relax
\footnotetext{$^+$Department of Electrical Engineering, Indian Institute of Technology Bombay, Mumbai, India (\texttt{nikhilk@ee.iitb.ac.in}).}
\begin{abstract}
We propose a novel stochastic cascading bandit setup where probing an ordered shortlist of $K$ out of $L$ arms incurs position-dependent costs, and the learner observes only noisy feedback regarding each arm's true output. We identify a sub-class of noise models for which the optimal shortlist can be found by sorting arms according to a per-arm score: such a score exists for both selection and ordering when the channel has no false positives, and for ordering alone (but not selection) when measurements are costless; however in general neither selection nor ordering is separable once noise and cost are both present. For the case where both selection and ordering scores exist, we provide a UCB-based algorithm with $O((L-K)\log T)$ instance-dependent regret over a time horizon $T$, which matches an algorithm-agnostic lower  bound (up to constant factors). For the more challenging scenario where scores don't exist, we develop a new algorithmic template, DP-UCB, that embeds optimism directly inside an exact  dynamic program, yielding $O\big(LK\log T\big)$ regret. Finally, we validate our theoretical results through extensive numerical evaluations and simulations over both synthetic and real-world datasets.
\end{abstract}
\section{Introduction}
Stochastic multi-armed bandits (MABs) \cite{lattimore2020bandit, cesa2006prediction, auer2002finite} are a canonical framework for studying sequential decision-making, with applications ranging from clinical trials \cite{ThananjeyanKSGG21} and crowdsourcing \cite{DidwaniaNH22} to recommendation systems \cite{yao2022learning} and hyperparameter tuning \cite{li2018hyperband}. 

Cascading bandits \cite{kveton2015cascading, combes2015learning} were introduced to model online learning-to-rank (OLTR) problems, which capture user interactions with ranked lists in applications such as web search and recommendation systems. In this model, there is a catalog of $L$ items (arms), each with their own \textit{attraction} level for the user. In each round, the learner presents an ordered shortlist of $K$ items to the user, who then proceeds to sequentially examine the items in the list, and selects the first attractive item. For example, in the case of web search, this might be manifested as a click. Unlike standard stochastic bandits, the feedback to the leaner is not only that the selected item is attractive, but also that the items preceding it in the list were found unattractive, while items following it remain unobserved. This intricate feedback structure makes this problem interesting as well as challenging, and has therefore led to a significant amount of literature studying various online-learning metrics such as \textit{regret} \cite{kveton2015cascading, vial2022minimax, cheung2019thompson} and  \textit{sample complexity} for best arm identification \cite{zhong2020best}, as well as different aspects such as observations costs \cite{zhou2018cost, cheng2022cascading}, diversity of rankings \cite{combes2015learning, magureanu2017online}, multiple clicks \cite{katariya2016dcm}, and variable list sizes \cite{santara2022learning}. 

In this work, we study a variant of cascading bandits with noisy (and costly) observations, wherein the feedback obtained from each probe provides only a noisy indication of the arm's true output to the learner. Such noisy feedback arises naturally in cascades of machine-learning models \cite{chen2023frugalgpt, enomoro2021learning, jitkrittum2023confidence}, where the output of each model may be assessed by a lightweight scoring or confidence module before deciding whether to accept the output or proceed to another model. For cascading bandits, we find that observation noise (especially in conjunction with observation costs) fundamentally affects both which arms should be selected and how they should be ordered. Our goal is to characterize the structure of the optimal shortlist under different noise models, and to develop online learning algorithms that efficiently learn this shortlist from noisy cascading feedback. 

Costly and imperfect observations have been considered in several variants of the multi-armed bandit problem. Prior work has studied settings in which acquiring observations incurs a cost~\cite{seldin2014prediction,yun2018multi,elumar2025multi}, as well as settings where the learner receives stochastically noisy or corrupted observations~\cite{gajane2018corrupt,agrawal2024crimed,mayekar2023communication}. \cite{lindstahl2020predictive} allow the learner to acquire costly and potentially noisy measurements of arm rewards before making its action decision. A separate line of work considers adversarial corruptions of otherwise stochastic bandit feedback~\cite{lykouris2018stochastic,gupta2019better}, with the goal of designing learning algorithms whose performance degrades gracefully with the amount of adversarial corruption. Such adversarial corruptions have also been studied in cascading bandits and OLTR~\cite{zuo2023adversarial,wang2024adversarial,chattopadhyay2025observation}. In our setting, stochastic observation noise is intrinsic to the cascade and affects both which items should be included in the shortlist and how the selected items should be ordered. \cite{prakash2025cascading} consider a closely related setting with stochastic observation errors, but focus on the special case of \(K=L\).

\subsection{Our Contributions}
\paragraph{Noise-model classification.}
We propose a partition of cascading bandit instances based on their underlying noise model, which determines whether the optimal shortlist selection and ordering can be characterized by per-arm `scores.' This yields a score-based subclass (\textit{Case 1 and Shared-Channel regime} \textit{}), where the optimal shortlist can be obtained by ranking arms according to an explicit score, and a more general regime (\textit{Case 2}), where such a score-based characterization does not exist.
\remove{
We formalize cascading environments through the lens of their underlying noise structures, distinguishing between a score-based subclass (\textbf{Case 1}), where an explicit ranking score exists, and the more challenging general regime (\textbf{Case 2}), where no such score can be defined. Having established this notion of a score in Section~\ref{Oracle_section}, we confirm the non-existence of explicit ranking metrics in the general regime through numerical simulations.
}
\remove{
\paragraph{Algorithm Design:}
We develop optimistic, UCB-based online learning algorithms for both noise regimes. For the class which supports scores ( \textbf{Case 1} ), we propose \textsc{Top-$K$-$P(Y=1)$-UCB} (Algorithm~\ref{alg:ucb_partial_exposure_case1}) to exploit explicit ranking scores. For the more complex  regime ( \textbf{Case 2} ), where no explicit ranking score exists, identifying the reward-maximizing shortlist $\pi^*$ demands solving a non-trivial combinatorial optimization problem. Our central insight is that this optimal sequence can be computed efficiently via dynamic programming, forming the backbone of our main algorithm, \textsc{DP-UCB} (Algorithm~\ref{alg:dp_ucb_case2}). Crucially, while \textsc{Top-$K$-$P(Y=1)$-UCB} relies on prior structural knowledge that the environment obeys \textbf{Case 1} assumptions, \textsc{DP-UCB} is structure-agnostic—it works well across both regimes without requiring the learner to know the underlying noise regime.
}
\paragraph{Algorithm Design:}
We develop Upper confidence Bounds (UCB)-based online learning algorithms for both noise regimes. For the score-based regime (\textit{Case 1 and Shared-Channel regime}), we propose \textsc{Top-$K$-$P(Y=1)$-UCB} (Algorithm~\ref{alg:ucb_partial_exposure_case1}), which exploits the explicit ranking scores. For the general regime (\textit{Case 2}), identifying the reward-maximizing shortlist $\pi^*$ requires solving a non-trivial combinatorial optimization problem. Our central insight is that this problem can be solved more efficiently using dynamic programming, which forms the backbone of our main algorithm, \textsc{DP-UCB} (Algorithm~\ref{alg:dp_ucb_case2}). While \textsc{Top-$K$-$P(Y=1)$-UCB} requires prior knowledge that the environment satisfies the \textit{Case 1 or Shared-Channel regime} assumptions, \textsc{DP-UCB} is structure-agnostic and applies to both regimes without requiring knowledge of the underlying noise regime.
\paragraph{Theoretical guarantees:}
We establish instance-dependent regret upper bounds for \textsc{Top-$K$-$P(Y=1)$-UCB} (Theorem~\ref{thm:instance_dependent_bound}) under both Case~1 and shared channel regimes, complemented by matching lower bounds (up to constant factors) in Theorem~\ref{thm:score_lower_bounds}. These results demonstrate that \textsc{Top-$K$-$P(Y=1)$-UCB} achieves the optimal regret scaling of $\Theta((L-K)\log T)$.
For \textsc{DP-UCB}, Theorem~\ref{thm:instance_regret_DP_UCB} proves an instance-dependent upper bound of $O(LK \log T)$ irrespective of the noise regime. Furthermore, for the special case of $K = L$, Theorem~\ref{thm:DP_const_regret_bounds} shows that the bound collapses to a constant $O(1)$.

\paragraph{Empirical Validation.}
We validate our theoretical findings through extensive experiments on synthetic instances, the real-world MovieLens-1M dataset, and a novel LLM cascade constructed using real-world news headline datasets and LLM-as-a-judge evaluators. Across all  settings, our proposed algorithms achieve significantly better regret performance relative to various benchmarks. Details of our experiments and results are provided in Section~\ref{sec:experiments}.
\section{Problem Formulation}
We study a stochastic cascading bandit problem with a pool of $L$ arms, indexed by $[L]\triangleq\{1,2,\dots,L\}$. At each round $t \in [T]$, the learner selects an ordered tuple of $K \le L$ distinct arms, called a \textbf{shortlist}, $\pi_t = (\pi_t(1), \pi_t(2), \dots, \pi_t(K))$, where $\pi_t(m) \in [L]$ is the arm placed at position $m \in [K]$.

Each arm $i\in[L]$ carries an unobserved intrinsic binary state $X_i\sim\mathrm{Bernoulli}(\mu_i)$, mutually independent across arms and i.i.d.\ across rounds, where $\mu_i=\mathbb P(X_i=1)$ is the true success probability. The learner does not observe $X_i$ directly. Instead, observing arm $i$ yields a noisy binary observation $Y_i\in\{0,1\}$ generated by passing through an arm-specific channel defined by two error rates: the false-positive rate $\epsilon_{0,i}\triangleq\mathbb P(Y_i=1\mid X_i=0)$ and false-negative rate $\epsilon_{1,i}\triangleq\mathbb P(Y_i=0\mid X_i=1)$.\remove{
as summarized in Table~\ref{tab:transition_probabilities}.
\begin{table}[h]
\centering
\caption{Channel transition probabilities for arm $i$.}
\label{tab:transition_probabilities}
\begin{tabular}{c|cc}
\hline
\textbf{Input }$X_i$ & \textbf{Output }$Y_i=0$ & \textbf{Output }$Y_i=1$ \\ \hline
$X_i=0$ & $1-\epsilon_{0,i}$ & $\epsilon_{0,i}$ \\
$X_i=1$ & $\epsilon_{1,i}$ & $1-\epsilon_{1,i}$ \\ \hline
\end{tabular}
\end{table}
}
In each round $t$, the learner examines the shortlist $\pi_t$ sequentially from position $m=1$ to $K$. At position $m$, $Y_{\pi_t(m)}$ is observed if and only if $Y_{\pi_t(j)} = 0$ for all preceding positions $j < m$. If $Y_{\pi_t(m)} = 0$, the learner receives no reward and advances to position $m+1$. If $Y_{\pi_t(m)} = 1$, the cascade halts immediately: all remaining arms $j > m$ remain unobserved, and the learner probes arm $\pi_t(m)$ to reveal $X_{\pi_t(m)}$. Probing at position $m$ incurs a position-dependent cost $m\tau$, where the cost parameter $\tau \in [0, 1/K]$ ensures non-negative net rewards. Upon probing arm $\pi_t(m)$, the learner receives a net reward of $1-m\tau$ if $X_{\pi_t(m)} = 1$, and $0$ if $X_{\pi_t(m)} = 0$. If all $K$ positions yield $0$, the cascade completes without probing any arm, and the net reward is $0$.

To express the expected reward compactly, define the per-arm quantities for each $i \in [L]$:
\begin{equation}
r_i \triangleq \mathbb P(X_i=1, Y_i=1), \qquad g_i \triangleq \mathbb P(Y_i=0),
\label{eq:rg-def}
\end{equation}
representing the probability that observing arm $i$ leads to a successful probe, and the probability that observing arm $i$ does not stop the cascade, respectively. For any shortlist $\pi$, the expected reward is thus given by 
\begin{equation}
R(\pi) = \sum_{m=1}^{K} (1-m\tau) \, \mathbb P\!\left(X_{\pi(m)}=1, Y_{\pi(m)}=1\right) \smashoperator{\prod_{j=1}^{m-1}} \mathbb P\!\left(Y_{\pi(j)}=0\right) = \sum_{m=1}^{K} (1-m\tau) \, r_{\pi(m)} \smashoperator{\prod_{j=1}^{m-1}} g_{\pi(j)}.
\label{eq:reward_rg}
\end{equation}
We denote the optimal \textbf{oracle policy} as $\pi^\star \triangleq \arg\max_{\pi} R(\pi)$, optimized  over\footnote{For simplicity, we will assume $\pi^\star$ to be unique.} all ordered $K$-subsets of $[L]$. Lacking parameter knowledge, the learner must learn $\pi^\star$ from experience. At each round $t$, the learner incurs instantaneous regret $\mathcal{R}_t \triangleq R(\pi^\star) - R(\pi_t)$. Performance over $T$ rounds is evaluated via cumulative regret: $\mathcal R(T) = \sum_{t=1}^{T}\Big(R(\pi^\star)-R(\pi_t)\Big) = TR(\pi^\star)-\sum_{t=1}^{T}R(\pi_t)$. We define the shortlist gap as $\Delta(\pi) \triangleq R(\pi^\star) - R(\pi)$, with maximum and minimum gaps defined over all shortlists as: $\Delta_{\max} \triangleq \max_{\pi} \Delta(\pi)$, $\Delta_{\min} \triangleq \min_{\pi \neq \pi^\star} \Delta(\pi)$. Any played suboptimal shortlist $\pi_t \neq \pi^\star$ thus satisfies $\Delta_{\min} \le \mathcal{R}_t \le \Delta_{\max}$.

\section{Oracle Shortlist Policy}
\label{Oracle_section}
In general, finding the optimal oracle policy $\pi^\star$ requires a combinatorial search over both which $K$ arms to \textbf{select} and how to \textbf{order} them within the cascade. This search simplifies considerably when the instance admits per-arm \emph{scores}: the optimal shortlist can then be constructed by ranking arms according to quantities that depend only on their individual parameters, rather than by comparing entire shortlists. In particular, a selection score $s_i^{(1)}\triangleq f_1(\mu_i,\epsilon_{0,i},\epsilon_{1,i})$ identifies the $K$ arms to include, while an ordering score $s_i^{(2)}\triangleq f_2(\mu_i,\epsilon_{0,i},\epsilon_{1,i})$ determines their order within the shortlist; the two scores need not coincide. Theorem~\ref{thm:score-characterization} ( proof in Appendix~\ref{app:score_proofs_proof}) specifies when such a score-based simplification is possible. This characterization, summarized in Table~\ref{tab:oracle_scores_unified}, partitions the problem into two regimes based on the false-positive rates.
\begin{table*}[t]
\centering
\caption{Selection score $s_i^{(1)}$ and ordering score $s_i^{(2)}$ across the four named regimes (Theorem~\ref{thm:score-characterization}).}
\label{tab:oracle_scores_unified}
\footnotesize
\setlength{\tabcolsep}{6pt}
\renewcommand{\arraystretch}{1.25}
\begin{tabular}{|l|c|c|c|}
\hline
& \multicolumn{2}{c|}{\textbf{Cost-free ($\tau=0$)}} & \textbf{Costly ($\tau>0$)} \\ \cline{2-4}
\textbf{Regime} & $s_i^{(1)}$ & $s_i^{(2)}$ & $s_i^{(1)} = s_i^{(2)}$ \\ \hline
\textbf{Case 1} ($\epsilon_{0,i}=0 \quad \forall i \in [L]$)
  & \multicolumn{2}{c|}{\textbf{(Case 1a)}} 
  & \textbf{(Case 1b)} \\
  & $1-g_i = (1-\epsilon_{1,i})\mu_i$ 
  & order doesn't matter
  & $1-g_i = (1-\epsilon_{1,i})\mu_i$ \\ \hline
\textbf{Case 2} ($\epsilon_{0,i}>0$ for some $i$)
  & \multicolumn{2}{c|}{\textbf{(Case 2a)}} 
  & \textbf{(Case 2b)} \\
  & No score 
  & $\frac{r_i}{1-g_i} = \frac{(1-\epsilon_{1,i})\mu_i}{(1-\epsilon_{1,i})\mu_i + \epsilon_{0,i}(1-\mu_i)}$
  & No score \\ \hline
\end{tabular}
\end{table*}
\remove{
\begin{proof}
See Appendix~\ref{app:score_proofs_proof}.
\end{proof}
}
\remove{
Intersecting the channel condition ($\epsilon_{0,i}=0$ vs.\ $\epsilon_{0,i}>0$) with the cost parameter ($\tau=0$ vs.\ $\tau>0$) partitions the problem into four canonical regimes, summarized in Table~\ref{tab:oracle_scores_unified}:
}
The channel condition ($\epsilon_{0,i}=0 \quad \forall i \in [L]$, or $\epsilon_{0,i}>0$ for some $i \in [L]$) and the cost parameter (\(\tau=0\) or \(\tau>0\)) together define four canonical regimes, summarized in Table~\ref{tab:oracle_scores_unified}.
\begin{theorem}[Characterization of Selection and Ordering Scores]
\label{thm:score-characterization}
The expressions for $s_i^{(1)}$ and $s_i^{(2)}$ are provided in Table~\ref{tab:oracle_scores_unified} under their respective regimes.
\end{theorem}
Detailed proof of Theorem \ref{thm:score-characterization} is provided in Appendix~\ref{app:score_proofs_proof}. We collectively refer to Cases~1a and 1b as \textit{Case 1}, and Cases~2a and 2b as \textit{Case 2}. We show that scores are not always feasible in \textit{Case 2} via counterexamples detailed in Appendix~\ref{sec:case1-selection}.
\begin{remark}[Shared-Channel Regime]
\label{rem:shared-channel}
In addition to \textit{Case 1}, if the noise models are common across all arms ($\epsilon_{0,i} = \epsilon_0$, $\epsilon_{1,i} = \epsilon_1$ for all $i$, with $\epsilon_0 + \epsilon_1 < 1$), valid selection and ordering scores exist and coincide. Since $\epsilon_0$ and $\epsilon_1$ are constant, ranking arms by $\mu_i$, $\mathbb{P}(Y_i=1) = 1-g_i$, or $\mathbb{P}(X_i=1 \mid Y_i=1) = \frac{r_i}{1-g_i}$ produces identical orderings. Thus, any of these quantities can serve interchangeably as both $s_i^{(1)}$ and $s_i^{(2)}$ (proven in Appendix~\ref{app:score_proofs_proof}, Part~III).
\end{remark}


\subsection{Optimal Oracle Shortlist: Case 2}

In the absence of per-arm scores, computing $\pi^\star$ requires solving the joint selection and ordering problem. We show that this optimization can be solved efficiently via two dynamic programs, tailored to the zero-cost ($\tau=0$) and positive-cost ($\tau>0$) regimes (some additional intuition behind this is provided in Appendix~\ref{DP_algo}). We first develop generic dynamic-programming subroutines in Section~\ref{sec:case2-dp-subroutines}, and then instantiate them using the true parameters $(r_i,g_i)$ for the oracle in Section~\ref{sec:case2-oracle}, and the corresponding UCB indices $(U_i^r(t),U_i^g(t))$ for the learner in Section~\ref{sec:case2-alg-dpucb}.

\subsubsection{Dynamic-Programming Subroutines}
\label{sec:case2-dp-subroutines}
We construct two subroutines: \textsc{DP-Ordered} for $\tau=0$ and \textsc{DP-Subset} for $\tau>0$. Both accept candidate vectors $\rho=(\rho_1,\dots,\rho_L)$ and $\gamma=(\gamma_1,\dots,\gamma_L)$ in place of the true parameter vectors $r=(r_1,\dots,r_L)$ and $g=(g_1,\dots,g_L)$, respectively. Given $(\rho,\gamma,L,K)$ (and additionally $\tau$ for \textsc{DP-Subset}), each subroutine returns the optimal value and the corresponding ordered shortlist.
\paragraph{Subroutine \textsc{DP-Ordered} ($\tau=0$).}
The pseudo-code is provided in Appendix~\ref{sec:pseudo} (Algorithm~\ref{alg:dp_ordered}). For generic inputs $(\rho,\gamma)$, define the surrogate objective
$
\mathcal R_{\rho,\gamma}(\pi)
\triangleq
\sum_{m=1}^K \rho_{\pi(m)}
\prod_{j=1}^{m-1}\gamma_{\pi(j)},
$
obtained from \eqref{eq:reward_rg} by replacing $(r,g)$ with $(\rho,\gamma)$ and setting $\tau=0$. The corresponding ordering rule sorts arms in descending order of $s_i^{(2)}\triangleq\rho_i/(1-\gamma_i)$; let $\pi$ denote the resulting global ordering of all $L$ arms. For each state $(i,m)$, the subproblem is to select $m$ arms from $(\pi(i),\dots,\pi(L))$, preserving their order under $\pi$, so as to optimize the corresponding surrogate objective. The quantity $W(i,m)$ is the clipped dynamic-programming value associated with this subproblem. With $W(i,0)=0$ for all $i\in[L+1]$ and $W(L+1,m)=-\infty$ for all $m\in[K]$, the recursion is
{\setlength{\abovedisplayskip}{3pt}
\setlength{\belowdisplayskip}{3pt}
\begin{equation}
W(i,m)=\min\Big\{1,\max\Big\{
\underbrace{W(i+1,m)}_{\text{skip }\pi(i)}, \,
\underbrace{\rho_{\pi(i)}
+\gamma_{\pi(i)}W(i+1,m-1)}_{\text{include }\pi(i)}
\Big\}\Big\}.
\label{eq:dpA}
\end{equation}}
The recursion either skips $\pi(i)$ or includes it, contributing $\rho_{\pi(i)}$ and weighting the continuation value by $\gamma_{\pi(i)}$. Backtracking from $W(1,K)$ recovers the corresponding ordered shortlist in $O(L\log L+LK)$ time.
\remove{
\paragraph{Subroutine \textsc{DP-Ordered} ($\tau=0$).}
The formal pseudo-code is provided in Appendix~\ref{sec:pseudo} (Algorithm~\ref{alg:dp_ordered}). Here, any optimal subset is displayed in descending order of the ordering score $s_i^{(2)}\triangleq \rho_i/(1-\gamma_i)$. Let $\pi$ denote the permutation obtained by sorting all $L$ candidates in descending order of $s_i^{(2)}$. We define $W(i,m)$ as the maximum expected reward attainable by selecting $m$ arms from the candidate pool $(\pi(i),\dots,\pi(L))$, with their display order fixed by $\pi$. The base cases are $W(i,0) = 0$ for all $i \in [L+1]$ and $W(L+1,m) = -\infty$ for all $m \in [K]$. For $i \in [L]$ and $m \in [K]$, the recursive equation is given by
{\setlength{\abovedisplayskip}{3pt}
\setlength{\belowdisplayskip}{3pt}
\begin{align}
&W(i,m) = \min\Big\{ 1, \max\Big\{ \underbrace{W(i+1,m)}_{\text{skip } \pi(i)}, \nonumber \\
&\qquad \underbrace{\rho_{\pi(i)} + \gamma_{\pi(i)}W(i+1,m-1)}_{\text{include } \pi(i)} \Big\}\Big\}. \label{eq:dpA}
\end{align}}
Skipping candidate $\pi(i)$ leaves $m$ arms to be selected from $\pi(i+1),\dots,\pi(L)$. Including candidate $\pi(i)$ yields immediate expected reward $\rho{\pi(i)}$ and, with probability $\gamma_{\pi(i)}$, continues to the optimal solution $W(i+1,m-1)$ over the remaining $m-1$ choices. The optimal value $W(1,K)$ and the corresponding shortlist are recovered via backtracking, with an overall computational complexity of $O(L\log L+LK)$.
}
\paragraph{Subroutine \textsc{DP-Subset} ($\tau > 0$).}
Detailed in Appendix~\ref{sec:pseudo} (Algorithm~\ref{alg:dp_subset}), \textsc{DP-Subset} handles the positive-cost regime, where selection and positioning are jointly coupled and the score-based pre-sorting approach no longer applies. The subroutine maintains a state $(S,m)$, where $S\subseteq[L]$ is the set of available arms and $m\in[K]$ is the current position. Analogous to \textsc{DP-Ordered}, $W(S,m)$ denotes the clipped dynamic-programming value associated with the surrogate objective over positions $m,\dots,K$ using distinct arms from $S$. With $W(S,K+1)=0$ for all $S\subseteq[L]$, the recursion is
{\setlength{\abovedisplayskip}{3pt}
\setlength{\belowdisplayskip}{3pt}
\begin{equation}
W(S,m) = \min\Big\{1, \max_{i\in S}\Big\{ (1-m\tau)\rho_i + \gamma_i W(S\setminus\{i\}, m+1)\Big\}\Big\}.
\label{eq:dpB}
\end{equation}}
At position $m$, each candidate $i\in S$ contributes $(1-m\tau)\rho_i$, with $\gamma_i$ weighting the continuation value over the remaining arms. Backtracking from $W([L],1)$ recovers the corresponding ordered shortlist, with computational complexity $O(2^L L K)$.

\remove{
The base case is $W(S,K+1) = 0$ for all $S \subseteq [L]$. For all non-empty $S$ and $m \in [K]$, the recurrence is given by
{\setlength{\abovedisplayskip}{3pt}
\setlength{\belowdisplayskip}{3pt}
\begin{align}
W(S,m) = &\min\Big\{1, \max_{i\in S}\Big\{ (1-m\tau)\rho_i \nonumber \\
&+ \gamma_i W(S\setminus\{i\}, m+1)\Big\}\Big\}. \label{eq:dpB}
\end{align}}
At position $m$, the recursion considers each available arm $i\in S$, collects the position-discounted immediate reward $(1-m\tau)\rho_i$, and, with continuation factor $\gamma_i$, proceeds to position $m+1$ with the reduced candidate set $S\setminus{i}$. The optimal value $W([L],1)$ and the corresponding ordered shortlist are recovered by backtracking through the maximizing arm $i^\star(S,m)$ at each state. The resulting computational complexity is $O(2^L L K)$.
}
\subsubsection{Oracle policy: dynamic programming}
\label{sec:case2-oracle}

The optimal oracle strategy for Case 2 (the no-score regime) is summarized in Algorithm~\ref{alg:oracle_case2}.
Since the oracle knows the true parameters $(r_i,g_i)_{i\in[L]}$ from \eqref{eq:rg-def}, it invokes the dynamic programming subroutines of Section~\ref{sec:case2-dp-subroutines} with $\rho=r$, $\gamma=g$. Under these inputs, the surrogate objective coincides with the true reward in \eqref{eq:reward_rg}. Moreover, all resulting continuation values lie in $[0,1]$, so the clipping in \eqref{eq:dpA} and \eqref{eq:dpB} is inactive.

\textit{Case 2a ($\tau=0$):} For any selected subset, the optimal ordering is in descending order of $s_i^{(2)}=r_i/(1-g_i)$, precisely the ordering used by \textsc{DP-Ordered} (Algorithm~\ref{alg:oracle_case2}). Substituting $(\rho,\gamma)=(r,g)$ into \eqref{eq:dpA} therefore yields
$V_{\pi^\star}(i,m)\triangleq W(i,m)\big|_{(\rho,\gamma)=(r,g)}$, 
with $V_{\pi^\star}(1,K)=R(\pi^\star)$. A complete proof is provided in Appendix~\ref{sec:dp-ordered-proof}.
\remove{
\textit{Case 2a ($\tau=0$):} For any selected subset, the optimal ordering is in descending order of $s_i^{(2)}=r_i/(1-g_i)$, precisely the ordering used by \textsc{DP-Ordered} (Line 2 of Algorithm~\ref{alg:oracle_case2}). Substituting $(\rho,\gamma)=(r,g)$ into \eqref{eq:dpA} yields the oracle value function $V_{\pi^\star}(i,m)\triangleq W(i,m)\big|_{(\rho,\gamma)=(r,g)}$, yielding $V_{\pi^\star}(1,K) = R(\pi^\star)$. We present the complete proof of this identity in Appendix~\ref{sec:dp-ordered-proof}.
}

\textit{Case 2b ($\tau>0$):} Selection and positioning are jointly coupled, so no pre-sorting is possible and the oracle instead invokes \textsc{DP-Subset}  (Line 4 of Algorithm~\ref{alg:oracle_case2}). Substituting $(\rho,\gamma)=(r,g)$ into \eqref{eq:dpB} yields the oracle value function $V_{\pi^\star}(S,m)\triangleq W(S,m)\big|_{(\rho,\gamma)=(r,g)}$, with $V_{\pi^\star}([L],1)=R(\pi^\star)$. We present the complete proof of this identity in Appendix~\ref{sec:dp-subset-proof}.

\remove{
Algorithmically, score-based settings enable independent, global arm rankings, allowing for lightweight UCB algorithms (Section~\ref{sec:case1-alg}). In contrast, no-score instances demand combinatorial optimism embedded directly within dynamic programming (Section~\ref{sec:case2-value}).
Specifically, in Case~2a ($\tau=0$), perturbing an unselected arm flips the shortlist selection preference between two candidate arms, ruling out any decoupled selection score $s_i^{(1)}$. Likewise, in Case~2b ($\tau>0$), perturbing an unselected arm reverses the optimal ordering of candidate arms within the shortlist, ruling out any decoupled ordering score $s_i^{(2)}$. whenever score doesn't exist, the oracle strategy is detailed further in Section~\ref{sec:case2-value}.
}
\begin{algorithm}[H]
\caption{Oracle policy for Case 2 - optimal shortlist $\pi^\star$}
\label{alg:oracle_case2}
\begin{algorithmic}[1]
\Require true parameters $(r_i,g_i)_{i\in[L]}$, pool size $L$, shortlist size $K$, position cost $\tau\ge0$
\If{$\tau=0$} \Comment{Case 2a}
    \State $\big(R(\pi^\star),\ \pi^\star\big) \leftarrow \text{DP-Ordered}\big(r,\ g,\ L,\ K)$
\Else \Comment{Case 2b}
    \State $\big(R(\pi^\star),\ \pi^\star\big) \leftarrow \text{DP-Subset}\big(r,\ g,\ L,\ K,\ \tau)$
\EndIf
\State \Return $\big(R(\pi^\star),\ \pi^\star\big)$
\end{algorithmic}
\end{algorithm}
\section{Online schemes \& Regret bounds}
\label{sec:case1-alg}
Let $O_u(i) \triangleq \mathbf{1}\{i \in \pi_u \text{ and } Y_{\pi_u(k)}(u)=0,\ k=1,\ldots,\pi_u^{-1}(i)-1\}$ indicate whether arm $i$ is reached in round $u$, where $\pi_u^{-1}(i)$ denotes the position of arm $i$ in shortlist $\pi_u$. Also, let $N_i(t) \triangleq \sum_{u=1}^{t} O_u(i)$ denote the number of times arm $i$ is observed till round $t$. 
\remove{
{\color{red}NK: The score based regime as currently written does not invoke any relation to $r_i$, $g_i$, thus not being in line with above statement.}
{\color{blue}KY: removed it r, g are used only for no-score case} }
\subsection{Score-Based UCB and Regret Bounds (Case 1 and Shared-Channel Regime)}
For the score-based regime (Case~1 and Shared-Channel Regime, Remark~\ref{rem:shared-channel}), we propose the \textbf{Top-$K$-$P(Y=1)$-UCB} scheme (detailed in Appendix~\ref{sec:pseudo}, Algorithm~\ref{alg:ucb_partial_exposure_case1}). In this case, $s_i \triangleq \mathbb{P}(Y_i=1)$ serves as both the selection and ordering score. We estimate $s_i$ empirically as $\hat{s}_i(t) \triangleq \frac{1}{N_i(t)}\sum_{u=1}^{t} \mathbf{1}\{Y_i(u)=1\}\,O_u(i)$, and define its upper confidence bound as $U_i^s(t) \triangleq \min\{1, \, \hat{s}_i(t) + \sqrt{2\log t / \max(1, N_i(t))}\}$. At each round $t$, the algorithm selects the $K$ arms with the largest $U_i^s(t)$ indices and and ranks them in descending order  to form the shortlist $\pi_t$.

Indexing arms by their true score $s_1 > s_2 > \dots > s_L$, we define the score selection gap for any suboptimal arm $a > K$ as $\Delta_{a,K}^s \triangleq s_K - s_a$. Then we have the following regret upper bound.

\begin{theorem}[Instance-Dependent Regret Upper Bound]
\label{thm:instance_dependent_bound}
Under Case 1 and the Shared-Channel Regime (Remark~\ref{rem:shared-channel}), the cumulative expected regret of Algorithm~\ref{alg:ucb_partial_exposure_case1} over horizon $T$ satisfies
\begin{equation}
\mathbb E[\mathcal R(T)] \le (1-\tau)\sum_{a=K+1}^{L}\frac{16\log T}{\Delta_{a,K}^s} + O(1).
\end{equation}
\end{theorem}
\begin{proof}[Proof sketch]
The complete proof is deferred to Appendix~\ref{app:regret_proof}. Let $G_t$ (Appendix~\ref{app:aux_lemmas}, Definition~\ref{def:good_event}) denote the good event that all empirical estimates lie within their confidence radii, and let $\mathcal E_t \triangleq \{U_1^s(t) > \dots > U_K^s(t)\}$ be the event that the UCB indices preserve the true relative order of the top-$K$ arms. Splitting the instantaneous regret across $(G_t, \mathcal E_t)$ and their complements yields $\mathbb E[\mathcal R(T)] \le \sum_{t=1}^{T}\mathbb E[\mathcal R_t\mathbf 1\{G_t\cap\mathcal E_t\}] + \sum_{t=1}^{T}\mathbb E[\mathcal R_t\mathbf 1\{G_t\cap\mathcal E_t^c\}] + \sum_{t=1}^{T}\mathbb E[\mathcal R_t\mathbf 1\{G_t^c\}] \triangleq \mathcal R_{11} + \mathcal R_{12} + \mathcal R_2$. Under $G_t \cap \mathcal E_t$, positive regret arises solely when a suboptimal arm $a > K$ enters the shortlist, yielding $\mathcal R_{11} \le (1-\tau)\sum_{a=K+1}^{L}\frac{16\log T}{\Delta_{a,K}^s}$ for both Case~1 and the Shared-Channel Regime (Lemma~\ref{lem:r11_bound}). On $G_t \cap \mathcal E_t^c$, UCB index inversions occur among top-$K$ arms; however, we prove that the optimal arms are indeed observed $\Omega(T)$ times, thus bounding the probability of index inversions and establishing $\mathcal R_{12} = O(1)$ (Lemmas~\ref{lem:obs_concentration}--\ref{lem:disruption_bound}). Finally, the concentration failure bound $\mathbb P(G_t^c) \le 4Lt^{-3}$ (Lemma~\ref{lem:concentration_failure}) ensures $\mathcal R_2 \le \frac{2L\pi^2}{3}\Delta_{\max} = O(1)$, thus completing the proof.
\end{proof}

\subsection{From Oracle to Algorithm: DP-UCB}
\label{sec:case2-alg-dpucb}
For $N_i(t)>0$, $\hat r_i(t)$ and $\hat g_i(t)$ are unbiased empirical estimators of $r_i = \mathbb P(X_i=1, Y_i=1)$ and $g_i = \mathbb P(Y_i=0)$, respectively. These estimators and their  clipped upper confidence bounds (UCBs) are given by:
\remove{
 Let $O_u(i) \triangleq \mathbf{1}\big\{ i \in \pi_u \ \land Y_{\pi_u(k)}(u) = 0\text{ for }$ \\$k = 1, \dots, \pi_u^{-1}(i)-1 \big\}$ where $\pi_u^{-1}(i)$ denotes the position of arm $i$ in shortlist $\pi_u$ indicate whether arm $i$ is reached, and let $N_i(t) \triangleq \sum_{u=1}^{t}O_u(i)$ count its cumulative observations through round $t$. Conditioned on $N_i(t) > 0$, $\hat r_i(t)$ and $\hat g_i(t)$ are unbiased empirical estimators of $r_i$ and $g_i$. These estimators and their clipped upper confidence bounds (UCBs) are defined as:}
{\setlength{\abovedisplayskip}{3pt}
\setlength{\belowdisplayskip}{3pt}
\begin{align}
  \hat r_i(t) &= \frac{\sum_{u=1}^{t} \mathbf 1 \{Y_i(u)=1, X_i(u)=1\}\,O_u(i)}{N_i(t)}, \nonumber \\
  \hat g_i(t) &= \frac{\sum_{u=1}^{t} \mathbf 1 \{Y_i(u)=0\}\,O_u(i)}{N_i(t)}, \label{eq:emp_rg} \\
  U_i^r(t) &= \min\Big\{1, \, \hat r_i(t) + \sqrt{\tfrac{2\log t}{\max(1, N_i(t))}}\Big\}, \nonumber \\
  U_i^g(t) &= \min\Big\{1, \, \hat g_i(t) + \sqrt{\tfrac{2\log t}{\max(1, N_i(t))}}\Big\}, \label{eq:ucb_rg}
\end{align}}
\remove{Operationally, whenever arm $i$ is reached, i.e., $O_u(i)=1$, the signal $Y_i(u)$ is observed and $N_i(t)$ is incremented. If $Y_i(u)=0$, the cascade proceeds and the event $\mathbf{1}\{Y_i(u)=0\}$ contributes to the estimate $\hat g_i(t)$. If $Y_i(u)=1$, the cascade halts and the learner probes the arm to observe $X_i(u)$, with the joint event $\mathbf{1}\{Y_i(u)=1,X_i(u)=1\}$ contributing to $\hat r_i(t)$. Normalizing these cumulative event counts by $N_i(t)$ yields the unbiased empirical estimates $\hat g_i(t)$ and $\hat r_i(t)$.
}
Since the learner does not know $(r_i,g_i)$, DP-UCB (Algorithm~\ref{alg:dp_ucb_case2}) replaces them in each round with the optimistic estimates $(U_i^r(t),U_i^g(t))$ from \eqref{eq:ucb_rg}. It then invokes the same dynamic-programming subroutines as the oracle (Algorithm~\ref{alg:oracle_case2}), with inputs $\rho=U^r(t)$, $\gamma=U^g(t)$. The returned shortlist $\pi_t$ is played, cascade feedback is observed, and the statistics are updated.

\textit{Case 2a ($\tau=0$):} \textsc{DP-Ordered} (Line 5 of \textsc{DP-UCB}, Algorithm~\ref{alg:dp_ucb_case2}) ranks all $L$ arms in descending order of the optimistic ordering score $U_i^{s_2}(t)=U_i^r(t)/(1-U_i^g(t))$. Substituting $(\rho,\gamma)=(U^r(t),U^g(t))$ into the general recursion \eqref{eq:dpA} gives
\begin{equation}
\hat V_{\pi_t}(i,m)\triangleq W(i,m)\big|_{(\rho,\gamma)=(U^r(t),U^g(t))}.
\label{eq:vhat_tau0}
\end{equation}
Backtracking from state $(1,K)$ yields the optimistic shortlist $\pi_t$.

\textit{Case 2b ($\tau>0$):} With no global pre-sorting, \textsc{DP-Subset} (Line 7 of \textsc{DP-UCB}, Algorithm~\ref{alg:dp_ucb_case2}) substitutes $(\rho,\gamma)=(U^r(t),U^g(t))$ into the general recursion \eqref{eq:dpB} to obtain
\begin{equation}
\hat V_{\pi_t}(S,m)\triangleq W(S,m)\big|_{(\rho,\gamma)=(U^r(t),U^g(t))},
\label{eq:vhat_tau_pos}
\end{equation}
initialized with $\hat{V}_{\pi_t}(S, K{+}1) = 0$. Backtracking from state $([L],1)$ yields the optimistic shortlist $\pi_t$.
We have shown that $\hat{V}_{\pi_t}(i,m)$ and $\hat{V}_{\pi_t}([L], 1)$ are valid UCBs of $R({\pi_t})$ in Lemmas~\ref{lmm:2_algo_sameproof} and \ref{lem:valid_UCB}.
\begin{algorithm}[H]
\caption{DP-UCB }
\label{alg:dp_ucb_case2}
\begin{algorithmic}[1]
\Require pool size $L$, shortlist size $K$, position cost $\tau\ge0$, horizon $T$
\State Initialize $N_i(1)=0,\ \hat r_i(1)=0,\ \hat g_i(1)=0$ for all $i\in[L]$
\For{$t=1$ to $T$}
    \State Compute $U_i^r(t),\ U_i^g(t)$ via \eqref{eq:ucb_rg} for all $i\in[L]$
    \If{$\tau=0$} \Comment{Case 2a}
        \State $\big(\hat V_t, \pi_t\big) \leftarrow \text{DP-Ordered}\big(U^r(t), U^g(t), L, K)$
    \Else \Comment{Case 2b}
        \State $\big(\hat V_t, \pi_t\big) \leftarrow \text{DP-Subset}\big(U^r(t), U^g(t), L, K, \tau)$
    \EndIf
    \State Play the shortlist $\pi_t$; observe cascade feedback
    \State Update $N_i(t{+}1),\ \hat r_i(t{+}1),\ \hat g_i(t{+}1)$ for every observed arm
\EndFor
\end{algorithmic}
\end{algorithm}


We have the following regret bounds for DP-UCB. 
\begin{theorem}[Instance-Dependent Regret Bound, $L> K$]
\label{thm:instance_regret_DP_UCB}
 The cumulative regret of Algorithm~\ref{alg:dp_ucb_case2}, run with $\tau=0$ (Case 2a, invoking \textsc{DP-Ordered}) or $\tau>0$ (Case 2b, invoking \textsc{DP-Subset}), satisfies:
{\setlength{\abovedisplayskip}{3pt}
\setlength{\belowdisplayskip}{3pt}
\begin{align}
  \mathbb E[\mathcal R(T)] &= O\!\left(\frac{LK(1-\tau)\log T}{\Delta_{\min}}\right). \label{eq:general-LK-bound}
\end{align}}
\end{theorem}
\begin{proof}[Proof sketch]
The full proof is in Appendix~\ref{DP_UCB}. To unify Case 2a ($\tau=0$) and Case 2b ($\tau > 0$), Lemma~\ref{lmm:2_algo_sameproof} establishes that the state-wise optimistic suffix values $\hat{V}_{\pi_t}(i,m)$ in \eqref{eq:vhat_tau0} and $\hat{V}_{\pi_t}(S,m)$ in \eqref{eq:vhat_tau_pos} reduce to the identical positional recurrence $\hat{V}_{\pi_t}(m) = \min\{1, \, (1-m\tau)U_{\pi_t(m)}^r(t) + U_{\pi_t(m)}^g(t)\hat{V}_{\pi_t}(m+1)\}$. The true expected value at suffix $m$ is given by $V_{\pi_t}(m) = (1-m\tau)r_{\pi_t(m)} + g_{\pi_t(m)}V_{\pi_t}(m+1)$, where $V_{\pi_t}(1) = R(\pi_t)$ is the total expected reward of policy $\pi_t$. Letting $G_t$ denote the good event that all empirical estimates lie within their UCB confidence intervals at round $t$, we decompose cumulative regret as $\mathbb{E}[\mathcal{R}(T)] = \sum_{t=1}^{T} \mathbb{E}[\mathcal{R}_t \mathbf{1}\{G_t\}] + \sum_{t=1}^{T} \mathbb{E}[\mathcal{R}_t \mathbf{1}\{G_t^c\}] \triangleq \mathcal{R}_1(T) + \mathcal{R}_2(T)$. Under $G_t$, optimism from Lemma~\ref{lem:valid_UCB} and dynamic programming optimality yield $V_{\pi^\star}(1) \le \hat{V}_{\pi^\star}(1) \le \hat{V}_{\pi_t}(1)$. Unrolling the unified recurrence bounds the optimism gap as $\hat{V}_{\pi_t}(1) - V_{\pi_t}(1) \le 4(1-\tau)\sum_{k=1}^{K} \big(\prod_{j=1}^{k-1} g_{\pi_t(j)}\big) \sqrt{2 \log t / N_{\pi_t(k)}(t)}$ (Lemma~\ref{lem:unif_diff}), proving $\mathcal{R}_1(T) \le \frac{256(1-\tau)LK\log T}{\Delta_{\min}}$. Under $G_t^c$, the failure confidence bound $\mathbb{P}(G_t^c) \le \frac{2L}{t^2}$ (Lemma~\ref{lem:concentration_failure}) guarantees $\mathcal{R}_2(T) \le \sum_{t=1}^T \frac{2L}{t^2} \le \frac{2\pi^2 L}{3} ={O}(1)$, completing the proof.
\end{proof}

When $L=K$, selection is unnecessary, leaving only ordering. This  setting admits a constant regret bound:

\begin{theorem}[Constant Expected Regret for $L=K$]
\label{thm:DP_const_regret_bounds}
Consider Algorithm~\ref{alg:dp_ucb_case2} run with $L=K$. The cumulative expected regret satisfies $\mathbb{E}[\mathcal{R}(T)] = O(1)$ for both $\tau = 0$ and $\tau > 0$.
  
\end{theorem}
\begin{proof}[Proof sketch]
The complete proof is deferred to Appendix~\ref{pf:L=K}. Instantaneous regret arises only from misordering the arms. Crucially, when $L=K$, every arm is included in the shortlist at every round and is reached with probability uniformly lower-bounded by $\bar g_i\triangleq\prod_{k\neq i}g_k>0$. Consequently, each observation count $N_i(t)$ grows linearly with $t$, allowing the ordering scores to be learned sufficiently quickly to yield $O(1)$ cumulative regret.
\end{proof}

\section{Lower Bounds}
\label{sec:lower_bounds}

We begin by establishing information-theoretic lower bounds on cumulative regret for the score-based regimes. A policy is \textit{uniformly consistent} if $\mathbb{E}[\mathcal{R}(T)] = o(T^\alpha)$ for every $\alpha > 0$ across all problem instances.

\begin{theorem}[Regret Lower Bounds for Score-Based Regimes]
\label{thm:score_lower_bounds}
Under Case 1 and the Shared-Channel Regime, any uniformly consistent policy satisfies an asymptotic regret lower bound of $\Omega\big((L-K)\log T\big)$.
\end{theorem}
\begin{proof}[Proof sketch]
Full proofs are provided in Appendix~\ref{app:Lower_proof} (Case~1) and Appendix~\ref{lb:lower_bound} (Shared-Channel). By Lemma~\ref{lem:regret_decompo}, instantaneous regret satisfies $\mathcal{R}_t \ge R(\pi^\star) - R(\pi_t^\star)$, where $\pi_t^\star$ denotes the selected shortlist arranged in descending order of the true arm scores. We lower-bound expected regret via arm observation counts as $\mathbb{E}[\mathcal{R}(T)] \ge \sum_{t=1}^{T}\mathbb{E}[R(\pi^\star) - R(\pi_t^\star)] \ge C \sum_{a=K+1}^{L} \omega_a \mathbb{E}[N_a(T)]$ for constants $C, \omega_a > 0$(Lemma~\ref{lem:observation_bound} for Case~1; Lemma~\ref{lem:observation_bound_shared} for Shared-Channel). The core structural step decomposes the trajectory KL-divergence via the Markov property of the cascade observation sequence (Lemmas~\ref{lemma:cascaded_divergence_decomposition} and~\ref{for_perm} in Appendix~\ref{llowerb}). Applying the change-of-measure inequality for uniformly consistent policies \cite{garivier2011kl, kaufmann2016complexity} yields $\mathbb{E}[N_a(T)] = \Omega(\log T)$ for each of the $L-K$ suboptimal arms, completing the proof.
\end{proof}
\remove{
In the general regime (Case 2), establishing an explicit lower bound requires solving an optimization problem following the Lai--Robbins framework \cite{lai1985asymptotically} (see Appendix~\ref{sec:general-lb} for details). However, due to non-linear coupling between parameters across arbitrary noise configurations, expressing regret directly in terms of arm observation counts $\mathbb{E}[N_a(T)]$ as in \eqref{eq:regret_obs_lb} remains challenging, making an exact analytical solution for general instances difficult to obtain.
}
For the score-based regime, the lower bound derived above order-wise matches the regret upper bound in Theorem~\ref{thm:instance_dependent_bound}, establishing the order-optimality of our proposed scheme. In the general regime (Case~2), deriving an explicit lower bound instead requires solving an optimization problem based on the Lai--Robbins framework~\cite{lai1985asymptotically}; see Appendix~\ref{sec:general-lb}. However, the nonlinear coupling of parameters under general noise configurations makes it difficult to express regret directly in terms of the arm observation counts $\mathbb{E}[N_a(T)]$,  precluding a tractable analytical solution for general instances.
\section{Experiments}
\label{sec:experiments}
We benchmark our proposed policies against two baselines detailed in Appendix~\ref{exps}: \textbf{Perm-UCB} (Algorithm~\ref{perm_algo}) and \textbf{Corr-UCB} (Algorithm~\ref{alg:c_ucb_shortlist}). Perm-UCB treats all $\frac{L!}{(L-K)!}$ candidate shortlists as independent actions under UCB1 \cite{auer2002finite}, ignoring correlations between overlapping shortlists and resulting in a  regret bound of $O\left(\frac{L!}{(L-K)!} \log T\right)$. Corr-UCB adapts C-UCB~\cite{gupta2021multi} to exploit these correlations by estimating the required `pseudo-rewards'  from cascading feedback; see Appendices~\ref{corr} and \ref{corr1}. Nevertheless, Theorem~\ref{competitive proof} shows that Corr-UCB offers no improvement in the regret order, retaining the same $O\left(\frac{L!}{(L-K)!}\log T\right)$ bound.
For Case~2, we additionally consider heuristics based on $\mathbb{P}(Y=1)$ (Algorithm \ref{alg:ucb_partial_exposure_case1}) and $\mathbb{P}(X=1\mid Y=1)$ (Algorithm~\ref{alg:top_k_ratio_ucb}) to evaluate the performance of natural score-based approximations in the absence of a valid per-arm score.
\remove{
\paragraph{Evaluation Regimes.} We evaluate two settings: \textbf{(1) Case~1 (Score):} We benchmark \textsc{Top-$K$-$P(Y=1)$-UCB} (Algorithm~\ref{alg:ucb_partial_exposure_case1}) and \textsc{DP-UCB} (Algorithm~\ref{alg:dp_ucb_case2}) against \textsc{Perm-UCB} (Algorithm~\ref{perm_algo}) and \textsc{Corr-UCB} (Algorithm~\ref{alg:c_ucb_shortlist}), validating the logarithmic regret bound in Theorem~\ref{thm:instance_dependent_bound}. \textbf{(2) Case~2 (General):} We evaluate \textsc{DP-UCB} against the baselines and greedy UCB heuristics based on $P(Y=1)$ and $P(X=1 \mid Y=1)$ (Algorithm~\ref{alg:top_k_ratio_ucb}), showing how score-based approximations induce linear regret whereas \textsc{DP-UCB} achieves significantly lower regret (Theorems~\ref{thm:instance_regret_DP_UCB}--\ref{thm:DP_const_regret_bounds}).
}


\subsection{Synthetic Data}
\label{sec:exp_synthetic}
We consider several synthetic instances and evaluate the performance of the various schemes over $T=500{,}000$ rounds. We report cumulative regret averaged over 100 independent Monte Carlo trials, with shaded regions indicating one standard deviation. Full instance details $(\boldsymbol{\mu},\boldsymbol{\epsilon}_1,\boldsymbol{\epsilon}_0)$ and additional experimental results are provided in Appendix~\ref{app:extra_experiments}.
\paragraph{Case 1 (Score-Based Regime).} When false-positive rates vanish ($\boldsymbol{\epsilon}_0 = \mathbf{0}_L$), explicit ranking scores exist via $P(Y_i=1)$. Figures~\ref{fig:unified_empirical_results}(a) ($K < L$) and \ref{fig:unified_empirical_results}(b) ($K = L$) show that \textsc{DP-UCB} and \textsc{Top-$K$-$P(Y=1)$-UCB} achieve sub-linear regret, vastly outperforming \textsc{Perm-UCB} and \textsc{Corr-UCB}. Figure~\ref{fig:unified_empirical_results}(c) further shows that \textsc{Top-$K$-$P(Y=1)$-UCB} outperforms \textsc{DP-UCB} by directly exploiting the score-based structure, consistent with Theorems~\ref{thm:instance_dependent_bound}, \ref{thm:instance_regret_DP_UCB}, and \ref{thm:DP_const_regret_bounds}.
\paragraph{Case 2 (No-Score Regime).} When $\boldsymbol{\epsilon}_0 \neq \mathbf{0}_L$, no valid per-arm ranking score exists. Heuristics that treat $\mathbb{P}(Y_i=1)$ or $\mathbb{P}(X_i=1\mid Y_i=1)$ as ranking scores can select or order arms suboptimally, resulting in linear regret. In contrast, \textsc{DP-UCB} achieves sub-linear regret for both $K<L$ (Figure~\ref{fig:unified_empirical_results}(d)) and $K = L$ (Figure~\ref{fig:unified_empirical_results}(e)), consistent with Theorems~\ref{thm:instance_regret_DP_UCB} and \ref{thm:DP_const_regret_bounds}.

\paragraph{Shared-Channel Regime.} When noise models are common across arms ($\boldsymbol{\epsilon}_1 = 0.3 \cdot \mathbf{1}_6, \boldsymbol{\epsilon}_0 = 0.2 \cdot \mathbf{1}_6$), valid per-arm ranking scores exist; in particular, both $\mathbb{P}(Y_i=1)$ and $\mathbb{P}(X_i=1\mid Y_i=1)$ serve as ranking scores. Accordingly, \textsc{Top-$P(X=1\mid Y=1)$-UCB}, \textsc{Top-$P(Y=1)$-UCB}, and \textsc{DP-UCB} achieve sub-linear regret while significantly outperforming the other baselines (Figure~\ref{fig:unified_empirical_results}(f)).

\begin{figure*}[h!]
    \centering
    \begin{subfigure}[b]{0.3\textwidth}
        \centering
        \includegraphics[width=\textwidth]{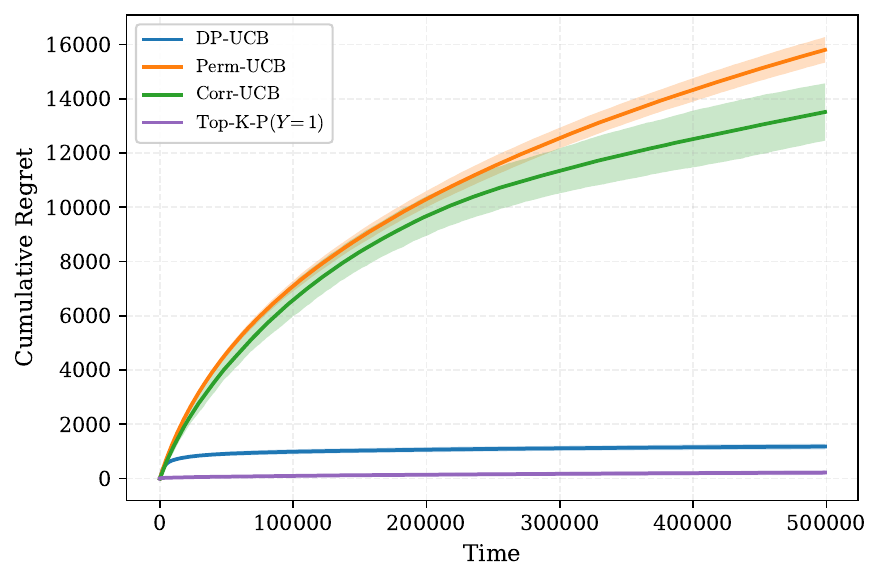}
        \caption{Case 1: $L=6, K=3, \tau=0.20$}
        \label{fig:oracle_results_5}
    \end{subfigure}
    \hfill
    \begin{subfigure}[b]{0.3\textwidth}
        \centering
        \includegraphics[width=\textwidth]{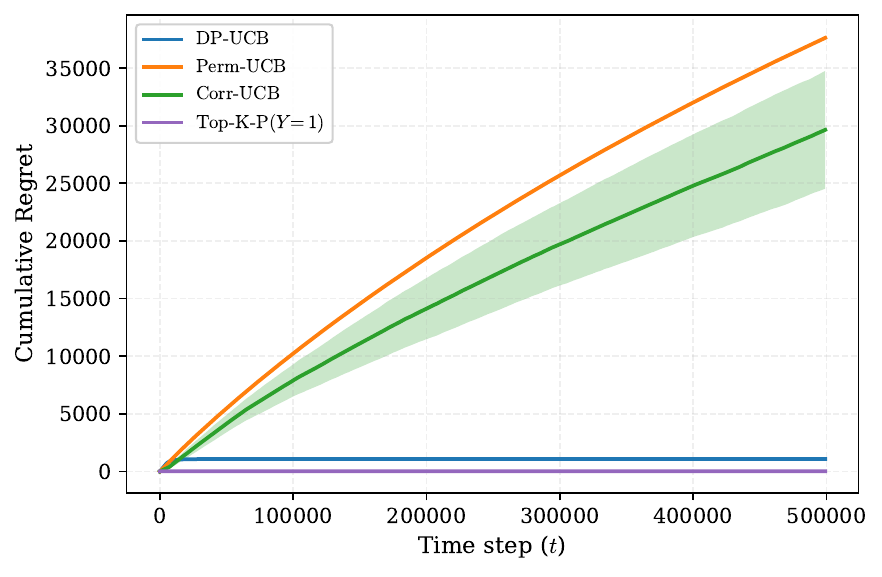}
        \caption{Case 1: $L=6, K=6, \tau=0.15$}
        \label{fig:oracle_results_6}
    \end{subfigure}
    \hfill
    \begin{subfigure}[b]{0.3\textwidth}
        \centering
        \includegraphics[width=\textwidth]{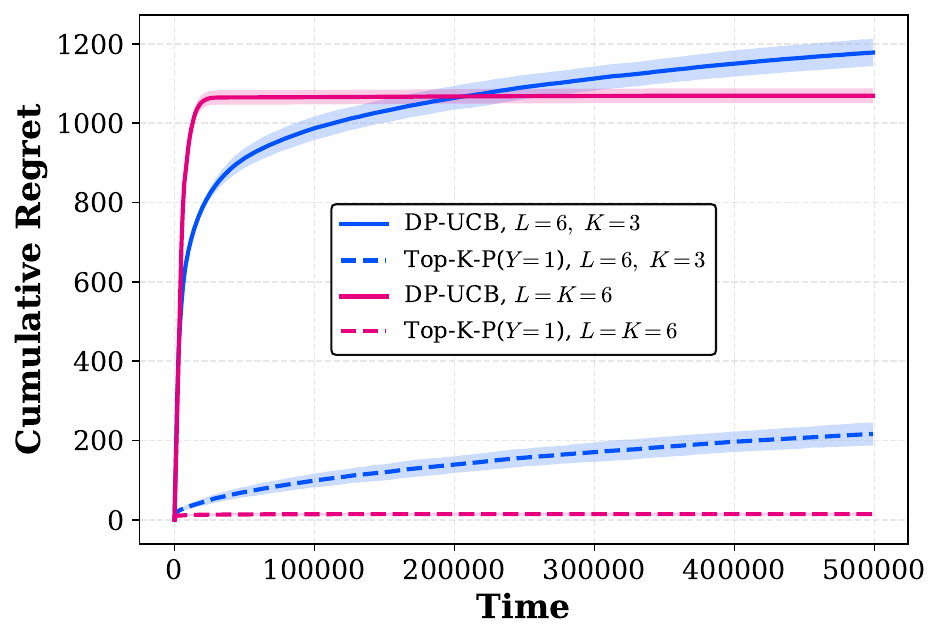}
        \caption{Case 1: $K < L$ vs. $K = L$}
        \label{fig:dp_ucb_vs_topk_l6}
    \end{subfigure}

    \vspace{0.4em}

    \begin{subfigure}[b]{0.3\textwidth}
        \centering
        \includegraphics[width=\textwidth]{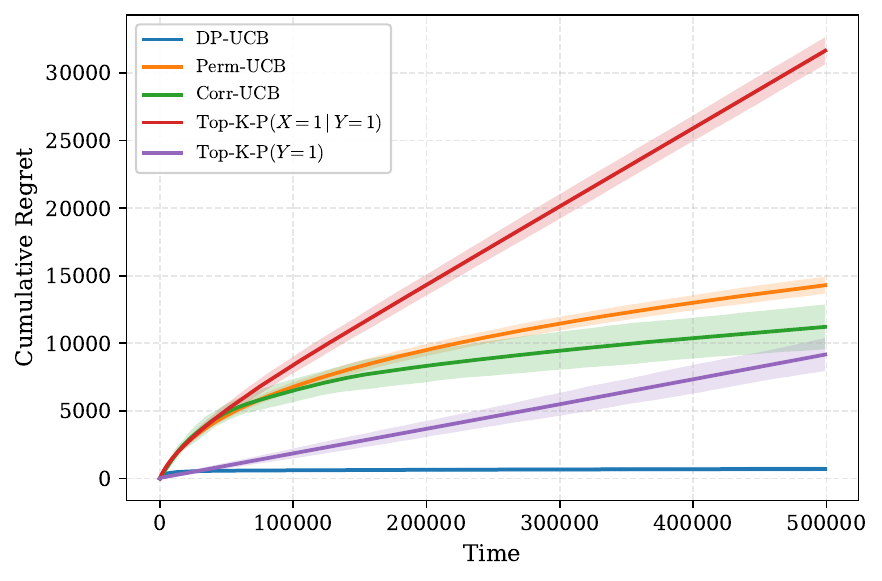}
        \caption{Case 2: $\tau = 0.20, L=6, K = 3$}
        \label{fig:oracle_results_tau_02}
    \end{subfigure}
    \hfill
    \begin{subfigure}[b]{0.3\textwidth}
        \centering
        \includegraphics[width=\textwidth]{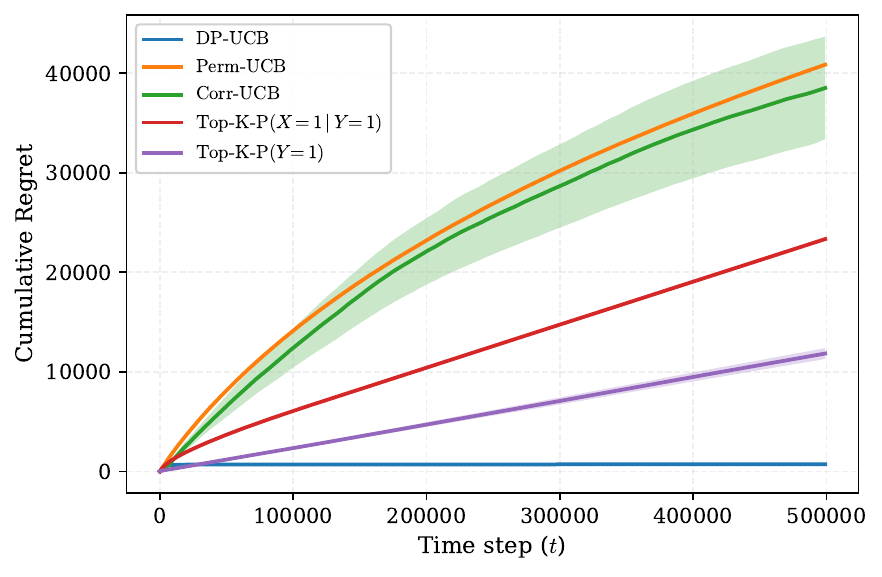}
        \caption{Case 2: $\tau = 0.15, K = L = 6$}
        \label{fig:oracle_results_full_feedback}
    \end{subfigure}
    \hfill
    \begin{subfigure}[b]{0.3\textwidth}
        \centering
        \includegraphics[width=\textwidth]{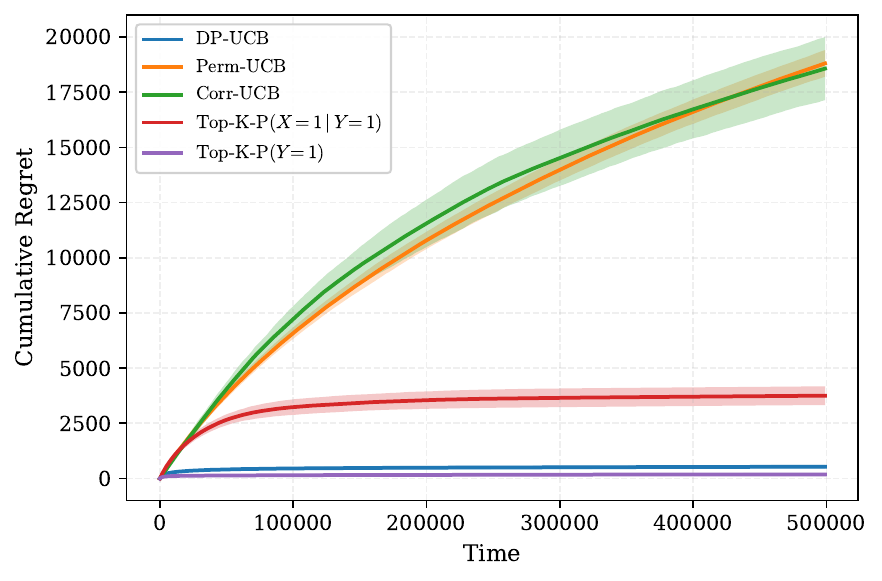}
        \caption{Shared-Channel: $\tau = 0.30$}
        \label{fig:common_eps_tau_03}
    \end{subfigure}

    \caption{Cumulative regret vs.\ time. 
\textbf{(a)--(c) Case 1 (Score):} $\boldsymbol{\epsilon}_0=\mathbf{0}_6$, \textsc{Top-$K$-$P(Y=1)$-UCB} vs. \textsc{DP-UCB}. 
\textbf{(d)--(e) Case 2 (no-score):} Regret for $K < L$ and $K = L$. 
\textbf{(f) Shared-Channel:} $K=3, L=6, \tau=0.15$.}
\label{fig:unified_empirical_results}
    \label{fig:unified_empirical_results}
\end{figure*}

\subsection{MovieLens Real-World Benchmark}
\label{sec:exp_movielens}
We validate our framework on the MovieLens-1M dataset~\cite{harper2015movielens}, modeling the $L=10$ highest rated movie genres as arms. For each genre $i$, ratings are quantized into binary outcomes $X_i\in \{0,1\}$, with $X_i=1$ for ratings at least $4$ ; the arm mean $\mu_i$ is the empirical fraction of ratings in genre $i$ that are at least $4$. Noisy feedback $Y_i\in \{0,1\}$ is generated by passing $X_i$ through synthetic false-negative ($\boldsymbol{\epsilon}_1$) and false-positive ($\boldsymbol{\epsilon}_0$) channels.



Figure~\ref{fig:movielens_results_combined} reports cumulative regret over $T=500{,}000$ rounds, averaged across 100 independent trials. In the score-based regime ($\boldsymbol{\epsilon}_0=\mathbf{0}_{10}$), \textsc{DP-UCB} and \textsc{Top-$K$-$P(Y=1)$-UCB} achieve sub-linear regret and substantially outperform the other baselines. In the no-score regime ($\boldsymbol{\epsilon}_0\neq\mathbf{0}_{10}$), heuristics treating $\mathbb{P}(Y_i=1)$ or $\mathbb{P}(X_i=1\mid Y_i=1)$ as ranking scores incur linear regret for both $\tau=0.15$ and $\tau=0$ (Figures~\ref{fig:movielens_results_combined}(b)-(c)), while \textsc{DP-UCB} maintains sub-linear regret, consistent with Theorem~\ref{thm:instance_regret_DP_UCB}. Further details regarding the dataset construction, parameter configurations, and additional results are provided in Appendix~\ref{_app:movielens_details}.


\begin{figure*}[h!]
    \centering
    \begin{subfigure}[b]{0.3\textwidth}
        \centering
        \includegraphics[width=\textwidth]{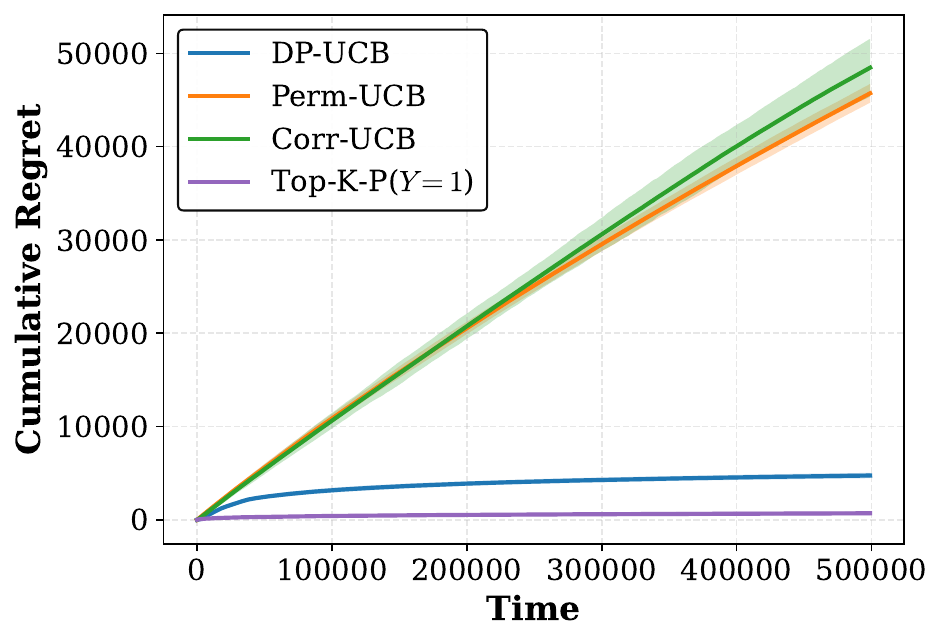}
        \caption{$L = 10, K=3, \tau=0.20$}
        \label{fig:movielens_score_based}
    \end{subfigure}
    \hfill
    \begin{subfigure}[b]{0.3\textwidth}
        \centering
        \includegraphics[width=\textwidth]{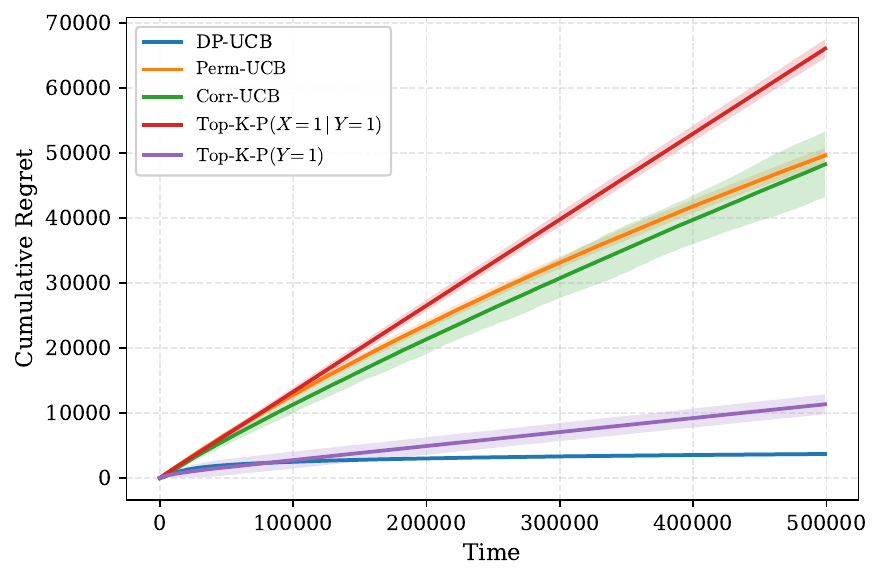}
        \caption{$L = 10, K=3, \tau = 0.15$}
        \label{fig:movielens_tau_015}
    \end{subfigure}
    \hfill
    \begin{subfigure}[b]{0.3\textwidth}
        \centering
        \includegraphics[width=\textwidth]{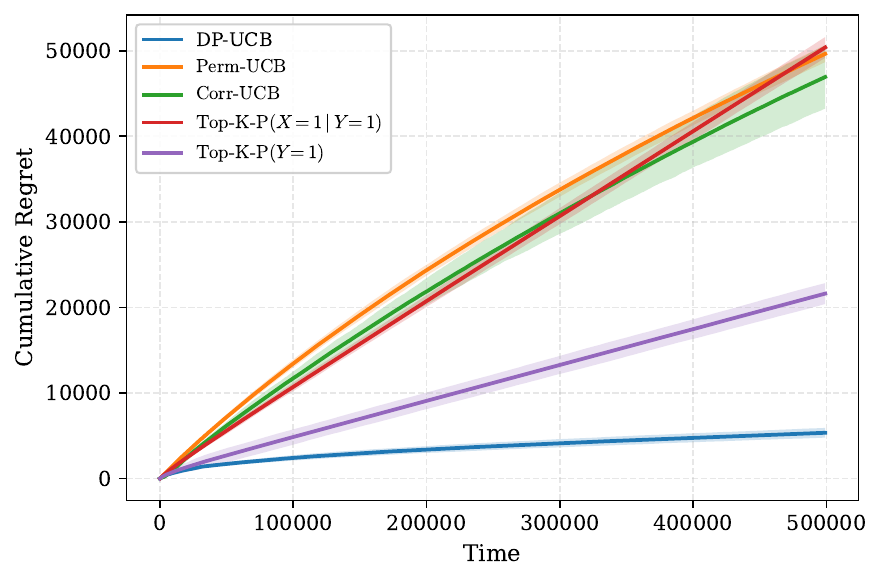}
        \caption{$\boldsymbol{\epsilon}_0 \neq \mathbf{0}_{10}, L= 10, K=3, \tau = 0$}
        \label{fig:movielens_tau_0}
    \end{subfigure}

    \caption{Cumulative regret vs.\ time on MovieLens-1M. 
\textbf{(a) Case 1 (Score regime):} $\boldsymbol{\epsilon}_0 = \mathbf{0}_{10}$. 
\textbf{(b) Case 2 (No-score regime):} $\tau = 0.15$. 
\textbf{(c) Case 2 (No-score regime):} $\tau = 0$.}
\label{fig:movielens_results_combined}
\end{figure*}

\subsection{LLM Headlines Cascade}
\label{sec:exp_llm}
To demonstrate an application of our framework to LLM routing via model cascades~\cite{chen2023frugalgpt}, we consider sentiment prediction on a financial news headline dataset~\cite{sinha2021impact}. Here, each of the $L = 5$ arms corresponds to an auto-regressive LLM generator paired with an LLM judge that acts as a noisy evaluation channel. For each sampled headline, $X_i\in \{0,1\}$ indicates whether generator $i$ correctly predicts the gold-standard sentiment, while $Y_i\in \{0,1\}$ denotes the judge's assessment. A rejection ($Y_i=0$) routes the headline to the next LLM in the cascade, whereas an acceptance ($Y_i=1$) terminates the cascade. The arm accuracies $\boldsymbol{\mu}$ and judge false-negative and false-positive rates $(\boldsymbol{\epsilon}_1,\boldsymbol{\epsilon}_0)$ are estimated  from the dataset and reported in Table~\ref{tab:llm_cascade_params}. We evaluate regret over $T=1{,}000{,}000$ rounds, averaging results across 100 independent experiments. Details regarding model architectures, noise parameters, and dataset pre-processing are provided in Appendix~\ref{app:LLM_data}.
\remove{
The arm accuracies $\boldsymbol{\mu}$, false-negative rates $\boldsymbol{\epsilon}_1$, and false-positive rates $\boldsymbol{\epsilon}_0$ are the empirical averages over the dataset listed in Table~\ref{tab:llm_cascade_params}. All evaluations are conducted across $T = 1{,}000{,}000$ rounds and averaged over 100 independent experiments.
}
Figure~\ref{fig:llm_results_combined} compares performance under two  configurations: $K=3$, $L = 5$, $\tau=0.20$ (Figure~\ref{fig:llm_results_combined}(a)) and $K=L=5$, $\tau=0.18$ (Figure~\ref{fig:llm_results_combined}(b)). In both settings, heuristics that treat $\mathbb{P}(Y_i=1)$ or $\mathbb{P}(X_i=1\mid Y_i=1)$ as ranking scores yield suboptimal cascades and incur linear regret. In contrast, \textsc{DP-UCB} achieves sub-linear regret for both $K<L$ and $K=L$, consistent with Theorems~\ref{thm:instance_regret_DP_UCB} and \ref{thm:DP_const_regret_bounds}.
\section{Conclusion}
\remove{
In this work, we introduced a novel framework for cascading bandits under noisy observations, modeling key real-world setups like LLM-as-a-judge cascades. Within this setting, we adapted \textsc{Top-$K$-$P(Y=1)$-UCB}—which achieves optimal $\Theta((L-K)\log T)$ regret in relevant regimes—and proposed a novel DP-UCB algorithm for general regimes. DP-UCB significantly outperform existing baselines like Perm-UCB and Corr-UCB. While DP-UCB offers lower computational complexity than fixed-order searches, scaling to very large action spaces remains computationally intensive. Developing an algorithm that is simultaneously computationally efficient and regret-optimal is an exciting direction for future work.
}
We study cascading bandits with noisy observations, motivated by applications such as LLM cascades in which model outputs are screened by smaller, imperfect evaluation models. We characterize when the optimal policy admits a per-arm score-based representation. For the score-based regime, we propose \textsc{Top-$K$-$P(Y=1)$-UCB} and establish an order-optimal regret guarantee. For the more challenging no-score regime, we develop \textsc{DP-UCB}, which uses dynamic programming to jointly optimize shortlist selection and ordering, and provide analytical regret guarantees. Extensive experiments on synthetic and real-world settings demonstrate that our proposed algorithms substantially outperform other baselines. Several directions remain open. Although dynamic programming is considerably more efficient than exhaustive combinatorial search, its complexity can still be prohibitive for large action spaces (under positive observation cost), motivating more scalable algorithms. Further, while our lower bound establishes order-optimality in the score-based regime, obtaining tractable regret lower bounds for the general no-score regime remains open.
\begin{table*}
    \centering
    \footnotesize
    \caption{Empirical parameters for LLM-judge cascade arms across the financial headlines dataset.}
    \begin{tabular}{c l l c c c}
        \hline
        Arm & Generator (LLM) & Discriminator (Judge) & $\mu$ & $\epsilon_1$ & $\epsilon_0$ \\
        \hline
        $1$ & SmolLM2-1.7B-Instruct & bert-base-uncased & $0.7036$ & $0.0387$ & $0.1596$ \\
        $2$ & Qwen2.5-0.5B-Instruct & roberta-base & $0.6496$ & $0.0328$ & $0.2182$ \\
        $3$ & phi-2 & electra-base-discriminator & $0.7462$ & $0.0444$ & $0.3463$ \\
        $4$ & stablelm-2-1\_6b-chat & albert-base-v2 & $0.6579$ & $0.0761$ & $0.2074$ \\
        $5$ & SmolLM2-360M-Instruct & DeBERTa-v3-base & $0.4155$ & $0.1328$ & $0.0398$ \\
        \hline
    \end{tabular}
    \label{tab:llm_cascade_params}
\end{table*}
\begin{figure}[H]
    \centering
    \begin{subfigure}[b]{0.49\linewidth}
        \centering
        \includegraphics[width=\textwidth]{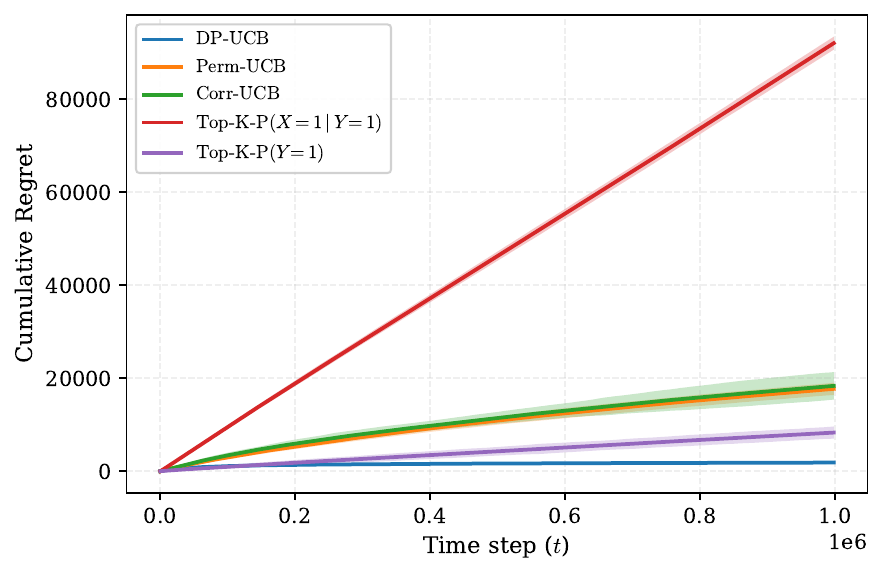}
        \caption{$L=5, K=3, \tau=0.20$}
        \label{fig:llm_k3_comparison}
    \end{subfigure}
    \hfill
    \begin{subfigure}[b]{0.49\linewidth}
        \centering
        \includegraphics[width=\textwidth]{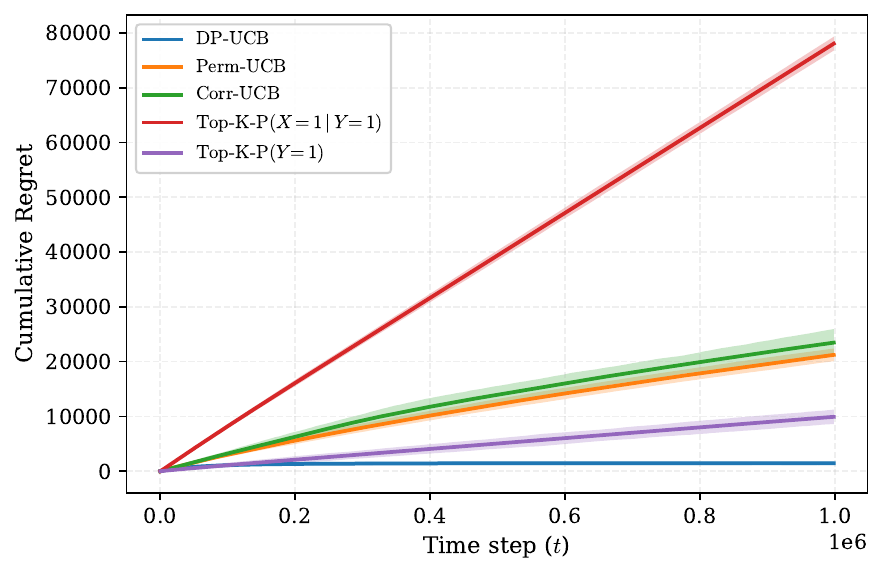}
        \caption{$L=K=5, \tau=0.18$}
        \label{fig:llm_k5_comparison}
    \end{subfigure}

    \caption{Cumulative regret on LLM Headlines.}
    \label{fig:llm_results_combined}
\end{figure}

\bibliographystyle{unsrtnat}
\bibliography{references}

\onecolumn
\section*{APPENDIX \\ Cascading Bandits with Noisy Observations: \\ Algorithms and Regret Bounds \\ Supplementary Materials}

\section{Proof of Theorem~\ref{thm:score-characterization}}
\label{app:score_proofs_proof}

Recall $r_i=\mathbb P(X_i=1,Y_i=1)$, $g_i=\mathbb P(Y_i=0)$. We use the
four regime names Case 1a ($\epsilon_{0,i}=0,\tau=0$), Case 1b ($\epsilon_{0,i}=0,\tau>0$),
Case 2a ($\epsilon_{0,i}>0$ for some $i$, $\tau=0$), Case 2b ($\epsilon_{0,i}>0$ for
some $i$, $\tau>0$) exactly as in Table~\ref{tab:oracle_scores_unified}. For a
shortlist $\pi$ and position $m\in\{1,\dots,K+1\}$, write $\Phi_\pi(m)$ for the
expected reward accumulated from position $m$ through position $K$ under $\pi$,
\[
\Phi_\pi(m)\triangleq\sum_{i=m}^{K}(1-i\tau)\,r_{\pi(i)}\prod_{j=m}^{i-1}g_{\pi(j)}.
\]

\subsubsection*{Two exchange identities}

\begin{lemma}[Adjacent swap]
\label{lem:adjacent_swap}
Let $\pi'$ swap the arms at positions $i,i+1$ of $\pi$, and let
$P_{i-1}=\prod_{j=1}^{i-1}g_{\pi(j)}$. Then
\begin{equation}
R(\pi)-R(\pi')=P_{i-1}\Big[(1-i\tau)(r_{\pi(i)}-r_{\pi(i+1)})
+\big(1-(i+1)\tau\big)\big(g_{\pi(i)}r_{\pi(i+1)}-g_{\pi(i+1)}r_{\pi(i)}\big)\Big].
\label{eq:swap_general}
\end{equation}
\end{lemma}
\begin{proof}
Let $a=\pi(i)$, $b=\pi(i+1)$, so $\pi'(i)=b$, $\pi'(i+1)=a$, and $\pi,\pi'$ agree
elsewhere. From \eqref{eq:reward_rg},
\begin{align}
R(\pi)&=\sum_{m\ne i,i+1}(1-m\tau)r_{\pi(m)}\prod_{j<m}g_{\pi(j)}
+P_{i-1}(1-i\tau)r_a+P_{i-1}g_a(1-(i+1)\tau)r_b,
\label{eq:Rpi}\\
R(\pi')&=\sum_{m\ne i,i+1}(1-m\tau)r_{\pi(m)}\prod_{j<m}g_{\pi(j)}
+P_{i-1}(1-i\tau)r_b+P_{i-1}g_b(1-(i+1)\tau)r_a.
\label{eq:Rpip}
\end{align}
Subtracting \eqref{eq:Rpip} from \eqref{eq:Rpi}, the sums over $m\ne i,i+1$ cancel
since $\pi(m)=\pi'(m)$ and $\prod_{j<m}g_{\pi(j)}=\prod_{j<m}g_{\pi'(j)}$ there, leaving
\begin{equation}
R(\pi)-R(\pi')=P_{i-1}(1-i\tau)(r_a-r_b)+P_{i-1}(1-(i+1)\tau)(g_ar_b-g_br_a).
\label{eq:swap_step}
\end{equation}
Factoring $P_{i-1}$ out of \eqref{eq:swap_step} and substituting $a=\pi(i)$, $b=\pi(i+1)$
gives \eqref{eq:swap_general}.
\end{proof}

\begin{lemma}[Slot substitution]
\label{lem:slot_sub}
Let $\pi,\pi''$ agree except at position $m$, with $\pi(m)=i$, $\pi''(m)=j\notin\pi$.
Let $P_{m-1}=\prod_{k=1}^{m-1}g_{\pi(k)}$ and $\Phi=\Phi_\pi(m+1)=\Phi_{\pi''}(m+1)$. Then
\begin{equation}
R(\pi)-R(\pi'')=P_{m-1}\Big[(1-m\tau)(r_i-r_j)+(g_i-g_j)\Phi\Big],\qquad 0\le\Phi\le1-(m+1)\tau<1-m\tau.
\label{eq:slot_general}
\end{equation}
\end{lemma}
\begin{proof}
From \eqref{eq:reward_rg} and the identity satisfied by $\Phi_\pi$ above,
\begin{align}
R(\pi)&=\sum_{k<m}(1-k\tau)r_{\pi(k)}\prod_{l<k}g_{\pi(l)}
+P_{m-1}\big[(1-m\tau)r_i+g_i\Phi\big],
\label{eq:Rpi_slot}\\
R(\pi'')&=\sum_{k<m}(1-k\tau)r_{\pi(k)}\prod_{l<k}g_{\pi(l)}
+P_{m-1}\big[(1-m\tau)r_j+g_j\Phi\big],
\label{eq:Rpipp_slot}
\end{align}
where the sums over $k<m$ are identical in \eqref{eq:Rpi_slot} and
\eqref{eq:Rpipp_slot} since $\pi,\pi''$ agree on positions $1,\dots,m-1$, and the
suffix quantity from $m+1$ onward is the same $\Phi$ in both since $\pi,\pi''$ agree on
positions $m+1,\dots,K$. Subtracting \eqref{eq:Rpipp_slot} from \eqref{eq:Rpi_slot},
the $k<m$ sums cancel, leaving
\begin{equation}
R(\pi)-R(\pi'')=P_{m-1}(1-m\tau)(r_i-r_j)+P_{m-1}(g_i-g_j)\Phi,
\label{eq:slot_step}
\end{equation}
and factoring $P_{m-1}$ out of \eqref{eq:slot_step} gives \eqref{eq:slot_general}.

For the bound on $\Phi$, note $\Phi_\pi(K+1)=0\le1-(K+1)\tau$. If $\Phi_\pi(l+1)\le1-(l+1)\tau$ for some $l\ge m$, then
using $r_{\pi(l)}+g_{\pi(l)}\le1$,
\begin{equation}
\Phi_\pi(l)=(1-l\tau)r_{\pi(l)}+g_{\pi(l)}\Phi_\pi(l+1)
\le(1-l\tau)r_{\pi(l)}+g_{\pi(l)}(1-l\tau)=(1-l\tau)\big(r_{\pi(l)}+g_{\pi(l)}\big)\le1-l\tau.
\label{eq:Phi_induction}
\end{equation}
Applying \eqref{eq:Phi_induction} inductively from $l=K$ down to $l=m+1$ gives
$\Phi=\Phi_\pi(m+1)\le1-(m+1)\tau$, and $1-(m+1)\tau<1-m\tau$ since $\tau>0$.
\end{proof}

Lemma~\ref{lem:adjacent_swap} governs the ordering score (Part I); Lemma
\ref{lem:slot_sub} governs the selection score (Part II).

\subsubsection*{Part I: ordering score}

\paragraph{Case 1a/1b ($\epsilon_{0,i}=0$, all $\tau$):}
$\epsilon_{0,i}=0$ means a positive signal can only come from a true success, so $\mathbb P(X_i=1|Y_i=1)=1$
for every arm
\begin{equation}
\mathbb P(X_i=1,Y_i=1)=r_i\ \Longrightarrow\ g_i=1-r_i.
\label{eq:case1_gi}
\end{equation}
Using \eqref{eq:case1_gi} for both $\pi(i)$ and $\pi(i+1)$, the cross term in
\eqref{eq:swap_general} becomes
\begin{equation}
g_{\pi(i)}r_{\pi(i+1)}-g_{\pi(i+1)}r_{\pi(i)}
=(1-r_{\pi(i)})r_{\pi(i+1)}-(1-r_{\pi(i+1)})r_{\pi(i)}
=-\big(r_{\pi(i)}-r_{\pi(i+1)}\big).
\label{eq:case1_cross}
\end{equation}
Substituting \eqref{eq:case1_cross} back into \eqref{eq:swap_general},
\begin{equation}
R(\pi)-R(\pi')=P_{i-1}\big(r_{\pi(i)}-r_{\pi(i+1)}\big)\Big[(1-i\tau)-\big(1-(i+1)\tau\big)\Big]
=P_{i-1}\,\tau\,\big(r_{\pi(i)}-r_{\pi(i+1)}\big).
\label{eq:case1_swap}
\end{equation}
\emph{Case 1a} ($\tau=0$): \eqref{eq:case1_swap} is $0$ for every $i$, so swapping
any two arms leaves the reward unchanged and ordering plays no role.
\emph{Case 1b} ($\tau>0$): since $P_{i-1}>0$ and $\tau>0$, \eqref{eq:case1_swap}
gives
\begin{equation}
R(\pi)>R(\pi')\iff r_{\pi(i)}>r_{\pi(i+1)}
\label{eq:case1b_criterion}
\end{equation}
at every position and for every $\pi$. If two adjacent arms are out of order by
\eqref{eq:case1b_criterion}, swapping them strictly increases the reward, so an
optimal shortlist must have its arms sorted in decreasing order of
$s_i^{(2)}=r_i=(1-\epsilon_{1,i})\mu_i$.

\paragraph{Case 2a ($\epsilon_{0,i}>0$, $\tau=0$):}
Setting $\tau=0$ in \eqref{eq:swap_general},
\begin{equation}
R(\pi)-R(\pi')=P_{i-1}\Big[r_{\pi(i)}\big(1-g_{\pi(i+1)}\big)-r_{\pi(i+1)}\big(1-g_{\pi(i)}\big)\Big].
\label{eq:case2a_swap}
\end{equation}
Both $1-g_{\pi(i)}$ and $1-g_{\pi(i+1)}$ are positive, so dividing
\eqref{eq:case2a_swap} by $P_{i-1}(1-g_{\pi(i)})(1-g_{\pi(i+1)})$ turns it into
\begin{equation}
R(\pi)>R(\pi')\iff \frac{r_{\pi(i)}}{1-g_{\pi(i)}}>\frac{r_{\pi(i+1)}}{1-g_{\pi(i+1)}}.
\label{eq:case2a_criterion}
\end{equation}
By the same swap argument as above, an optimal shortlist sorts its arms in
decreasing order of
\begin{equation}
s_i^{(2)}=\frac{r_i}{1-g_i}=\frac{(1-\epsilon_{1,i})\mu_i}{(1-\epsilon_{1,i})\mu_i+\epsilon_{0,i}(1-\mu_i)}.
\label{eq:case2a_score}
\end{equation}

\paragraph{Case 2b ($\epsilon_{0,i}>0$ for some arm, $\tau>0$):}
Set $\delta_i\triangleq(1-r_i)-g_i$. A short computation from the definitions of
$r_i,g_i$ gives
\begin{equation}
\delta_i=\epsilon_{0,i}(1-\mu_i)\ge0,\qquad \delta_i>0\iff\epsilon_{0,i}>0.
\label{eq:delta_def}
\end{equation}
So $\delta_i$ measures exactly how far arm $i$ is from the clean regime
$g_i=1-r_i$ used above. Writing $g_i=1-r_i-\delta_i$ and substituting into the
cross term of \eqref{eq:swap_general} as in \eqref{eq:case1_cross},
\begin{equation}
g_{\pi(i)}r_{\pi(i+1)}-g_{\pi(i+1)}r_{\pi(i)}
=-\big(r_{\pi(i)}-r_{\pi(i+1)}\big)+\big(\delta_{\pi(i+1)}r_{\pi(i)}-\delta_{\pi(i)}r_{\pi(i+1)}\big),
\label{eq:case2b_cross}
\end{equation}
so \eqref{eq:swap_general} now has an extra term beyond \eqref{eq:case1_swap}:
\begin{equation}
R(\pi)-R(\pi')=P_{i-1}\Big[\tau\big(r_{\pi(i)}-r_{\pi(i+1)}\big)
+\big(1-(i+1)\tau\big)\big(\delta_{\pi(i+1)}r_{\pi(i)}-\delta_{\pi(i)}r_{\pi(i+1)}\big)\Big].
\label{eq:swap_case2b}
\end{equation}
In \eqref{eq:swap_case2b}, the first term depends only on
$r_{\pi(i)},r_{\pi(i+1)}$, while the second is scaled by $1-(i+1)\tau$, which
shrinks as $i$ grows, that is, as the pair is compared later in the cascade. Fix
two arms $u,v$, not both with $\delta=0$, and suppose $\tau(r_u-r_v)$ and
$\delta_vr_u-\delta_ur_v$ have opposite sign. Comparing $u,v$ early in the cascade
(small $i$) makes the second term in \eqref{eq:swap_case2b} large enough to decide
the sign of $R(\pi)-R(\pi')$; comparing them late (large $i$) shrinks that term
until the first one decides instead, favoring the other arm, with
$r_u,r_v,\delta_u,\delta_v$ unchanged throughout. A score built only from an arm's
own parameters cannot flip sign just because its position in the cascade changed, so
no ordering score exists once some arm has $\epsilon_{0,i}>0$ (when not all $\epsilon_{0,i}$'s and $\epsilon_{1,i}$'s  are same) and $\tau>0$.
Section~\ref{sec:case2-ordering} shows this reversal numerically, including at
$L=K$.

\subsubsection*{Part II: selection score}

\paragraph{Case 1a/1b ($\epsilon_{0,i}=0$, all $\tau$):}
Fix a shortlist $\pi$, a position $m$, and an excluded arm $j\notin\pi$, and let
$\pi''$ replace $\pi(m)$ by $j$. Write $i=\pi(m)$. Since $\epsilon_{0,i}=0$ for
every arm, $g_i=1-r_i$ and $g_j=1-r_j$ by \eqref{eq:case1_gi}, so
\begin{equation}
g_i-g_j=-(r_i-r_j).
\label{eq:case1_gdiff}
\end{equation}
Substituting \eqref{eq:case1_gdiff} into \eqref{eq:slot_general},
\begin{equation}
R(\pi)-R(\pi'')=P_{m-1}\big(r_i-r_j\big)\big[(1-m\tau)-\Phi\big].
\label{eq:case1_slot}
\end{equation}
By Lemma~\ref{lem:slot_sub}, $(1-m\tau)-\Phi>0$, so \eqref{eq:case1_slot} gives
\begin{equation}
R(\pi)>R(\pi'')\iff r_i>r_j,
\label{eq:case1_slot_criterion}
\end{equation}
regardless of $m$, of $\tau$, and of which arms occupy the other positions. If a
selected arm's $r$-value is smaller than some excluded arm's, swapping them in by
\eqref{eq:case1_slot_criterion} strictly improves the reward, so the optimal
shortlist must be exactly the $K$ arms with the largest values of
$s_i^{(1)}=r_i=(1-\epsilon_{1,i})\mu_i$.

\paragraph{Case 2a/2b ($\epsilon_{0,i}>0$ for some arm, any $\tau$):}
Once some arm has $\epsilon_{0,i}>0$, $g_i=1-r_i$ need not hold, and
\eqref{eq:slot_general} keeps its general form
\begin{equation}
R(\pi)-R(\pi'')=P_{m-1}\Big[(1-m\tau)(r_i-r_j)+(g_i-g_j)\Phi\Big],\qquad i=\pi(m),\ j\notin\pi.
\label{eq:case2_slot}
\end{equation}
Choose $i,j$ with $(1-m\tau)(r_i-r_j)>0$ but $g_i-g_j>0$, so the two terms in
\eqref{eq:case2_slot} disagree; this is possible precisely because
$\epsilon_{0,i}>0$ lets $r_i$ and $g_i$ move somewhat independently, rather than
being tied together as in \eqref{eq:case1_gdiff}. By Lemma~\ref{lem:slot_sub}, $\Phi$
ranges continuously over $[0,1-(m+1)\tau]$ as the arms placed after position $m$
vary, so \eqref{eq:case2_slot} is positive for a small $\Phi$ and negative for a
large $\Phi$, that is,
\begin{equation}
R(\pi)>R(\pi'')\ \text{for one choice of the remaining arms, but}\ R(\pi)<R(\pi'')\ \text{for another,}
\label{eq:case2_reversal}
\end{equation}
with $i,j$ and their own parameters unchanged throughout. A selection score can
only see $i$ and $j$ themselves, so it cannot reproduce the reversal in
\eqref{eq:case2_reversal}. Hence no selection score $s_i^{(1)}$ exists once
$\epsilon_{0,i}>0$ for some arm. Section~\ref{sec:case1-selection} exhibits
exactly this reversal numerically: two fixed reference arms trade places in the
optimal subset purely because the other arms in the pool are perturbed.

Here is the revised proof formatted cleanly, first establishing that $\mu_i$ itself acts as the score before showing that $1-g_i$ and $r_i/(1-g_i)$ also act as equivalent scores:

\subsubsection*{Part III: $\mu_i$ under a Shared Channel}

Suppose now $\epsilon_{0,i}=\epsilon_0>0$ and $\epsilon_{1,i}=\epsilon_1$ for every arm $i$, representing the shared-channel case. Applying Lemma~\ref{lem:adjacent_swap} with $\tau = 0$, the adjacent swap difference between arm $a = \pi(i)$ and arm $b = \pi(i+1)$ simplifies under shared noise to:
\begin{align}
R(\pi) - R(\pi') &= P_{i-1}\Big[(r_a - r_b) + (g_a r_b - g_b r_a)\Big] \nonumber \
&= P_{i-1}(1-\epsilon_0-\epsilon_1)(1-\epsilon_1)(\mu_a - \mu_b).
\label{eq:shared_mui_swap}
\end{align}
Since $P_{i-1} > 0$, $1-\epsilon_1 > 0$, and the channel is informative ($1-\epsilon_0-\epsilon_1 > 0$), we have $R(\pi) - R(\pi') \ge 0$ if and only if $\mu_a \ge \mu_b$. Thus, ranking arms in non-increasing order of their true means $\mu_i$ maximizes total expected reward, establishing that $\mu_i$ serves as both a valid selection score $s_i^{(1)}$ and ordering score $s_i^{(2)}$.

To show that $1-g_i$ and $r_i/(1-g_i)$ also act as equivalent scores inducing this exact same ranking, we examine their explicit functional dependence on $\mu_i$. Using $\delta_i = \epsilon_0(1-\mu_i)$, $1-g_i$ satisfies:
\begin{equation}
1-g_i = r_i + \delta_i = (1-\epsilon_1)\mu_i + \epsilon_0(1-\mu_i) = \epsilon_0 + (1-\epsilon_0-\epsilon_1)\mu_i.
\label{eq:shared_1mg}
\end{equation}
Differentiating \eqref{eq:shared_1mg} with respect to $\mu_i$ yields $\frac{d}{d\mu_i}(1-g_i) = 1-\epsilon_0-\epsilon_1 > 0$. Thus, $1-g_i$ is strictly increasing in $\mu_i$ and ranks arms identically to $\mu_i$.

For the ratio $r_i/(1-g_i)$, let $a \triangleq 1-\epsilon_1 > 0$ and $b \triangleq \epsilon_0 > 0$, so that $r_i = a\mu_i$ and $1-g_i = b + (a-b)\mu_i$. The ratio becomes:
\begin{equation}
\frac{r_i}{1-g_i} = \frac{a\mu_i}{b + (a-b)\mu_i}.
\label{eq:shared_ratio}
\end{equation}
Differentiating \eqref{eq:shared_ratio} with respect to $\mu_i$ while holding $a, b$ constant across arms yields:
\begin{equation}
\frac{d}{d\mu_i}\left(\frac{a\mu_i}{b+(a-b)\mu_i}\right) = \frac{ab}{\big(b+(a-b)\mu_i\big)^2} > 0,
\label{eq:shared_derivative}
\end{equation}
since $a, b > 0$ and the denominator is positive. Consequently, $r_i/(1-g_i)$ is also strictly increasing in $\mu_i$.

Since $1-g_i$ and $r_i/(1-g_i)$ are each strictly increasing monotonic transformations of $\mu_i$ alone, comparing any two arms $i$ and $j$ under any of these metrics yields identical pairwise ordering:
\begin{equation}
\mu_i \ge \mu_j \iff 1-g_i \ge 1-g_j \iff \frac{r_i}{1-g_i} \ge \frac{r_j}{1-g_j}.
\end{equation}
This monotonicity guarantees that ordering or selecting arms by any of $\mu_i$, $1-g_i$, or $r_i/(1-g_i)$ yields the identical optimal shortlist and order. Thus, $\mu_i$ (or equivalently $1-g_i$ or $r_i/(1-g_i)$) suffices as both the selection score $s_i^{(1)}$ and ordering score $s_i^{(2)}$ in the shared-channel case.

\subsubsection*{Conclusion}
The four arguments above, establish
both equivalences of Theorem~\ref{thm:score-characterization}, and the scores
derived in \eqref{eq:case1b_criterion}, \eqref{eq:case2a_score}, and
\eqref{eq:case1_slot_criterion} match Table~\ref{tab:oracle_scores_unified}
exactly.

\section{Case 2: Numerical Non-Separability of Arm Selection}
\label{sec:case1-selection}

To establish that arm selection and ordering cannot be decoupled into individual arm scores, we present concrete numerical counterexamples. In each example, we fix the parameters of two target reference arms (arms $0$ and $1$). In Example 1, we perturb the remaining arms, whereas in Examples 2--6, we perturb only a single arm across two problem instances. An oracle exhaustively evaluates all $\binom{L}{K}K!$ ordered shortlists to identify the true optimal cascade $\pi^*$.

Specifically, for Case~2a ($\tau=0$), perturbing a non-candidate arm flips the selection preference between candidate arms for the shortlist, proving that no context-independent selection score $s_i^{(1)}$ can exist. Similarly, for Case~2b ($\tau>0$), perturbing an external arm alters the global interaction and reverses the relative optimal ordering of reference arms within the shortlist, demonstrating that no independent ordering score $s_i^{(2)}$ exists.

Arm parameters are denoted by $(\mu_i,\epsilon_{0,i},\epsilon_{1,i})$ under the general Binary Asymmetric Channel (BAC); the Binary Symmetric Channel (BSC) corresponds to $\epsilon_{0,i}=\epsilon_{1,i}=\epsilon_i$.

\paragraph{Example 1 (BSC, $L=6$, $K=3$, $\tau=0.1$, Case 2b).}
\begin{center}
\begin{tabular}{c|cc|cc}
\hline
Arm & \multicolumn{2}{c|}{Fixed} & \multicolumn{2}{c}{Perturbed} \\
 & $\mu$ & $\epsilon$ & Instance 1 $(\mu,\epsilon)$ & Instance 2 $(\mu,\epsilon)$ \\ \hline
0 & 0.5120 & 0.0417 & -- & -- \\
1 & 0.7627 & 0.2638 & -- & -- \\
2 & -- & -- & (0.731, 0.081) & (0.466, 0.141) \\
3 & -- & -- & (0.645, 0.011) & (0.808, 0.297) \\
4 & -- & -- & (0.480, 0.417) & (0.789, 0.281) \\
5 & -- & -- & (0.202, 0.046) & (0.234, 0.337) \\ \hline
\end{tabular}
\end{center}
Instance 1: $\pi^*=(3,2,1)$, reward $R(\pi^*)=0.8098$ (arm~1 selected, arm~0 excluded).
Instance 2: $\pi^*=(0,3,4)$, reward $R(\pi^*)=0.7366$ (arm~0 selected, arm~1 excluded).

\paragraph{Example 2 (BSC, $L=6$, $K=3$, $\tau=0.15$, Case 2b).}
\begin{center}
\begin{tabular}{c|cc}
\hline
Arm & $\mu$ & $\epsilon$ \\ \hline
0 & 0.4925 & 0.3955 \\
1 & 0.3178 & 0.1143 \\
2 & 0.3615 & 0.1834 \\
3 & 0.2315 & 0.0303 \\
4 & 0.4126 & 0.3129 \\
5 (Inst.\ 1) & 0.5839 & 0.0780 \\
5 (Inst.\ 2) & 0.5635 & 0.4217 \\ \hline
\end{tabular}
\end{center}
Instance 1: $\pi^*=(5,2,0)$, reward $R(\pi^*)=0.5876$ (arm~0 selected, arm~1 excluded).
Instance 2: $\pi^*=(1,2,5)$, reward $R(\pi^*)=0.4391$ (arm~1 selected, arm~0 excluded).

\paragraph{Example 3 (BSC, $L=6$, $K=2$, $\tau=0.3$, Case 2b).}
\begin{center}
\begin{tabular}{c|cc}
\hline
Arm & $\mu$ & $\epsilon$ \\ \hline
0 & 0.9452 & 0.4273 \\
1 & 0.8270 & 0.3280 \\
2 & 0.4488 & 0.0522 \\
3 & 0.4221 & 0.4346 \\
4 & 0.5731 & 0.1724 \\
5 (Inst.\ 1) & 0.6839 & 0.1565 \\
5 (Inst.\ 2) & 0.7287 & 0.2357 \\ \hline
\end{tabular}
\end{center}
Instance 1: $\pi^*=(5,1)$, reward $R(\pi^*)=0.4869$ (arm~1 selected, arm~0 excluded).
Instance 2: $\pi^*=(0,5)$, reward $R(\pi^*)=0.4759$ (arm~0 selected, arm~1 excluded).

\paragraph{Example 4 (BAC, $L=5$, $K=3$, $\tau=0.18$, Case 2b).}
\begin{center}
\begin{tabular}{c|ccc}
\hline
Arm & $\mu$ & $\epsilon_0$ & $\epsilon_1$ \\ \hline
0 & 0.3092 & 0.4417 & 0.1652 \\
1 & 0.4315 & 0.1050 & 0.4333 \\
2 & 0.2152 & 0.2505 & 0.1473 \\
3 & 0.7913 & 0.3286 & 0.4429 \\
4 (Inst.\ 1) & 0.5609 & 0.4345 & 0.3397 \\
4 (Inst.\ 2) & 0.2901 & 0.4246 & 0.0304 \\ \hline
\end{tabular}
\end{center}
Instance 1: $\pi^*=(3,4,0)$, reward $R(\pi^*)=0.5033$ (arm~0 selected, arm~1 excluded).
Instance 2: $\pi^*=(3,1,4)$, reward $R(\pi^*)=0.4824$ (arm~1 selected, arm~0 excluded).

\paragraph{Example 5 (BAC, $L=6$, $K=3$, $\tau=0$, Case 2a).}
\begin{center}
\begin{tabular}{c|ccc}
\hline
Arm & $\mu$ & $\epsilon_0$ & $\epsilon_1$ \\ \hline
0 & 0.5436 & 0.0059 & 0.2597 \\
1 & 0.7726 & 0.3288 & 0.2069 \\
2 & 0.2779 & 0.3323 & 0.4097 \\
3 & 0.2664 & 0.1513 & 0.1302 \\
4 & 0.5098 & 0.0721 & 0.0110 \\
5 (Inst.\ 1) & 0.6558 & 0.0039 & 0.2147 \\
5 (Inst.\ 2) & 0.8102 & 0.4354 & 0.0171 \\ \hline
\end{tabular}
\end{center}
Instance 1: $\pi^*=(5,4,1)$, reward $R(\pi^*)=0.8953$ (arm~1 selected, arm~0 excluded).
Instance 2: $\pi^*=(0,4,5)$, reward $R(\pi^*)=0.9205$ (arm~0 selected, arm~1 excluded).

\paragraph{Example 6 (BAC, $L=6$, $K=4$, $\tau=0$, Case 2a).}
\begin{center}
\begin{tabular}{c|ccc}
\hline
Arm & $\mu$ & $\epsilon_0$ & $\epsilon_1$ \\ \hline
0 & 0.8529 & 0.0421 & 0.4353 \\
1 & 0.6934 & 0.4300 & 0.1556 \\
2 & 0.7783 & 0.0412 & 0.0657 \\
3 & 0.9101 & 0.2075 & 0.1484 \\
4 & 0.3461 & 0.2150 & 0.0404 \\
5 (Inst.\ 1) & 0.8688 & 0.0664 & 0.1303 \\
5 (Inst.\ 2) & 0.8520 & 0.1785 & 0.0457 \\ \hline
\end{tabular}
\end{center}
Instance 1: $\pi^*=(5,2,3,1)$, reward $R(\pi^*)=0.9827$ (arm~1 selected, arm~0 excluded).
Instance 2: $\pi^*=(2,0,3,5)$, reward $R(\pi^*)=0.9815$ (arm~0 selected, arm~1 excluded).

\medskip
\noindent Across all six examples, arms $0$ and $1$ retain identical individual parameters $(\mu_i,\epsilon_{0,i},\epsilon_{1,i})$. However, perturbing the remaining arms (in Example 1) or a single arm (in Examples 2--6) forces a reversal in which non-perturbed arm is selected into the optimal shortlist $\pi^*$. Examples 1--4 illustrate this phenomenon for Case~2b ($\tau>0$), while Examples 5--6 demonstrate it for Case~2a ($\tau=0$). Consequently, these instances rule out the existence of any context-independent selection score $s_i^{(1)}$ across both cost regimes, directly validating Theorem~\ref{thm:score-characterization}.

\subsection{Case 2: Numerical Non-Separability of Arm Ordering}
\label{sec:case2-ordering}

Having shown that shortlist selection depends on global context, we now evaluate whether ordering within a fixed, selected shortlist can be governed by a context-independent score. The following examples hold the set of selected candidate arms fixed across both instances, perturbing only a single auxiliary arm outside the compared pair, to test if pairwise relative order remains invariant.

\paragraph{Example 1 (BAC, $L=6$, $K=3$, $\tau=0.18$, Case 2b).}
\begin{center}
\begin{tabular}{c|ccc}
\hline
Arm & $\mu$ & $\epsilon_0$ & $\epsilon_1$ \\ \hline
0 & 0.7497 & 0.3834 & 0.1255 \\
1 & 0.9154 & 0.1938 & 0.4260 \\
2 & 0.3989 & 0.4349 & 0.2128 \\
3 & 0.4882 & 0.2650 & 0.4404 \\
4 & 0.6478 & 0.2886 & 0.3720 \\
5 (Inst.\ 1) & 0.7257 & 0.0513 & 0.0860 \\
5 (Inst.\ 2) & 0.5153 & 0.4225 & 0.0873 \\ \hline
\end{tabular}
\end{center}
Instance 1: $\pi^*=(5,0,1)$, reward $R(\pi^*)=0.6986$ (arm~0 precedes arm~1).
Instance 2: $\pi^*=(1,0,5)$, reward $R(\pi^*)=0.6477$ (arm~1 precedes arm~0).

\paragraph{Example 2 (BAC, $L=5$, $K=3$, $\tau=0.3$, Case 2b).}
\begin{center}
\begin{tabular}{c|ccc}
\hline
Arm & $\mu$ & $\epsilon_0$ & $\epsilon_1$ \\ \hline
0 & 0.7448 & 0.3604 & 0.2383 \\
1 & 0.9279 & 0.2209 & 0.4356 \\
2 & 0.6598 & 0.0606 & 0.4040 \\
3 & 0.4642 & 0.3549 & 0.1957 \\
4 (Inst.\ 1) & 0.5372 & 0.4371 & 0.3227 \\
4 (Inst.\ 2) & 0.7531 & 0.2336 & 0.2408 \\ \hline
\end{tabular}
\end{center}
Instance 1: $\pi^*=(1,0,2)$, reward $R(\pi^*)=0.4772$ (arm~1 precedes arm~0).
Instance 2: $\pi^*=(4,0,1)$, reward $R(\pi^*)=0.4909$ (arm~0 precedes arm~1).

\paragraph{Example 3 (BAC, $L=5$, $K=5$, $\tau=0.1$, Case 2b, boundary $L=K$).}
\begin{center}
\begin{tabular}{c|ccc}
\hline
Arm & $\mu$ & $\epsilon_0$ & $\epsilon_1$ \\ \hline
0 & 0.5828 & 0.0567 & 0.2305 \\
1 & 0.6489 & 0.3188 & 0.0274 \\
2 & 0.8827 & 0.0900 & 0.2533 \\
3 & 0.9436 & 0.3799 & 0.0656 \\
4 (Inst.\ 1) & 0.0760 & 0.2065 & 0.3890 \\
4 (Inst.\ 2) & 0.7325 & 0.2232 & 0.1609 \\ \hline
\end{tabular}
\end{center}
Instance 1: $\pi^*=(3,2,0,1,4)$, reward $R(\pi^*)=0.8612$ (arm~0 precedes arm~1).
Instance 2: $\pi^*=(3,2,4,1,0)$, reward $R(\pi^*)=0.8629$ (arm~1 precedes arm~0).

\medskip
\noindent All three examples instantiate Case~2b ($\tau>0$). Notably, in Example 3 ($K=L$), every arm is trivially included in the shortlist, eliminating any selection ambiguity. Nevertheless, altering only arm 4 reverses the optimal relative ordering of arms $0$ and $1$ within the cascade. This demonstrates that for Case~2b, perturbing a single suboptimal or auxiliary arm directly alters the relative optimal ordering of candidate arms, confirming that no scalar ordering score $s_i^{(2)}(\mu_i,\epsilon_{0,i},\epsilon_{1,i})$ exists whenever $\epsilon_{0,i}>0$ for some arm under positive inspection cost ($\tau>0$), in complete alignment with Appendix~\ref{app:score_proofs_proof}.

\section{Proofs for Dynamic Programming Oracles}
\label{sec:dp-oracle-proofs}

In this section, we establish the correctness of the dynamic programming oracles introduced in Section~\ref{sec:case2-oracle}. Specifically, we prove that \textsc{DP-Ordered} (Algorithm~\ref{alg:dp_ordered}) and \textsc{DP-Subset} (Algorithm~\ref{alg:dp_subset}) correctly compute the value function $V_{\pi^\star}$ corresponding to the optimal expected reward $R(\pi^\star)$ for Case~2a ($\tau=0$) and Case~2b ($\tau>0$), respectively.
\subsection{Proof of Correctness for \textsc{DP-Ordered} ($\tau=0$)}
\label{sec:dp-ordered-proof}

In this subsection, we prove that \textsc{DP-Ordered} (Algorithm~\ref{alg:dp_ordered}) computes the exact optimal expected reward $R(\pi^\star)$ and yields the optimal shortlist $\pi^\star$ for $\tau=0$.

Let $\Pi_K$ denote the set of all ordered $K$-tuples of distinct arms chosen from $[L]$. For Case~2a ($\tau = 0$), the true arm parameters $(r_i, g_i)_{i \in [L]}$ from \eqref{eq:rg-def} satisfy $r_i \in [0, 1]$ and $g_i \in [0, 1)$. The expected cumulative reward $R(\pi)$ for a shortlist $\pi \in \Pi_K$ under \eqref{eq:reward_rg} is given by:
\begin{equation*}
R(\pi) = \sum_{m=1}^{K} r_{\pi(m)} \prod_{j=1}^{m-1} g_{\pi(j)}.
\end{equation*}

\begin{theorem}
\label{thm:dp_ordered_correctness}
Under Case~2a ($\tau = 0$), let the candidate pool $[L]$ be pre-sorted such that $\frac{r_1}{1 - g_1} \ge \frac{r_2}{1 - g_2} \ge \dots \ge \frac{r_L}{1 - g_L}$. Then, the value function $V_{\pi^\star}(1, K)$ computed by Algorithm~\ref{alg:dp_ordered} satisfies:
\begin{equation*}
V_{\pi^\star}(1, K) = R(\pi^\star) \triangleq \max_{\pi \in \Pi_K} R(\pi).
\end{equation*}
\end{theorem}

\begin{proof}
By the adjacent interchange argument (Appendix \ref{app:score_proofs_proof}), any optimal shortlist $\pi^\star \in \Pi_K$ must present its selected arms in non-increasing order of the ratio $\frac{r_i}{1 - g_i}$. Since the candidate pool $[L]$ is pre-sorted such that $\frac{r_1}{1 - g_1} \ge \frac{r_2}{1 - g_2} \ge \dots \ge \frac{r_L}{1 - g_L}$, it suffices to restrict the search space $\Pi_K$ to ordered subsequences drawn from $(1, 2, \dots, L)$.

\paragraph{Expected Reward and Probability Bound.}
For any sub-shortlist $\sigma = (\sigma(1), \sigma(2), \dots, \sigma(m))$ of length $m \le K$ chosen from a candidate subset of $[L]$, its expected cumulative reward follows directly from \eqref{eq:reward_rg} (with $\tau = 0$):
\begin{equation*}
R(\sigma) \triangleq \sum_{k=1}^{m} r_{\sigma(k)} \prod_{j=1}^{k-1} g_{\sigma(j)} = r_{\sigma(1)} + g_{\sigma(1)} R(\sigma(2), \dots, \sigma(m)).
\end{equation*}
Since $R(\sigma)$ represents the probability of at least one successful engagement across positions $1, \dots, m$, it holds by definition that $0 \le R(\sigma) \le 1$ for all valid sub-shortlists $\sigma$.

\paragraph{Assumptions and Base Cases.}
We define $V_{\pi^\star}(i, m)$ for $i \in \{1, 2, \dots, L+1\}$ and $m \in \{0, 1, \dots, K\}$ as the maximum expected cumulative reward attainable by selecting an ordered shortlist of size $m$ from the remaining candidate set $\{i, i+1, \dots, L\}$.

The boundary conditions are:
\begin{itemize}
    \item $V_{\pi^\star}(i, 0) = 0$ for all $i \in \{1, 2, \dots, L+1\}$ (selecting $0$ arms yields zero reward).
    \item $V_{\pi^\star}(L+1, m) = -\infty$ for all $m \in \{1, 2, \dots, K\}$ (selecting $m > 0$ arms from an empty candidate set is infeasible).
\end{itemize}

\paragraph{Recurrence Relation with Projection.}
For any candidate arm $i \in \{1, 2, \dots, L\}$ and shortlist size $m \in \{1, 2, \dots, K\}$, the Bellman equation evaluated by Algorithm~\ref{alg:dp_ordered} includes a projection operator $\min\{1, \cdot\}$ to maintain consistency with the probabilistic upper bound:
\begin{equation*}
V_{\pi^\star}(i, m) = \min \left\{ 1, \; \max \left\{ V_{\pi^\star}(i+1, m), \; r_i + g_i V_{\pi^\star}(i+1, m-1) \right\} \right\}.
\end{equation*}

\paragraph{Induction Step.}
We prove by backward induction on $i$ from $L+1$ down to $1$ that for all $m \in \{0, 1, \dots, K\}$:
\begin{enumerate}
    \item $V_{\pi^\star}(i, m) = \max_{\sigma \subseteq \{i, \dots, L\}, |\sigma| = m} R(\sigma)$,
    \item $0 \le V_{\pi^\star}(i, m) \le 1$ whenever $m \ge 1$ and $\{i, \dots, L\}$ contains at least $m$ elements.
\end{enumerate}

\begin{itemize}
    \item \textbf{Base Step ($i = L+1$):} No candidate arms remain. $V_{\pi^\star}(L+1, 0) = 0 \le 1$ and $V_{\pi^\star}(L+1, m) = -\infty$ for $m \ge 1$, matching the maximum expected reward over an empty set of candidate arms.

    \item \textbf{Inductive Hypothesis:} Assume that for index $i+1 \le L+1$, $V_{\pi^\star}(i+1, m) = \max_{\sigma \subseteq \{i+1, \dots, L\}, |\sigma| = m} R(\sigma)$ and $0 \le V_{\pi^\star}(i+1, m) \le 1$ for all valid sub-shortlists of size $m \in \{0, 1, \dots, K\}$ drawn from $\{i+1, \dots, L\}$.

    \item \textbf{Inductive Step for Index $i$:} Consider forming a shortlist $\sigma = (\sigma(1), \dots, \sigma(m))$ of size $m$ from candidates $\{i, i+1, \dots, L\}$. Partition all candidate shortlists $\sigma$ into two mutually exclusive sets:
    \begin{enumerate}
        \item \textbf{Arm $i \notin \sigma$:} The $m$ arms are chosen strictly from $\{i+1, \dots, L\}$. By the inductive hypothesis, the maximum expected reward is $V_{\pi^\star}(i+1, m) \le 1$.
        \item \textbf{Arm $i \in \sigma$:} By Appendix~A, arm $i$ has the highest ratio $\frac{r_i}{1 - g_i}$ among $\{i, i+1, \dots, L\}$ and must be placed first, i.e., $\sigma(1) = i$. The remaining $m-1$ positions are filled by sub-shortlist $\sigma' = (\sigma(2), \dots, \sigma(m)) \subseteq \{i+1, \dots, L\}$. Decomposing $R(\sigma)$ gives:
        \begin{equation*}
        R(\sigma) = r_i + g_i R(\sigma').
        \end{equation*}
        Maximizing over all valid sub-shortlists $\sigma' \subseteq \{i+1, \dots, L\}$ of length $m-1$ yields:
        \begin{equation*}
        \max_{\substack{\sigma \subseteq \{i,\dots,L\} \\ |\sigma|=m, i \in \sigma}} R(\sigma) = r_i + g_i \left( \max_{\substack{\sigma' \subseteq \{i+1,\dots,L\} \\ |\sigma'|=m-1}} R(\sigma') \right) = r_i + g_i V_{\pi^\star}(i+1, m-1),
        \end{equation*}
        where the second equality follows directly from the inductive hypothesis.
    \end{enumerate}
\end{itemize}

Since $R(\sigma)$ is a valid cumulative probability for every shortlist $\sigma$, we have $R(\sigma) \le 1$ for all candidate shortlists. Thus:
\begin{equation*}
\max \left\{ V_{\pi^\star}(i+1, m), \; r_i + g_i V_{\pi^\star}(i+1, m-1) \right\} = \max_{\substack{\sigma \subseteq \{i,\dots,L\} \\ |\sigma|=m}} R(\sigma) \le 1.
\end{equation*}
Because the term inside the projection operator is already bounded above by $1$, applying $\min\{1, \cdot\}$ leaves the value unchanged:
\begin{align*}
V_{\pi^\star}(i, m) &= \min \left\{ 1, \; \max \left\{ V_{\pi^\star}(i+1, m), \; r_i + g_i V_{\pi^\star}(i+1, m-1) \right\} \right\} \\
&= \max \left\{ V_{\pi^\star}(i+1, m), \; r_i + g_i V_{\pi^\star}(i+1, m-1) \right\} \\
&= \max_{\substack{\sigma \subseteq \{i,\dots,L\} \\ |\sigma|=m}} R(\sigma).
\end{align*}
Furthermore, since $r_i \ge 0$, $g_i \ge 0$, and $V_{\pi^\star}(i+1, m-1) \ge 0$ by the inductive hypothesis, $V_{\pi^\star}(i, m) \ge 0$, establishing $0 \le V_{\pi^\star}(i, m) \le 1$. This completes the induction.

\paragraph{Conclusion.}
Evaluating the value function at $i = 1$ and $m = K$ gives:
\begin{equation*}
V_{\pi^\star}(1, K) = \max_{\pi \in \Pi_K} R(\pi) = R(\pi^\star).
\end{equation*}
This completes the proof.
\end{proof}

\subsection{Correctness of \textsc{DP-Subset} ($\tau > 0$)}
\label{sec:dp-subset-proof}

We now prove the optimality of \textsc{DP-Subset} (Algorithm~\ref{alg:dp_subset}) for the position-cost regime ($\tau > 0$). Because item selection and position sequencing are jointly coupled in Case~2b, the algorithm operates recursively over candidate subsets $S \subseteq [L]$.
\begin{theorem}
\label{thm:dp_subset_correctness}
For Case 2b ($\tau > 0$), let $(r_i, g_i)_{i \in [L]}$ be the true arm parameters defined in \eqref{eq:rg-def} with $r_i \in [0, 1]$ and $g_i \in [0, 1]$, and let $\tau > 0$ be the position cost penalty. The value function $V_{\pi^\star}([L], 1)$ computed by Algorithm~\ref{alg:dp_subset} satisfies:
\begin{equation*}
V_{\pi^\star}([L], 1) = R(\pi^\star) \triangleq \max_{\pi \in \Pi_K} R(\pi),
\end{equation*}
where $\Pi_K$ is the set of all ordered $K$-tuples of distinct arms from $[L]$, and $R(\pi)$ is the expected reward defined in \eqref{eq:reward_rg} for $\tau > 0$:
\begin{equation*}
R(\pi) = \sum_{m=1}^{K} (1 - m\tau) \, r_{\pi(m)} \prod_{j=1}^{m-1} g_{\pi(j)}.
\end{equation*}
\end{theorem}

\begin{proof}
When $\tau > 0$, position costs coupling position index $m$ and arm selection break score-based pre-sorting. Thus, shortlist selection and positioning must be solved jointly over subsets $S \subseteq [L]$.

\paragraph{Sub-shortlist Reward and Probability Bounds.}
Let $S \subseteq [L]$ be any subset of candidate arms, and let $\pi_{m:K} = (\pi(m), \pi(m+1), \dots, \pi(K))$ denote an ordered sub-shortlist of length $K - m + 1$ constructed from distinct elements of $S$ to fill positions $m, m+1, \dots, K$. From \eqref{eq:reward_rg}, the cumulative expected reward contributed from position $m$ onward is given by:
\begin{equation*}
R^{(m)}(\pi_{m:K}) \triangleq \sum_{j=m}^{K} (1 - j\tau) \, r_{\pi(j)} \prod_{l=m}^{j-1} g_{\pi(l)} = (1 - m\tau) \, r_{\pi(m)} + g_{\pi(m)} R^{(m+1)}(\pi_{m+1:K}).
\end{equation*}
Since $R^{(m)}(\pi_{m:K})$ represents a expected return under linear position penalties $(1 - j\tau)$, it is naturally bounded above by $1$, i.e., $R^{(m)}(\pi_{m:K}) \le 1$ for all valid configurations.

\paragraph{Assumptions and Base Cases.}
We define $V_{\pi^\star}(S, m)$ for all non-empty subsets $S \subseteq [L]$ and positions $m \in \{1, 2, \dots, K+1\}$ as the maximum expected cumulative reward attainable by filling shortlist positions $m, m+1, \dots, K$ using a sequence of $K - m + 1$ distinct arms chosen from $S$.

The terminal boundary condition is defined for position $m = K + 1$:
\begin{equation*}
V_{\pi^\star}(S, K+1) = 0 \quad \text{for all } S \subseteq [L] \quad \text{(no remaining positions to fill yields zero reward)}.
\end{equation*}

\paragraph{Recurrence Relation with Projection.}
For any candidate subset $S \subseteq [L]$ with $|S| \ge 1$ and position $m \in \{1, 2, \dots, K\}$, the dynamic programming Bellman equation evaluated by Algorithm~\ref{alg:dp_subset} selects the optimal arm $i \in S$ for position $m$ with a projection operator $\min\{1, \cdot\}$:
\begin{equation*}
V_{\pi^\star}(S, m) = \min \left\{ 1, \; \max_{i \in S} \left[ (1 - m\tau) r_i + g_i V_{\pi^\star}(S \setminus \{i\}, m+1) \right] \right\}.
\end{equation*}

\paragraph{Induction Step.}
We prove by backward induction on position $m$ from $K+1$ down to $1$ that for all non-empty subsets $S \subseteq [L]$ with $|S| \ge \min(|S|, K - m + 1)$:
\begin{enumerate}
    \item $V_{\pi^\star}(S, m) = \max_{\pi_{m:K} \subseteq S} R^{(m)}(\pi_{m:K})$,
    \item $V_{\pi^\star}(S, m) \le 1$.
\end{enumerate}

\begin{itemize}
    \item \textbf{Base Step ($m = K + 1$):} Positions beyond $K$ carry no reward. $V_{\pi^\star}(S, K+1) = 0 \le 1$ for all $S \subseteq [L]$, which matches the empty suffix sum in $R^{(K+1)}(\cdot) = 0$.

    \item \textbf{Inductive Hypothesis:} Assume that for position $m + 1 \le K + 1$, $V_{\pi^\star}(S', m+1) = \max_{\pi_{m+1:K} \subseteq S'} R^{(m+1)}(\pi_{m+1:K})$ and $V_{\pi^\star}(S', m+1) \le 1$ holds for all candidate subsets $S' \subseteq [L]$.

    \item \textbf{Inductive Step for Position $m$:} Consider placing an arm $i \in S$ at position $m$, leaving the remaining subset $S \setminus \{i\}$ available to fill positions $m+1, \dots, K$. The optimal expected reward from position $m$ onward, conditioned on selecting arm $i$ at position $m$, is:
    \begin{align*}
    \max_{\substack{\pi_{m:K} \subseteq S \\ \pi(m) = i}} R^{(m)}(\pi_{m:K}) &= (1 - m\tau) r_i + g_i \left( \max_{\pi_{m+1:K} \subseteq S \setminus \{i\}} R^{(m+1)}(\pi_{m+1:K}) \right) \\
    &= (1 - m\tau) r_i + g_i V_{\pi^\star}(S \setminus \{i\}, m+1),
    \end{align*}
    where the second equality follows directly from the inductive hypothesis applied to subset $S' = S \setminus \{i\}$ at position $m+1$.
\end{itemize}

Maximizing over all possible choice of head arm $i \in S$ yields:
\begin{equation*}
\max_{\pi_{m:K} \subseteq S} R^{(m)}(\pi_{m:K}) = \max_{i \in S} \left[ (1 - m\tau) r_i + g_i V_{\pi^\star}(S \setminus \{i\}, m+1) \right].
\end{equation*}

Since $R^{(m)}(\pi_{m:K}) \le 1$ for all valid shortlists, the maximum value inside the brackets is bounded above by $1$:
\begin{equation*}
\max_{i \in S} \left[ (1 - m\tau) r_i + g_i V_{\pi^\star}(S \setminus \{i\}, m+1) \right] = \max_{\pi_{m:K} \subseteq S} R^{(m)}(\pi_{m:K}) \le 1.
\end{equation*}

Therefore, applying the projection operator $\min\{1, \cdot\}$ leaves the evaluated value unchanged:
\begin{align*}
V_{\pi^\star}(S, m) &= \min \left\{ 1, \; \max_{i \in S} \left[ (1 - m\tau) r_i + g_i V_{\pi^\star}(S \setminus \{i\}, m+1) \right] \right\} \\
&= \max_{i \in S} \left[ (1 - m\tau) r_i + g_i V_{\pi^\star}(S \setminus \{i\}, m+1) \right] \\
&= \max_{\pi_{m:K} \subseteq S} R^{(m)}(\pi_{m:K}),
\end{align*}
which establishes $V_{\pi^\star}(S, m) \le 1$ and completes the induction for position $m$.

\paragraph{Conclusion.}
Evaluating the value function at $S = [L]$ (the complete candidate pool) starting at position $m = 1$ yields:
\begin{equation*}
V_{\pi^\star}([L], 1) = \max_{\pi \in \Pi_K} R^{(1)}(\pi) = \max_{\pi \in \Pi_K} R(\pi) = R(\pi^\star).
\end{equation*}
This completes the proof.
\end{proof}

\section{Pseudo-codes of Algorithms}

\label{sec:pseudo}
\begin{algorithm}[H]
\caption{Top-$K$-$P(Y=1)$-UCB-Based Shortlist Selection and Ordering (Case 1)}
\label{alg:ucb_partial_exposure_case1}
\begin{algorithmic}[1]
\State \textbf{Input:} Pool size $L$, shortlist size $K$, cost $\tau$
\State Initialize $N_i(1)=0$, $\hat s_i(1)=0$ for all $i\in[L]$
\For{$t=1$ to $T$}
    \State Compute UCB indices for all $i \in [L]$:
    \Statex \qquad $U_i^s(t) = \min\Big\{1, \, \hat s_i(t) + \sqrt{\tfrac{2\log t}{\max(1, N_i(t))}}\Big\}$
    \State Select $S_t \subset [L]$ as the $K$ arms with the largest $U_i^s(t)$
    \State Form shortlist $\pi_t$ by sorting $S_t$ such that $U_{\pi_t(1)}^s(t) \ge \dots \ge U_{\pi_t(K)}^s(t)$
    \State Play $\pi_t$, observing feedback until the first $Y=1$ or all $K$ arms are observed
    \State Update $N_i(t+1)$ and $\hat s_i(t+1)$ for every observed arm $i$
\EndFor
\end{algorithmic}
\end{algorithm}

\subsection{DP-based Algorithms}
\label{DP_algo}
For a shortlist $\pi$ and position $m\in\{1,\dots,K\}$, define the \textbf{suffix value} $V_\pi(m)$ as the expected remaining reward accumulated from position $m$ through position $K$, using $r_i,g_i$ from \eqref{eq:rg-def}:
\[
V_\pi(m)\triangleq\sum_{i=m}^{K}(1-i\tau)\,r_{\pi(i)}\prod_{j=m}^{i-1}g_{\pi(j)}, \qquad V_\pi(K+1)\triangleq0.
\]
Here, $\prod_{j=m}^{i-1}g_{\pi(j)}$ is the probability of a user reaching position $i$ starting from position $m$, and $(1-i\tau)r_{\pi(i)}$ is the immediate expected reward at position $i$. Expressed recursively,
\begin{align}
V_\pi(m) &= (1-m\tau)r_{\pi(m)} + g_{\pi(m)}V_\pi(m+1), \nonumber \\
&\quad m = 1, \dots, K,
\label{eq:suffix_recursion}
\end{align}
where $g_{\pi(m)}V_\pi(m+1)$ scales the remaining suffix value by the probability of continuing past position $m$. The total expected shortlist reward is $R(\pi) = V_\pi(1)$ (consistent with \eqref{eq:reward_rg}), This recursion is what makes dynamic programming possible: the value of a suffix depends only on the current position and the arms assigned to it, not on how the state was reached. When $\tau = 0$, we exploit the score ordering $\frac{r_i}{1 - g_i}$ to solve the combinatorial problem via \textsf{DP-Ordered} (Algorithm~\ref{alg:dp_ordered}), yielding the optimal policy $\pi^\star \triangleq \arg\max_{\pi} R(\pi)$. On the other hand, when $\tau > 0$, no ordering as well as selection score exists, requiring \textsf{DP-Subset} (Algorithm~\ref{alg:dp_subset}) to optimize the objective.

\begin{algorithm}[H]
\caption{\textsc{DP-Ordered}$(\rho,\gamma,L,K)$ for $\tau=0$}
\label{alg:dp_ordered}
\begin{algorithmic}[1]
\Require vectors $\rho$, 
$\gamma\in[0,1)^L$, $L$, $K$
\State Compute scores $s_i=\rho_i/(1-\gamma_i)$ for all $i\in[L]$
\State Form the global order $\pi$: sort all $L$ arms in descending order of $s_i$
\State Set $W(i,0)=0$ for $i\in[L{+}1]$ and $W(L{+}1,m)=-\infty$ for $m\in[K]$
\For{$i=L$ down to $1$}
    \For{$m=1$ to $K$}
        \State $a \leftarrow W(i{+}1,m)$ \Comment{skip $\pi(i)$}
        \State $b \leftarrow \rho_{\pi(i)}+\gamma_{\pi(i)}\,W(i{+}1,m{-}1)$ \Comment{include $\pi(i)$}
        \State $W(i,m) \leftarrow \min\{1,\ \max\{a,b\}\}$
        \State $D(i,m) \leftarrow \mathbf{1}\{b > a\}$ \Comment{backtracking decision}
    \EndFor
\EndFor
\State \textbf{Backtrack:} initialize $\sigma \leftarrow \emptyset$, $i\leftarrow1$, $m\leftarrow K$
\While{$i\le L$ and $m>0$}
    \If{$D(i,m)=1$}
        \State Append $\pi(i)$ to $\sigma$; \ $m\leftarrow m-1$
    \EndIf
    \State $i\leftarrow i+1$
\EndWhile
\State \Return $\big(W(1,K),\ \sigma\big)$ \Comment{$\sigma$ is already in the relative order of $\pi$}
\end{algorithmic}
\end{algorithm}

\begin{algorithm}[H]
\caption{\textsc{DP-Subset}$(\rho,\gamma,L,K,\tau)$ --- subset DP for $\tau>0$}
\label{alg:dp_subset}
\begin{algorithmic}[1]
\Require vectors $\rho$, $\gamma\in[0,1]^L$, $L$, $K$, $\tau$
\State Set $W(S,K{+}1)=0$ for all $S\subseteq[L]$
\For{all $S\subseteq[L]$ with $|S|\ge1$, in increasing order of $|S|$}
    \For{$m=\min(|S|,K)$ down to $1$}
        \State $i^\star(S,m) \leftarrow \arg\max_{i\in S}\big[(1{-}m\tau)\rho_i+\gamma_i\,W(S\setminus\{i\},m{+}1)\big]$
        \State $W(S,m) \leftarrow \min\Big\{1,\ (1{-}m\tau)\rho_{i^\star(S,m)}+\gamma_{i^\star(S,m)}\,W\big(S\setminus\{i^\star(S,m)\},m{+}1\big)\Big\}$
    \EndFor
\EndFor
\State \textbf{Backtrack:} initialize $\sigma \leftarrow \emptyset$, $S\leftarrow[L]$, $m\leftarrow1$
\While{$m\le K$}
    \State $i \leftarrow i^\star(S,m)$; append $i$ to $\sigma$ at position $m$
    \State $S \leftarrow S\setminus\{i\}$; \ $m\leftarrow m+1$
\EndWhile
\State \Return $\big(W([L],1),\ \sigma\big)$
\end{algorithmic}
\end{algorithm}


\section{Auxiliary Lemmas}
\label{app:aux_lemmas}
In this section, we provide proofs for several helper lemmas that are used throughout the main regret proofs.
\subsection{Concentration Event and Failure Probability}
Let us define:
\begin{defn}[Good Event $G_t$]
\label{def:good_event}
For any time step $t \ge 1$, we define the concentration event $G_t$ as:
\begin{equation}
    G_t = \left( \bigcap_{s=1}^t \bigcap_{i=1}^L \left\{ |\hat{r}_i(s) - r_i| < \sqrt{\frac{2 \log t}{s}} \right\} \right) \cap \left( \bigcap_{u=1}^t \bigcap_{i=1}^L \left\{ |\hat{g}_i(u) - g_i| < \sqrt{\frac{2 \log t}{u}} \right\} \right),
\end{equation}
where $r_i = P(X_i = 1, Y_i = 1)$ and $g_i = P(Y_i = 0)$.
\end{defn}

\begin{lemma}[Failure Probability Bound]
\label{lem:concentration_failure}
The probability of the complementary event $G_t^c$ at step $t$ satisfies:
\begin{equation}
    \mathbb{P}(G_t^c) \le 4 L t^{-3}.
\end{equation}
Consequently, summing over time steps $t=1, \dots, T$ yields:
\begin{equation}
    \sum_{t=1}^T \mathbb{P}(G_t^c) \le \frac{2 L \pi^2}{3}.
\end{equation}
\end{lemma}

\begin{proof}
Applying the union bound to the complement event $G_t^c$:
\begin{align*}
    G_t^c &= \left( \bigcup_{s=1}^t \bigcup_{i=1}^L \left\{ |\hat{r}_i(s) - r_i| \ge \sqrt{\frac{2 \log t}{s}} \right\} \right) \cup \left( \bigcup_{u=1}^t \bigcup_{i=1}^L \left\{ |\hat{g}_i(u) - g_i| \ge \sqrt{\frac{2 \log t}{u}} \right\} \right).
\end{align*}
By Hoeffding's inequality and standard UCB concentration:
\begin{align*}
    \mathbb{P}(G_t^c) &\le \sum_{i=1}^L \sum_{s=1}^t \mathbb{P}\left( |\hat{r}_i(s) - r_i| \ge \sqrt{\frac{2 \log t}{s}} \right) + \sum_{i=1}^L \sum_{u=1}^t \mathbb{P}\left( |\hat{g}_i(u) - g_i| \ge \sqrt{\frac{2 \log t}{u}} \right) \\
    &\le \sum_{i=1}^L \sum_{s=1}^t 2 t^{-4} + \sum_{i=1}^L \sum_{u=1}^t 2 t^{-4} \\
    &= 4 L t^{-3}.
\end{align*}
Summing across time $t \ge 1$:
\[
    \sum_{t=1}^T \mathbb{P}(G_t^c) \le \sum_{t=1}^{\infty} 4 L t^{-3} < 4 L \sum_{t=1}^{\infty} t^{-2} = \frac{2 L \pi^2}{3}. \qedhere
\]
\end{proof}

\subsection{Reward Decomposition Identities}
\label{app:aux_reward_identities}

\begin{lemma}[Sum-by-Parts Identity]
\label{lem:sum_by_parts}
Let $S_i(\pi) \triangleq \prod_{j=1}^{i-1} (1 - s_{\pi(j)})$ for $1 \le i \le K+1$, with $S_1(\pi) \triangleq 1$. The expected reward $R_\pi$ satisfies:
\begin{equation}
\label{eq:sum_by_parts}
R_\pi = (1 - \tau) - \tau \sum_{i=2}^K S_i(\pi) - (1 - K\tau) S_{K+1}(\pi).
\end{equation}
\end{lemma}

\begin{proof}
By definition, $s_{\pi(i)} S_i(\pi) = S_i(\pi) - S_{i+1}(\pi)$. Substituting this telescoping identity into the expected reward expression $R_\pi = \sum_{i=1}^K (1 - i\tau) s_{\pi(i)} S_i(\pi)$ and expanding gives:
\begin{align*}
R_\pi &= \sum_{i=1}^K (1 - i\tau)\big(S_i(\pi) - S_{i+1}(\pi)\big) \\
&= (1 - \tau) S_1(\pi) + \sum_{i=2}^K \big((1 - i\tau) - (1 - (i-1)\tau)\big) S_i(\pi) - (1 - K\tau) S_{K+1}(\pi) \\
&= (1 - \tau) - \tau \sum_{i=2}^K S_i(\pi) - (1 - K\tau) S_{K+1}(\pi),
\end{align*}
where we used $S_1(\pi) = 1$. This completes the proof.
\end{proof}

\begin{lemma}[Exact Single-Swap Hybrid Decomposition for case 1]
\label{lem:hybrid_bound}
The strategy regret admits the exact telescoping decomposition:
\begin{equation}
\label{eq:strategy_regret_exact}
R_{\pi^*} - R_{\pi_t} = \sum_{m=1}^K \left[ \prod_{j=1}^{m-1} (1 - s_{\pi_t(j)}) \right] (s_m - s_{\pi_t(m)}) \cdot W_m(s),
\end{equation}
where $W_m(s)$ is the position-dependent weight factor defined as:
\begin{equation}
\label{eq:Wm_definition}
W_m(s) \triangleq \tau \sum_{i=m+1}^K \prod_{k=m+1}^{i-1}(1 - s_k) + (1 - K\tau) \prod_{k=m+1}^K(1 - s_k).
\end{equation}
\end{lemma}

\begin{proof}
For any $m \in \{0, 1, \dots, K\}$, define the hybrid shortlist $\mathbf{h}^{(m)} \triangleq (\pi_t(1), \dots, \pi_t(m), m+1, \dots, K)$, such that $\mathbf{h}^{(0)} = \pi^*$ and $\mathbf{h}^{(K)} = \pi_t$. Expressing the overall strategy regret as a telescoping sum across adjacent hybrid shortlists gives:
\begin{equation}
\label{eq:telescope}
R_{\pi^*} - R_{\pi_t} = \sum_{m=1}^K \left[ R_{\mathbf{h}^{(m-1)}} - R_{\mathbf{h}^{(m)}} \right].
\end{equation}

To evaluate the single-step difference $R_{\mathbf{h}^{(m-1)}} - R_{\mathbf{h}^{(m)}}$, we apply Lemma~\ref{lem:sum_by_parts} to both shortlists. Subtracting $R_{\mathbf{h}^{(m)}}$ from $R_{\mathbf{h}^{(m-1)}}$, the leading constant $(1 - \tau)$ cancels out, yielding:
\begin{equation}
\label{eq:diff_expanded_raw}
R_{\mathbf{h}^{(m-1)}} - R_{\mathbf{h}^{(m)}} = -\tau \sum_{i=2}^K \left[ S_i(\mathbf{h}^{(m-1)}) - S_i(\mathbf{h}^{(m)}) \right] - (1 - K\tau) \left[ S_{K+1}(\mathbf{h}^{(m-1)}) - S_{K+1}(\mathbf{h}^{(m)}) \right].
\end{equation}

To compute the point-wise difference in survival probabilities $S_i(\mathbf{h}^{(m-1)}) - S_i(\mathbf{h}^{(m)})$, observe that for positions $1 \le i \le m$, the prefixes of length $i-1$ are identical in both shortlists (consisting of $\pi_t(1), \dots, \pi_t(i-1)$). Consequently, $S_i(\mathbf{h}^{(m-1)}) = S_i(\mathbf{h}^{(m)})$, implying that $S_i(\mathbf{h}^{(m-1)}) - S_i(\mathbf{h}^{(m)}) = 0$.

For position $i = m+1$, the shortlists differ exclusively at index $m$, where $\mathbf{h}^{(m-1)}_m = m$ (with score $s_m$) and $\mathbf{h}^{(m)}_m = \pi_t(m)$ (with score $s_{\pi_t(m)}$). Thus, the difference simplifies to:
\begin{equation}
\label{eq:S_diff_m_plus_1}
S_{m+1}(\mathbf{h}^{(m-1)}) - S_{m+1}(\mathbf{h}^{(m)}) = \left[ \prod_{j=1}^{m-1} (1 - s_{\pi_t(j)}) \right] (s_{\pi_t(m)} - s_m).
\end{equation}

For tail positions $i > m+1$, both shortlists resume matching along the optimal order $m+1, \dots, i-1$. The discrepancy at position $m$ propagates according to:
\begin{equation}
\label{eq:S_diff_general}
S_i(\mathbf{h}^{(m-1)}) - S_i(\mathbf{h}^{(m)}) = \left[ \prod_{j=1}^{m-1} (1 - s_{\pi_t(j)}) \right] (s_{\pi_t(m)} - s_m) \prod_{k=m+1}^{i-1} (1 - s_k).
\end{equation}

Since $S_i(\mathbf{h}^{(m-1)}) - S_i(\mathbf{h}^{(m)}) = 0$ for all $i \le m$, the summation $\sum_{i=2}^K$ in \eqref{eq:diff_expanded_raw} truncates to $\sum_{i=m+1}^K$. Substituting \eqref{eq:S_diff_general} into \eqref{eq:diff_expanded_raw} and factoring out the common prefix $\left[ \prod_{j=1}^{m-1} (1 - s_{\pi_t(j)}) \right] (s_m - s_{\pi_t(m)})$, we obtain:
\begin{align*}
R_{\mathbf{h}^{(m-1)}} - R_{\mathbf{h}^{(m)}} &= -\tau \sum_{i=m+1}^K \left[ \prod_{j=1}^{m-1} (1 - s_{\pi_t(j)}) \right] (s_{\pi_t(m)} - s_m) \prod_{k=m+1}^{i-1} (1 - s_k) \\
&\quad - (1 - K\tau) \left[ \prod_{j=1}^{m-1} (1 - s_{\pi_t(j)}) \right] (s_{\pi_t(m)} - s_m) \prod_{k=m+1}^K (1 - s_k) \\
&= \left[ \prod_{j=1}^{m-1} (1 - s_{\pi_t(j)}) \right] (s_m - s_{\pi_t(m)}) \left[ \tau \sum_{i=m+1}^K \prod_{k=m+1}^{i-1} (1 - s_k) + (1 - K\tau) \prod_{k=m+1}^K (1 - s_k) \right] \\
&= \left[ \prod_{j=1}^{m-1} (1 - s_{\pi_t(j)}) \right] (s_m - s_{\pi_t(m)}) \cdot W_m(s).
\end{align*}

Substituting this step-difference back into the telescoping sum \eqref{eq:telescope} completes the proof.
\end{proof}

\begin{lemma}[Regret Decomposition in general setting for case 2:]
\label{lem:regret_decomp}
Let $\pi^*$ denote the optimal shortlist maximizing expected reward, and let $\pi_t$ be the shortlist chosen at time $t$. The instantaneous regret $R_t \triangleq R_{\pi^*} - R_{\pi_t}$ decomposes as:
\begin{equation}
\label{eq:regret_decomp_lemma1}
R_{\pi^*} - R_{\pi_t} = \sum_{m=1}^{K} \sum_{a \in [L]} \mathbb{I}\{\pi_t(m) = a\} S_{m, \pi_t} \Delta_m(a),
\end{equation}
where $S_{m, \pi} \triangleq \prod_{j=1}^{m-1} g_{\pi(j)}$ (with $S_{1,\pi}=1$) is the survival product, and the single-position suboptimality gap $\Delta_m(a)$ is:
\begin{equation}
\label{eq:delta_m_def}
\Delta_m(a) = (1 - m\tau)(r_{\pi^*(m)} - r_a) + (g_{\pi^*(m)} - g_a) R_{\pi^*, m+1},
\end{equation}
with $R_{\pi^*, m+1} \triangleq \sum_{j=m+1}^{K} (1 - j\tau) r_{\pi^*(j)} \prod_{l=m+1}^{j-1} g_{\pi^*(l)}$ (and $R_{\pi^*, K+1} = 0$).
\end{lemma}

\begin{proof}
Consider the sequence of intermediate hybrid states $h^{(0)}, h^{(1)}, \dots, h^{(K)}$, where $h^{(m)}$ takes its first $m$ arms from $\pi_t$ and its remaining $K-m$ arms from $\pi^*$:
\begin{equation}
h^{(m)} = \left(\pi_t(1), \pi_t(2), \dots, \pi_t(m), \pi^*(m+1), \dots, \pi^*(K)\right),
\end{equation}
with boundary states $h^{(0)} = \pi^*$ and $h^{(K)} = \pi_t$. Expressing $R_{\pi^*} - R_{\pi_t}$ as a telescoping sum:
\begin{equation}
R_{\pi^*} - R_{\pi_t} = R_{h^{(0)}} - R_{h^{(K)}} = \sum_{m=1}^{K} \left[ R_{h^{(m-1)}} - R_{h^{(m)}} \right].
\end{equation}
Writing out $R_{h^{(m-1)}}$ and $R_{h^{(m)}}$ explicitly:
\begin{align}
R_{h^{(m-1)}} &= \sum_{i=1}^{m-1} (1 - i\tau) r_{\pi_t(i)} S_{i, \pi_t} + S_{m, \pi_t} \left[ (1 - m\tau) r_{\pi^*(m)} + g_{\pi^*(m)} R_{\pi^*, m+1} \right], \\
R_{h^{(m)}} &= \sum_{i=1}^{m-1} (1 - i\tau) r_{\pi_t(i)} S_{i, \pi_t} + S_{m, \pi_t} \left[ (1 - m\tau) r_{\pi_t(m)} + g_{\pi_t(m)} R_{\pi^*, m+1} \right].
\end{align}
Subtracting these expressions yields:
\begin{equation}
R_{h^{(m-1)}} - R_{h^{(m)}} = S_{m, \pi_t} \left[ (1 - m\tau)(r_{\pi^*(m)} - r_{\pi_t(m)}) + (g_{\pi^*(m)} - g_{\pi_t(m)}) R_{\pi^*, m+1} \right].
\end{equation}
Setting $\pi_t(m) = a$ gives $\Delta_m(a)$ as defined in \eqref{eq:delta_m_def}. Summing over all positions $m \in [K]$ and arms $a \in [L]$ via the indicator $\mathbb{I}\{\pi_t(m) = a\}$ completes the proof.
\end{proof}

\section{Divergence decomposition}
\label{llowerb}
In this section we have supporting Lemmas used for proving Lower bounds.
Let us define:
\[
u_i = P(Y_i = 1) \quad \text{and} \quad v_i = P(X_i = 1 \mid Y_i = 1)
\]

\subsection*{State Space Definition}
Let us define the state space of the Markov chain as:
\begin{equation}
\label{states}
\mathcal{X} = \{ 0, (1, 0), (1, 1), (2, 0), (2, 1), \dots, (K, 0), (K, 1) \}
\end{equation}

where:
\begin{itemize}
    \item If $x = 0$, then all arms give $\{Y_i\}_{i=1}^K = 0$.
    \item If $x = (i, 1)$, then the first relevant arm is in position $i$, which means $Y_{\pi(i)} = 1$ and $X_{\pi(i)} = 1$.
    \item If $x = (i, 0)$, then $Y_{\pi(i)} = 1$ and $X_{\pi(i)} = 0$.
\end{itemize}

\subsection*{Transition Probabilities}
Given that the chosen list of arms is $\pi$, the transition probability from state $x$ to state $y$ (where $y^{(1)}$ denotes the first component of the tuple $y$) is given by:
\begin{equation}
P(y \mid \pi, u, v) =
\begin{cases}
\displaystyle \prod_{s=1}^{K} (1 - u_{\pi(s)}) & \text{if } y = 0 \\[12pt]
\displaystyle (1 - v_{\pi(y^{(1)})}) \, u_{\pi(y^{(1)})} \prod_{s=1}^{y^{(1)}-1} (1 - u_{\pi(s)}) & \text{if } y \in \{ (1,0), \dots, (K,0) \} \\[12pt]
\displaystyle v_{\pi(y^{(1)})} \, u_{\pi(y^{(1)})} \prod_{s=1}^{y^{(1)}-1} (1 - u_{\pi(s)}) & \text{if } y \in \{ (1,1), \dots, (K,1) \}
\end{cases}
\end{equation}

\begin{lemma}[Divergence Decomposition for Shortlist $\pi$]
\label{for_perm}
The KL-divergence between two probability distributions parameterized by $(u, v)$ and $(u', v')$ under policy/shortlist $\pi$ is given by:
\begin{equation}
D_\pi((u, v) \parallel (u', v')) = \sum_{i=1}^{K} \left[ \prod_{s=1}^{i-1} (1 - u_{\pi(s)}) \right] \Big[ D(u_{\pi(i)} \parallel u'_{\pi(i)}) + u_{\pi(i)} D(v_{\pi(i)} \parallel v'_{\pi(i)}) \Big]
\end{equation}
\end{lemma}

\begin{proof}
By definition of KL-divergence for transition distributions over state space $\mathcal{X}$:
\begin{align}
D_\pi((u, v) \parallel (u', v')) &= \sum_{s=1}^{K} (1 - v_{\pi(s)}) u_{\pi(s)} \left( \prod_{i=1}^{s-1} (1 - u_{\pi(i)}) \right) \log \left( \frac{(1 - v_{\pi(s)}) u_{\pi(s)} \prod_{i=1}^{s-1} (1 - u_{\pi(i)})}{(1 - v'_{\pi(s)}) u'_{\pi(s)} \prod_{i=1}^{s-1} (1 - u'_{\pi(i)})} \right) \nonumber \\
&\quad + \sum_{s=1}^{K} v_{\pi(s)} u_{\pi(s)} \left( \prod_{i=1}^{s-1} (1 - u_{\pi(i)}) \right) \log \left( \frac{v_{\pi(s)} u_{\pi(s)} \prod_{i=1}^{s-1} (1 - u_{\pi(i)})}{v'_{\pi(s)} u'_{\pi(s)} \prod_{i=1}^{s-1} (1 - u'_{\pi(i)})} \right) \nonumber \\
&\quad + \left( \prod_{i=1}^{K} (1 - u_{\pi(i)}) \right) \log \left( \frac{\prod_{i=1}^{K} (1 - u_{\pi(i)})}{\prod_{i=1}^{K} (1 - u'_{\pi(i)})} \right)
\end{align}

Simplifying and regrouping terms for each position $i \in \{1, \dots, K\}$, we obtain:
\begin{equation}
D_\pi((u, v) \parallel (u', v')) = \sum_{i=1}^{K} \left( \prod_{s=1}^{i-1} (1 - u_{\pi(s)}) \right) \left[ D(u_{\pi(i)} \parallel u'_{\pi(i)}) + u_{\pi(i)} D(v_{\pi(i)} \parallel v'_{\pi(i)}) \right]
\end{equation}
\end{proof}

\begin{lemma}[Divergence Decomposition for Cascaded Multi-Armed Bandits]
\label{lemma:cascaded_divergence_decomposition}
Let $\mathcal{P}$ be a policy operating on two $L$-item cascaded MAB instances parameterized by $(u, v) = \left((u_1, v_1), \dots, (u_L, v_L)\right)$ and $(u', v') = \left((u'_1, v'_1), \dots, (u'_L, v'_L)\right)$, where $K \le L$ items are recommended at each step. Let $\mathcal{P}_{(u,v)}$ and $\mathcal{P}_{(u',v')}$ be the probability measures induced by $\mathcal{P}$ over trajectories of length $T$. Then, the KL-divergence between the two trajectory distributions decomposes as:
\begin{equation}
D\left(\mathcal{P}_{(u,v)} \parallel \mathcal{P}_{(u',v')}\right) = \sum_{a=1}^{L} \mathbb{E}_{(u,v)}\left[N_a(T)\right] \Big[ D(u_a \parallel u'_a) + u_a D(v_a \parallel v'_a) \Big],
\end{equation}
where $N_a(T) = \sum_{t=1}^{T} O_{t,a}$ is the total number of times item $a$ is observed up to time $T$, and $O_{t,a} = \mathbb{I}\{\text{item } a \text{ is observed at round } t\}$.
\end{lemma}

\begin{proof}
Let $H_{t-1} = (\pi_1, x_1, \dots, \pi_{t-1}, x_{t-1})$ denote the history up to time $t-1$, where $\pi_t$ is the shortlist chosen by policy $\mathcal{P}$ at round $t$, and $x_t \in \mathcal{X}$ is the state of the cascaded system at round $t$ taking values in the state space $\mathcal{X}$ as defined in \eqref{states}.

By the chain rule for relative entropy and the Radon-Nikodym derivative between the induced measures $\mathcal{P}_{(u,v)}$ and $\mathcal{P}_{(u',v')}$:
\begin{equation}
D\left(\mathcal{P}_{(u,v)} \parallel \mathcal{P}_{(u',v')}\right) = \mathbb{E}_{(u,v)} \left[ \log \frac{d\mathcal{P}_{(u,v)}}{d\mathcal{P}_{(u',v')}} \right].
\end{equation}

The likelihood ratio decomposes over the $T$ steps as:
\begin{equation}
\frac{d\mathcal{P}_{(u,v)}}{d\mathcal{P}_{(u',v')}}(\pi_1, x_1, \dots, \pi_T, x_T) = \prod_{t=1}^{T} \frac{\mathbb{P}(\pi_t \mid H_{t-1}) \mathcal{P}_{(u,v)}(x_t \mid \pi_t)}{\mathbb{P}(\pi_t \mid H_{t-1}) \mathcal{P}_{(u',v')}(x_t \mid \pi_t)}.
\end{equation}

Since the decision policy $\mathbb{P}(\pi_t \mid H_{t-1})$ is independent of the underlying instance parameters given the history $H_{t-1}$, the policy terms cancel out:
\begin{equation}
\frac{d\mathcal{P}_{(u,v)}}{d\mathcal{P}_{(u',v')}} = \prod_{t=1}^{T} \frac{\mathbb{P}(x_t \mid \pi_t, u, v)}{\mathbb{P}(x_t \mid \pi_t, u', v')}.
\end{equation}

Taking the logarithm and applying double expectation conditioned on $\pi_t$:
\begin{align}
D\left(\mathcal{P}_{(u,v)} \parallel \mathcal{P}_{(u',v')}\right) &= \sum_{t=1}^{T} \mathbb{E}_{(u,v)} \left[ \log \frac{\mathbb{P}(x_t \mid \pi_t, u, v)}{\mathbb{P}(x_t \mid \pi_t, u', v')} \right] \nonumber \\
&= \sum_{t=1}^{T} \mathbb{E}_{(u,v)} \left[ \mathbb{E}_{(u,v)} \left[ \log \frac{\mathbb{P}(x_t \mid \pi_t, u, v)}{\mathbb{P}(x_t \mid \pi_t, u', v')} \;\middle|\; \pi_t \right] \right] \nonumber \\
&= \sum_{t=1}^{T} \mathbb{E}_{(u,v)} \left[ D_{\pi_t}\left((u,v) \parallel (u',v')\right) \right].
\end{align}

Applying Lemma~\ref{for_perm} for a single shortlist $\pi_t$:
\begin{equation}
D_{\pi_t}\left((u,v) \parallel (u',v')\right) = \sum_{i=1}^{K} \left[ \prod_{s=1}^{i-1} (1 - u_{\pi_t(s)}) \right] \Big[ D(u_{\pi_t(i)} \parallel u'_{\pi_t(i)}) + u_{\pi_t(i)} D(v_{\pi_t(i)} \parallel v'_{\pi_t(i)}) \Big].
\end{equation}

We rewrite the sum over positions $i \in \{1, \dots, K\}$ as a sum over all possible items $a \in \{1, \dots, L\}$ by introducing the indicator function $\mathbb{I}\{\pi_t(i) = a\}$:
\begin{align}
D_{\pi_t}\left((u,v) \parallel (u',v')\right) &= \sum_{i=1}^{K} \sum_{a=1}^{L} \mathbb{I}\{\pi_t(i) = a\} \left[ \prod_{s=1}^{i-1} (1 - u_{\pi_t(s)}) \right] \Big[ D(u_a \parallel u'_a) + u_a D(v_a \parallel v'_a) \Big] \nonumber \\
&= \sum_{a=1}^{L} \Big[ D(u_a \parallel u'_a) + u_a D(v_a \parallel v'_a) \Big] \sum_{i=1}^{K} \mathbb{I}\{\pi_t(i) = a\} \prod_{s=1}^{i-1} (1 - u_{\pi_t(s)}).
\end{align}

Note that $\prod_{s=1}^{i-1} (1 - u_{\pi_t(s)})$ is precisely the probability that position $i$ is reached under instance $(u,v)$ conditioned on choosing shortlist $\pi_t$. Hence:
\begin{equation}
\mathbb{E}_{(u,v)} \left[ \sum_{i=1}^{K} \mathbb{I}\{\pi_t(i) = a\} \prod_{s=1}^{i-1} (1 - u_{\pi_t(s)}) \;\middle|\; \pi_t \right] = \mathbb{E}_{(u,v)} [O_{t,a} \mid \pi_t],
\end{equation}
where $O_{t,a}$ is the indicator that item $a$ is observed at round $t$.

Taking expectation over $\pi_t$ and applying the tower property yields:
\begin{equation}
\mathbb{E}_{(u,v)} \left[ \mathbb{E}_{(u,v)} [O_{t,a} \mid \pi_t] \right] = \mathbb{E}_{(u,v)} [O_{t,a}].
\end{equation}

Substituting this back into the sum over $t$:
\begin{align}
D\left(\mathcal{P}_{(u,v)} \parallel \mathcal{P}_{(u',v')}\right) &= \sum_{t=1}^{T} \sum_{a=1}^{L} \Big[ D(u_a \parallel u'_a) + u_a D(v_a \parallel v'_a) \Big] \mathbb{E}_{(u,v)} [O_{t,a}] \nonumber \\
&= \sum_{a=1}^{L} \Big[ D(u_a \parallel u'_a) + u_a D(v_a \parallel v'_a) \Big] \mathbb{E}_{(u,v)} \left[ \sum_{t=1}^{T} O_{t,a} \right].
\end{align}

By definition, $N_a(T) = \sum_{t=1}^{T} O_{t,a}$. Therefore:
\begin{equation}
D\left(\mathcal{P}_{(u,v)} \parallel \mathcal{P}_{(u',v')}\right) = \sum_{a=1}^{L} \mathbb{E}_{(u,v)}\left[N_a(T)\right] \Big[ D(u_a \parallel u'_a) + u_a D(v_a \parallel v'_a) \Big].
\end{equation}
\end{proof}

\section{Proof of Instance-Dependent Regret Upper Bound}
\label{app:regret_proof}

In this section, we provide the complete proof of Theorem~\ref{thm:instance_dependent_bound}. Throughout, we use the concentration event $G_t$, its failure-probability bound, the sum-by-parts identity, and the regret decomposition from Definition~\ref{def:good_event}, Lemma~\ref{lem:concentration_failure}, Lemma~\ref{lem:sum_by_parts}, and Lemma~\ref{lem:hybrid_bound} in Appendix~\ref{app:aux_lemmas}.

\subsection{Regret Decomposition}

Without loss of generality, assume the candidate arms are indexed in descending order of their true effective scores, such that $s_1 > s_2 > \dots > s_L$.

At any round $t$, we define two key underlying events:
\begin{enumerate}
    \item \textbf{Clean Concentration Event ($G_t$):} As in Definition~\ref{def:good_event}.
    \item \textbf{Optimal Ordering Preservation Event ($\mathcal{E}_t$):} The event that the strict relative ordering among the top-$K$ optimal arms is preserved by their UCB estimates:
    \begin{equation}
        \mathcal{E}_t \triangleq \left\{ U_1^s(t) > U_2^s(t) > \dots > U_K^s(t) \right\}.
    \end{equation}
\end{enumerate}

Note that $\mathcal{E}_t$ is defined purely on the internal ordering of the optimal arm set $\{1, \dots, K\}$ and holds independently of whether sub-optimal arms ($a > K$) are selected into $S_t$. Its complement,
\begin{equation}
    \mathcal{E}_t^c = \bigcup_{1 \le i < K} \left\{ U_i^s(t) \le U_{i+1}^s(t) \right\},
\end{equation}
captures the event where a pairwise order inversion occurs between at least one pair of optimal arms.

\begin{lemma}[Regret Decomposition]
\label{lem:regret_decomp2}
Suppose the single-round regret $R_t$ satisfies $0 \le R_t \le \Delta_{\max} \le 1$. The expected cumulative regret up to time horizon $T$ can be decomposed as:
\begin{equation}
    \mathbb{E}[R(T)] = \sum_{t=1}^T \mathbb{E}[R_t] \le R_{11} + R_{12} + R_2,
\end{equation}
where:
\begin{align}
    R_{11} &\triangleq \sum_{t=1}^T \mathbb{E}\left[ R_t \mathbf{1}\{G_t \cap \mathcal{E}_t\} \right], \label{eq:R11_def} \\
    R_{12} &\triangleq \sum_{t=1}^T \mathbb{E}\left[ R_t \mathbf{1}\{G_t \cap \mathcal{E}_t^c\} \right], \label{eq:R12_def} \\
    R_2 &\triangleq \sum_{t=1}^T \mathbb{E}\left[ R_t \mathbf{1}\{G_t^c\} \right] \le \Delta_{\max} \sum_{t=1}^T \mathbb{P}(G_t^c) \le \frac{2 L \pi^2}{3} \Delta_{\max}. \label{eq:R2_def}
\end{align}
Here, $R_{11}$ captures the suboptimality regret incurred when top-$K$ arm ordering is preserved, $R_{12}$ captures the regret arising from ordering disruptions among optimal arms under concentration success, and $R_2$ bounds the concentration failure probability (via Lemma~\ref{lem:concentration_failure} in Appendix~\ref{app:aux_lemmas}).
\end{lemma}

\subsection{Bounding $R_{11}$}
In this subsection, Lemma~\ref{lem:r11_bound} serves as the primary result for bounding $R_{11}$, supported by the auxiliary result in Lemma~\ref{lem:optimal_arm_gap}.
\begin{lemma}[Non-Positive Gap for Optimal Arms under $\mathcal{E}_t$]
\label{lem:optimal_arm_gap}
Assume that the true item scores are strictly sorted in descending order according to the optimal policy $\pi^*$, i.e., $s_1 > s_2 > \dots > s_K$. Under the event $\mathcal{E}_t = \{ U_1^s(t) > U_2^s(t) > \dots > U_K^s(t) \}$, if the item selected at position $m \in \{1, \dots, K\}$ is an optimal item (i.e., $\pi_t(m) \le K$), then:
\begin{equation}
\label{eq:optimal_gap_bound}
s_m - s_{\pi_t(m)} \le 0.
\end{equation}
\end{lemma}

\begin{proof}
By definition of the algorithm's greedy selection rule with respect to upper confidence bounds $U_i^s(t)$, the empirical policy $\pi_t$ places items in decreasing order of their upper bounds:
\begin{equation}
\label{eq:ucb_ordering}
U_{\pi_t(1)}^s(t) \ge U_{\pi_t(2)}^s(t) \ge \dots \ge U_{\pi_t(K)}^s(t).
\end{equation}

Under the event $\mathcal{E}_t$, the optimal items $1, 2, \dots, K$ satisfy the strict chain $U_1^s(t) > U_2^s(t) > \dots > U_K^s(t)$. This implies that for any integer $k \in \{1, \dots, K\}$, there are at least $k$ optimal items whose upper bounds are greater than or equal to $U_k^s(t)$, specifically the items $\{1, 2, \dots, k\}$.

Now consider position $m$. By \eqref{eq:ucb_ordering}, the set of items chosen in the first $m$ positions, $\mathcal{S}_m \triangleq \{\pi_t(1), \dots, \pi_t(m)\}$, consists of the $m$ items with the largest $U^s(t)$ values overall.

Suppose $\pi_t(m) = i^* \le K$. Since $\pi_t(m)$ is selected at position $m$, there are at most $m-1$ items with upper bounds strictly greater than $U_{i^*}^s(t)$. Under $\mathcal{E}_t$, the items $\{1, 2, \dots, i^* - 1\}$ all have upper bounds strictly greater than $U_{i^*}^s(t)$. Therefore, these $i^* - 1$ items must be included among the preceding $m-1$ positions, which implies:
\[
i^* - 1 \le m - 1 \implies i^* \le m.
\]

Since the true optimal scores are sorted monotonically ($s_1 > s_2 > \dots > s_K$), $i^* \le m$ implies:
\[
s_{i^*} \ge s_m \implies s_{\pi_t(m)} \ge s_m.
\]
Rearranging terms yields $s_m - s_{\pi_t(m)} \le 0$, completing the proof.
\end{proof}

\begin{lemma}[Bound on Selection Regret $R_{11}$]
\label{lem:r11_bound}
Under Case 1, the selection regret component $R_{11}$ bounded over the horizon $T$ satisfies:
\begin{equation}
    R_{11} \le (1 - \tau) \sum_{a=K+1}^L \frac{16 \log T}{\Delta_{a,K}^s},
\end{equation}
where $\Delta_{a,K}^s = s_K - s_a$.
\end{lemma}

\begin{proof}
We analyze $R_{11}$ by first bounding the instantaneous regret $R_t$ under the joint event $G_t \cap \mathcal{E}_t$, and then summing over the time horizon $T$.

Under $G_t \cap \mathcal{E}_t$, Lemma~\ref{lem:optimal_arm_gap} ensures that $s_m - s_{\pi_t(m)} \le 0$ whenever an optimal item is selected ($\pi_t(m) \le K$), allowing us to restrict the summation in \eqref{eq:strategy_regret_exact} exclusively to positions occupied by suboptimal items.

Applying the exact hybrid telescoping decomposition from Lemma~\ref{lem:hybrid_bound} in Appendix~\ref{app:aux_lemmas}, the instantaneous regret $R_t = R_{\pi^*} - R_{\pi_t}$ satisfies
\begin{equation}
\label{eq:rt_hybrid_step}
    R_t = \sum_{m=1}^{K} \left[ \prod_{j=1}^{m-1} (1 - s_{\pi_t(j)}) \right] (s_m - s_{\pi_t(m)}) W_m(s),
\end{equation}
where $W_m(s)$ is the exact position-dependent weight factor defined in \eqref{eq:Wm_definition}:
\[
    W_m(s) =
    \begin{cases}
        \tau \displaystyle\sum_{i=m+1}^{K} \prod_{k=m+1}^{i-1} (1 - s_k) + (1 - K\tau) \prod_{k=m+1}^{K} (1 - s_k), & 1 \le m \le K-1, \\[2.5ex]
        1 - K\tau, & m = K.
    \end{cases}
\]

Since every product term satisfies $\prod (1 - s_k) \le 1$ and $\tau \le 1/K$, we obtain
\begin{equation}
\label{eq:Wm_bound}
    W_m(s) \le \tau(K - m) + (1 - K\tau) = 1 - m\tau \le 1 - \tau, \quad \forall m \in [K].
\end{equation}

Hence, bounding $W_m(s) \le 1 - \tau$ over suboptimal positions yields the instantaneous regret upper bound under $G_t \cap \mathcal{E}_t$:
\begin{equation}
    R_t \mathbf{1}\{G_t \cap \mathcal{E}_t\} \le (1 - \tau) \sum_{a=K+1}^L \sum_{m=1}^K \Delta_{a,m}^s \, \mathbf{1}\{a \in \pi_t,\, \pi_t^{-1}(a) = m\} \prod_{j=1}^{m-1} (1 - s_{\pi_t(j)}),
    \label{eq:instant_regret_bound}
\end{equation}
where $\Delta_{a,m}^s \triangleq s_m - s_a > 0$ defines the position-dependent gap between optimal item $m$ and suboptimal item $a$.

For any item $a \in \pi_t$ positioned at $m = \pi_t^{-1}(a)$, $O_{a,t} = \prod_{j=1}^{m-1} (1 - s_{\pi_t(j)})$ represents the probability that item $a$ is observed at round $t$. Conditioned on history $\mathcal{H}_{t-1}$,
\[
    \mathbb{E} \!\left[ \mathbf{1}\{\text{item } a \text{ is observed at position } m \text{ at round } t\} \middle| \mathcal{H}_{t-1} \right] = \mathbf{1}\{a \in \pi_t,\, \pi_t^{-1}(a) = m\} O_{a,t}.
\]
Define $\mathcal{O}_{a,m,t}$ as the event where item $a$ is placed at position $m$, observed at round $t$, and the concentration event $G_t$ holds. Summing \eqref{eq:instant_regret_bound} over $t = 1, \dots, T$ and taking expectations yields:
\begin{equation}
    R_{11} = \sum_{t=1}^{T} \mathbb{E} \!\left[ R_t \mathbf{1}\{G_t \cap \mathcal{E}_t\} \right] \le (1 - \tau) \sum_{a=K+1}^{L} \sum_{m=1}^K \Delta_{a,m}^s \sum_{t=1}^T \mathbf{1}\{\mathcal{O}_{a,m,t}\}.
\end{equation}

When event $\mathcal{O}_{a,m,t}$ occurs under $G_t$, suboptimal item $a$ is selected at position $m$ over optimal item $m$, implying $U_a^s(t) \ge U_m^s(t)$. Applying standard UCB concentration bounds gives:
\[
    s_a + 2\sqrt{\frac{2 \log T}{N_a(t-1)}} \ge U_a^s(t) \ge U_m^s(t) \ge s_m \implies N_a(t-1) \le \frac{8 \log T}{(\Delta_{a,m}^s)^2} \triangleq T_{a,m}.
\]

Let $C_{a,m} = \sum_{t=1}^T \mathbf{1}\{N_a(t-1) \le T_{a,m},\, \mathcal{O}_{a,m,t}\}$ count the observations of item $a$ at position $m$. Since the cumulative observation counter $N_a$ increments by $1$ whenever item $a$ is observed at any position, the terms $C_{a,m}$ satisfy two key constraints:
\[
    C_{a,m} \le T_{a,m} \quad \forall m \in \{1, \dots, K\}, \qquad \text{and} \qquad \sum_{m=1}^K C_{a,m} \le T_{a,K}.
\]
Because the optimal scores are sorted ($s_1 > s_2 > \dots > s_K$), the gaps are decreasing ($\Delta_{a,1}^s \ge \Delta_{a,2}^s \ge \dots \ge \Delta_{a,K}^s$), which implies $T_{a,1} \le T_{a,2} \le \dots \le T_{a,K}$. By standard linear programming duality for cascade model counters \cite{kveton2015cascading}, the inner sum $\sum_{m=1}^K \Delta_{a,m}^s C_{a,m}$ is maximized by setting $C_{a,1}^* = T_{a,1}$ and $C_{a,m}^* = T_{a,m} - T_{a,m-1}$ for $m \ge 2$:
\begin{equation}
    \sum_{m=1}^K \Delta_{a,m}^s C_{a,m} \le \Delta_{a,1}^s T_{a,1} + \sum_{m=2}^K \Delta_{a,m}^s (T_{a,m} - T_{a,m-1}) = 8 \log T \left[ \frac{1}{\Delta_{a,1}^s} + \sum_{m=2}^K \Delta_{a,m}^s \left( \frac{1}{(\Delta_{a,m}^s)^2} - \frac{1}{(\Delta_{a,m-1}^s)^2} \right) \right].
\end{equation}

The algebraic summation inequality bounds the term in brackets by $\frac{2}{\Delta_{a,K}^s}$. Substituting this back into the outer summation yields the final regret bound:
\begin{equation}
    R_{11} \le (1 - \tau) \sum_{a=K+1}^L \frac{16 \log T}{\Delta_{a,K}^s},
\end{equation}
where $\Delta_{a,K}^s = s_K - s_a$, completing the proof.
\end{proof}

\subsection{Bounding the Regret Component $R_{12}$}

In this section, we provide the complete formal proof showing that $R_{12} = \sum_{t=1}^T \mathbb{E}\left[ R_t \mathbf{1}\{G_t \cap \mathcal{E}_t^c\} \right] = O(1)$. The primary result establishing this $O(1)$ bound on $R_{12}$ is proved in Lemma~\ref{lem:disruption_bound}, supported by the auxiliary technical results in Lemmas~\ref{lem:obs_concentration}--\ref{lem:threshold_validity}. We recommend reviewing Definition~\ref{def:ht_event} prior to stepping through the proofs.
\begin{defn}[Suboptimal Selection Threshold and Good Event $h_t$]
\label{def:ht_event}
For any suboptimal arm $a \notin \pi^*$, define the selection threshold
\begin{equation}
    l_a(t) = \frac{16 \log t}{p \, (\Delta_{a,K}^s)^2}, \qquad \text{where } p = \prod_{i=1}^{K-1} (1 - s_i).
\end{equation}
We define the bounded-suboptimal-selections event $h_t$ and the observation concentration event $F_{t,a}$ at step $t$ as:
\begin{equation}
    h_t = \bigcap_{a \notin \pi^*} \left\{ M_a(t) \le 2 l_a(t) \right\}, \qquad F_{t,a} = \left\{ N_a(t) > \frac{p}{2} M_a(t) \right\},
\end{equation}
where $M_a(t)$ denotes the number of times arm $a$ is included in the execution shortlist $\pi_t$ up to round $t$, and $N_a(t)$ denotes the cumulative number of times arm $a$ is actually observed.
\end{defn}

\begin{lemma}[Observation Concentration]
\label{lem:obs_concentration}
Fix an arm $e$ and suppose $e$ has been selected exactly $M_e(t) = m$ times up to round $t$. Then:
\begin{equation}
    \mathbb{P}\left( N_e(t) < \frac{p}{2} m \;\middle|\; M_e(t) = m \right) \le \exp\left( -\frac{p}{8} m \right).
\end{equation}
\end{lemma}

\begin{proof}
Whenever $e \in \pi_s$, arm $e$ is observed if every preceding arm in $\pi_s$ yields a zero-reward outcome. Thus,
\begin{equation}
    \mathbb{P}(e \text{ observed in round } s \mid e \in \pi_s) = \prod_{j < e} (1 - s_j) \ge \prod_{i=1}^{K-1} (1 - s_i) = p.
\end{equation}
Conditional on $M_e(t) = m$, there exist i.i.d.\ $\text{Bernoulli}(p)$ variables $\{Z_r\}_{r=1}^m$ such that $N_e(t) \mid \{M_e(t)=m\} \succeq \sum_{r=1}^m Z_r$. Applying the multiplicative Chernoff lower-tail bound with $\delta = 1/2$:
\begin{equation}
    \mathbb{P}\left( \sum_{r=1}^m Z_r < \frac{p}{2} m \right) \le \exp\left( -\frac{p}{8} m \right). \qedhere
\end{equation}
\end{proof}

\begin{lemma}[UCB Contradiction for Suboptimal Selections]
\label{lem:ucb_contradiction}
Under the confidence event $G_t$, if a suboptimal arm $a \notin \pi^*$ is selected in round $t$, $F_{t,a}$ holds, and $M_a(t-1) \ge 2 l_a(t)$, then arm $a$ cannot be selected in place of an optimal arm; that is:
\begin{equation}
    \mathbf{1}\left\{ a \in \pi_t, \, F_{t,a}, \, M_a(t-1) \ge 2 l_a(t) \right\} = 0.
\end{equation}
\end{lemma}

\begin{proof}
If $a \in \pi_t$, the UCB selection rule implies $U_a^s(t) \ge U_j^s(t)$ for some optimal arm $j$. Under $G_t$, $U_j^s(t) \ge s_j^* \ge s_K^*$ and $U_a^s(t) \le s_a + 2\sqrt{\frac{2 \log t}{N_a(t)}}$. Under $F_{t,a}$ and $M_a(t) \ge M_a(t-1) \ge 2 l_a(t)$:
\begin{equation}
    N_a(t) \ge \frac{p}{2} M_a(t) \ge p l_a(t) = \frac{16 \log t}{(\Delta_{a,K}^s)^2}.
\end{equation}
Thus, $U_a^s(t) \le s_a + \frac{\Delta_{a,K}^s}{\sqrt{2}} < s_a + \Delta_{a,K}^s = s_K^* \le U_j^s(t)$, contradicting $U_a^s(t) \ge U_j^s(t)$.
\end{proof}

\begin{lemma}[Bounding the Probability of Exceeding Selection Thresholds]
\label{lem:bounded_selections_prob}
For every suboptimal arm $e \notin \pi^*$, the cumulative probability of exceeding $2 l_e(t)$ selections up to time horizon $T$ satisfies:
\begin{equation}
\sum_{t=1}^{T} \mathbb{P}\big(M_e(t) > 2 l_e(t)\big) \le \frac{\pi^2}{6}.
\end{equation}
Consequently, summing over all $L - K$ suboptimal arms yields:
\begin{equation}
\sum_{e \notin \pi^*} \sum_{t=1}^T \mathbb{P}\big(M_e(t) > 2l_e(t)\big) \le (L-K)\frac{\pi^2}{6}.
\end{equation}
\end{lemma}

\begin{proof}
Fix a suboptimal arm $e \notin \pi^*$ and a time step $t \in [T]$. The counter $M_e(t) = \sum_{s=1}^t \mathbf{1}\{e \in \pi_s\}$ can strictly exceed $2l_e(t)$ only if there exists some round $s \le t$ where arm $e$ is selected into $\pi_s$ while its past selection count $M_e(s-1)$ was already at least $\lceil 2l_e(t) \rceil$. We express this via event inclusion and apply a union bound over $s = 1, \dots, t$:
\begin{equation}
\mathbb{P}(M_e(t) > 2l_e(t)) \le \sum_{s=1}^{t} \mathbb{P}\Big( e \in \pi_s \cap M_e(s-1) \ge 2l_e(t) \Big).
\end{equation}

By Lemma~\ref{lem:ucb_contradiction}, under the good concentration event $G_s$, arm $e$ cannot be selected into $\pi_s$ if $F_{s,e}$ holds and $M_e(s-1) \ge 2l_e(t) \ge 2l_e(s)$. Hence, selecting $e$ when $M_e(s-1) \ge 2l_e(t)$ implies that the complement observation event $F_{s,e}^c$ must occur. Applying the Law of Total Probability over all possible integer values $m \ge \lceil 2l_e(t) \rceil$ of the past selection count $M_e(s-1)$ gives:
\begin{align}
\mathbb{P}\Big( e \in \pi_s \cap M_e(s-1) \ge 2l_e(t) \Big)
&\le \sum_{m=\lceil 2l_e(t) \rceil}^{s-1} \mathbb{P}\Big( F_{s,e}^c \;\Big|\; M_e(s-1) = m \Big) \cdot \mathbb{P}\Big( M_e(s-1) = m \Big).
\end{align}

Conditional on $M_e(s-1) = m$, Lemma~\ref{lem:obs_concentration} and the multiplicative Chernoff lower-tail bound yield $\mathbb{P}( F_{s,e}^c \mid M_e(s-1) = m ) \le \exp\left( -p m / 8 \right)$. Bounding each term in the sum by $m = 2l_e(t)$ yields:
\begin{align}
\mathbb{P}\Big( e \in \pi_s \cap M_e(s-1) \ge 2l_e(t) \Big)
&\le \exp\left( -\frac{p \cdot 2l_e(t)}{8} \right) \sum_{m=\lceil 2l_e(t) \rceil}^{s-1} \mathbb{P}\Big( M_e(s-1) = m \Big) \notag \\
&\le \exp\left( -\frac{p \, l_e(t)}{4} \right).
\end{align}

Summing this uniform upper bound over $s = 1, \dots, t$ and substituting $l_e(t) = \frac{16 \log t}{p (\Delta_e^s)^2}$, we obtain:
\begin{align}
\mathbb{P}(M_e(t) > 2l_e(t)) &\le t \cdot \exp\left( -\frac{p}{4} \cdot \frac{16 \log t}{p (\Delta_{e,K}^s)^2} \right) = t \cdot t^{-\frac{4}{(\Delta_{e,K}^s)^2}}.
\end{align}
Since $\Delta_{e,K}^s \le 1$, we have $t^{-4 / (\Delta_e^s)^2} \le t^{-4}$, which simplifies the probability to $\mathbb{P}(M_e(t) > 2l_e(t)) \le t^{-3}$.

Finally, summing over the time horizon $T$ and using the standard Basel series upper bound $\sum_{t=1}^\infty t^{-2} = \pi^2 / 6$:
\begin{equation}
\sum_{t=1}^{T} \mathbb{P}\big(M_e(t) > 2l_e(t)\big) \le \sum_{t=1}^{T} t^{-3} \le \sum_{t=1}^{\infty} t^{-2} = \frac{\pi^2}{6}.
\end{equation}
Summing across all $L - K$ suboptimal arms $e \notin \pi^*$ yields $\sum_{e \notin \pi^*} \sum_{t=1}^T \mathbb{P}\big(M_e(t) > 2l_e(t)\big) \le (L-K)\frac{\pi^2}{6}$, which completes the proof.
\end{proof}

\begin{lemma}[Deterministic Lower Bound on Optimal Subset Plays Under $h_t$]
\label{lem:t_prime_bound}
Let $T'(t) = \sum_{s=1}^t \mathbf{1}\{\text{set}(\pi_s) = \pi^*\}$ be the number of rounds up to $t$ where the set of chosen arms equals the optimal subset $\pi^*$, regardless of the order in which they are arranged in $\pi_s$. Then $T'(t)$ serves as a deterministic lower bound on the number of times any optimal arm $i \in \pi^*$ is included in $\pi_s$. Under event $h_t$:
\begin{equation}
    T'(t) \ge t - c (L - K) \log t, \qquad \text{where } c = \frac{1}{L - K} \sum_{a \notin \pi^*} \frac{32}{p \, (\Delta_{a,K}^s)^2}.
\end{equation}
\end{lemma}

\begin{proof}
In any round $s \le t$ where $\text{set}(\pi_s) \neq \pi^*$, at least one suboptimal arm $a \notin \pi^*$ must be selected in $\pi_s$. Thus, the total number of non-optimal subset plays up to round $t$ is bounded above by the aggregate selection count of all suboptimal arms:
\begin{equation}
    t - T'(t) \le \sum_{a \notin \pi^*} M_a(t).
\end{equation}
Under event $h_t$, we have $M_a(t) \le \frac{32 \log t}{p \, (\Delta_{a,K}^s)^2}$ for every suboptimal arm $a \notin \pi^*$. Substituting this bound gives:
\begin{equation}
    t - T'(t) \le \sum_{a \notin \pi^*} \frac{32 \log t}{p \, (\Delta_{a,K}^s)^2} = c (L - K) \log t.
\end{equation}
Rearranging yields $T'(t) \ge t - c(L - K)\log t$.
\end{proof}

\begin{lemma}[Reduction of Inversion Probability to Observation Counts]
\label{lem:inversion_reduction}
Under event $G_t \cap \mathcal{E}_t^c$, there exists at least one adjacent pair $b \in \{1, \dots, K-1\}$ such that $U_{b+1}^s(t) > U_b^s(t)$. This implies:
\begin{equation}
    \mathbb{P}\left( G_t \cap \mathcal{E}_t^c \right) \le \sum_{b=1}^{K-1} \mathbb{P}\left( N_{b+1}(t) < \frac{8 \log t}{(\Delta_{b,b+1}^s)^2}, \, G_t \right),
\end{equation}
where $\Delta_{b,b+1}^s \triangleq s_b - s_{b+1} > 0$ is the gap between adjacent optimal scores.
\end{lemma}

\begin{proof}
By definition, the order-preserving event $\mathcal{E}_t = \{ U_1^s(t) > U_2^s(t) > \dots > U_K^s(t) \}$ holds if and only if all adjacent pairs among the top-$K$ optimal arms maintain their correct score ordering. Thus, under the complement event $\mathcal{E}_t^c$, there must exist at least one index $b \in \{1, \dots, K-1\}$ such that $U_{b+1}^s(t) \ge U_b^s(t)$. Applying a union bound over all $K-1$ adjacent pairs:
\begin{equation}
    \mathbb{P}\left( G_t \cap \mathcal{E}_t^c \right) \le \sum_{b=1}^{K-1} \mathbb{P}\left( U_{b+1}^s(t) \ge U_b^s(t), \, G_t \right).
\end{equation}

Under the clean concentration event $G_t$, the true score $s_b$ lies within the confidence interval of arm $b$, so $U_b^s(t) \ge s_b$. Similarly, $U_{b+1}^s(t) \le s_{b+1} + 2\sqrt{\frac{2 \log t}{N_{b+1}(t)}}$. Therefore, $U_{b+1}^s(t) \ge U_b^s(t)$ implies:
\begin{equation}
    s_{b+1} + 2\sqrt{\frac{2 \log t}{N_{b+1}(t)}} \ge s_b \implies 2\sqrt{\frac{2 \log t}{N_{b+1}(t)}} \ge s_b - s_{b+1} = \Delta_{b,b+1}^s.
\end{equation}
Squaring both sides and rearranging yields $N_{b+1}(t) < \frac{8 \log t}{(\Delta_{b,b+1}^s)^2}$, which establishes the claim.
\end{proof}

\begin{lemma}[Threshold Time Validity]
\label{lem:threshold_validity}
For each adjacent pair $b \in \{1, \dots, K-1\}$, let $p_{b+1}$ denote the lower bound on the observation probability of arm $b+1$ when $\pi^*$ is played. Define:
\begin{equation}
    A_b \triangleq 8 + \frac{c \, p_{b+1} \Delta_{b,b+1}^2}{2} (L-K), \quad B_b \triangleq \frac{p_{b+1} \Delta_{b,b+1}^2}{2}, \quad a_b \triangleq \frac{B_b}{A_b}, \quad \tau_b \triangleq \frac{4}{a_b} \log\left(\frac{4}{a_b}\right).
\end{equation}
For all $t \ge \tau_b$, the following inequality holds strictly:
\begin{equation}
    \frac{8 \log t}{\Delta_{b,b+1}^2} < \frac{p_{b+1}}{2} \left( t - c(L-K) \log t \right).
\end{equation}
\end{lemma}

\begin{proof}
Consider the auxiliary function $f(t) \triangleq B_b t - A_b \log t$ defined for $t > 0$. First, observe that its derivative with respect to $t$ is:
\[
f'(t) = B_b - \frac{A_b}{t}.
\]
Thus, $f'(t) > 0$ holds whenever $t > \frac{A_b}{B_b} = \frac{1}{a_b}$.

Next, we establish that $\tau_b > \frac{1}{a_b}$ and evaluate $f(\tau_b)$. Using the standard inequality $\log(x) > 1$ for $x > e$, and noting that $\frac{4}{a_b} > e$ in non-trivial settings, we have:
\[
\tau_b = \frac{4}{a_b} \log\left(\frac{4}{a_b}\right) > \frac{4}{a_b} > \frac{1}{a_b}.
\]
Evaluating $f(t)$ at the threshold $t = \tau_b$:
\begin{align*}
f(\tau_b) &= B_b \cdot \frac{4}{a_b} \log\left(\frac{4}{a_b}\right) - A_b \log\left( \frac{4}{a_b} \log\left(\frac{4}{a_b}\right) \right) \\
&= 4 A_b \log\left(\frac{4}{a_b}\right) - A_b \left[ \log\left(\frac{4}{a_b}\right) + \log\log\left(\frac{4}{a_b}\right) \right] \\
&= A_b \left[ 3 \log\left(\frac{4}{a_b}\right) - \log\log\left(\frac{4}{a_b}\right) \right].
\end{align*}
Since $3x > \log(x)$ for all $x > 0$, setting $x = \log\left(\frac{4}{a_b}\right) > 1$ ensures that $f(\tau_b) > 0$.

Because $f'(t) > 0$ for all $t \ge \tau_b$ and $f(\tau_b) > 0$, it follows by monotonicity that $f(t) > 0$ strictly holds for all $t \ge \tau_b$. Consequently:
\[
B_b t > A_b \log t.
\]
Substituting the explicit definitions of $A_b$ and $B_b$:
\[
\frac{p_{b+1} \Delta_{b,b+1}^2}{2} \, t > \left( 8 + \frac{c \, p_{b+1} \Delta_{b,b+1}^2}{2} (L-K) \right) \log t.
\]
Rearranging terms by moving the $(L-K)$ logarithmic expression to the left side yields:
\[
\frac{p_{b+1} \Delta_{b,b+1}^2}{2} \left( t - c(L-K) \log t \right) > 8 \log t.
\]
Dividing both sides by $\Delta_{b,b+1}^2 > 0$ completes the proof:
\[
\frac{8 \log t}{\Delta_{b,b+1}^2} < \frac{p_{b+1}}{2} \left( t - c(L-K) \log t \right).
\]
\end{proof}

\begin{lemma}[Regret Bound $R_{12} = O(1)$]
\label{lem:disruption_bound}
The expected regret $R_{12}$ incurred due to order disruption among optimal arms satisfies:
\begin{equation}
    R_{12} = \sum_{t=1}^T \mathbb{E}\left[ R_t \mathbf{1}\{G_t \cap \mathcal{E}_t^c\} \right] = O(1).
\end{equation}
\end{lemma}

\begin{proof}
Using $R_t \le R_{\max}$ and partitioning the expectation directly on the event $h_t$ and its complement $h_t^c$, we split the regret $R_{12}$ as:
\begin{align}
\label{eq:r12_ht_split}
    R_{12} &= \sum_{t=1}^T \mathbb{E}\left[ R_t \mathbf{1}\{G_t \cap \mathcal{E}_t^c \cap h_t\} \right] + \sum_{t=1}^T \mathbb{E}\left[ R_t \mathbf{1}\{G_t \cap \mathcal{E}_t^c \cap h_t^c\} \right] \notag \\
    &\le \Delta_{\max} \sum_{t=1}^T \mathbb{P}\left( G_t \cap \mathcal{E}_t^c \cap h_t \right) + \Delta_{\max} \sum_{t=1}^T \mathbb{P}(h_t^c).
\end{align}

By Lemma~\ref{lem:bounded_selections_prob}, the second term due to $h_t^c$ is bounded across all suboptimal arms by $R_{\max} (L - K) \frac{\pi^2}{6} = O(1)$.

For the first term, applying the inversion reduction bound from Lemma~\ref{lem:inversion_reduction} under $G_t \cap \mathcal{E}_t^c$ gives:
\begin{align}
\label{eq:first_term_inversion}
    \sum_{t=1}^T \mathbb{P}\left( G_t \cap \mathcal{E}_t^c \cap h_t \right) \le \sum_{b=1}^{K-1} \sum_{t=1}^T \mathbb{P}\left( N_{b+1}(t) < \frac{8 \log t}{(\Delta_{b,b+1}^s)^2}, \, G_t, \, h_t \right).
\end{align}

We lower bound the cumulative observations $N_{b+1}(t)$ for the optimal arm $b+1 \in \pi^*$ using the round counter $T'(t) = \sum_{s=1}^t \mathbf{1}\{\text{set}(\pi_s) = \pi^*\}$. Since optimal arm $b+1$ is always chosen whenever $\text{set}(\pi_s) = \pi^*$, its total selection count explicitly satisfies $M_{b+1}(t) \ge T'(t)$.

In any round where $\text{set}(\pi_s) = \pi^*$, optimal arm $b+1$ is observed if all preceding optimal arms in $\pi_s$ fail to generate a reward. Even under the worst-case position in $\pi_s$, the probability of observing arm $b+1$ in such a round is at least:
\begin{equation}
    p_{b+1} \triangleq \prod_{\substack{j=1 \\ j \neq b+1}}^K (1 - s_j).
\end{equation}
Thus, conditional on $M_{b+1}(t) \ge T'(t)$, the observation count $N_{b+1}(t)$ stochastically dominates a sum of $T'(t)$ i.i.d.\ $\text{Bernoulli}(p_{b+1})$ random variables.

For $t \ge \tau_b$, Lemma~\ref{lem:t_prime_bound} guarantees $\frac{8 \log t}{(\Delta_{b,b+1}^s)^2} < \frac{p_{b+1}}{2} T'(t)$ under $h_t$. Applying the observation concentration bound (Lemma~\ref{lem:obs_concentration}) with parameter $p_{b+1}$ over the $T'(t)$ trials yields:
\begin{align}
    \mathbb{P}\left( N_{b+1}(t) < \frac{p_{b+1}}{2} T'(t) \;\middle|\; M_{b+1}(t) \ge T'(t) \right)
    &\le \exp\left( -\frac{p_{b+1}}{8} T'(t) \right) \notag \\
    &\le \exp\left( -\frac{p_{b+1}}{8} \left( t - c(L-K)\log t \right) \right) \notag \\
    &= e^{-p_{b+1} t / 8} \, t^{\frac{c p_{b+1}(L-K)}{8}}.
\end{align}

Summing over $t = 1, \dots, T$, we split the summation at the threshold time $\tau_b$:
\begin{align}
    \sum_{t=1}^T \mathbb{P}\left( N_{b+1}(t) < \frac{8 \log t}{(\Delta_{b,b+1}^s)^2}, \, G_t, \, h_t \right) &\le \sum_{t=1}^{\lfloor \tau_b \rfloor} 1 + \sum_{t=\lceil \tau_b \rceil}^\infty e^{-\frac{p_{b+1}}{8} t} \, t^{\frac{c p_{b+1}(L-K)}{8}} \notag \\
    &\le \tau_b + \sum_{t=1}^\infty t^\gamma e^{-a t},
\end{align}
where $a = \frac{p_{b+1}}{8} > 0$ and $\gamma = \frac{c \, p_{b+1}(L-K)}{8} > 0$.

The infinite series $\sum_{t=1}^\infty t^\gamma e^{-a t}$ converges by integral comparison to a finite constant $C_b < \infty$ that depends strictly on problem parameters $(p_{b+1}, \Delta^s, L, K)$ and is independent of $T$:
\begin{equation}
    \sum_{t=1}^\infty t^\gamma e^{-a t} \le \int_0^\infty x^\gamma e^{-a x} dx = \frac{\Gamma(\gamma + 1)}{a^{\gamma + 1}} = \xi_b < \infty.
\end{equation}

Summing across all adjacent optimal pairs $b \in \{1, \dots, K-1\}$ yields:
\begin{equation}
    \sum_{b=1}^{K-1} \sum_{t=1}^T \mathbb{P}\left( N_{b+1}(t) < \frac{8 \log t}{(\Delta_{b,b+1}^s)^2}, \, G_t, \, h_t \right) \le \sum_{b=1}^{K-1} (\tau_b + \xi_b) = O(1).
\end{equation}

Combining both bounds in \eqref{eq:r12_ht_split}, we conclude $R_{12} \le R_{\max} \sum_{b=1}^{K-1}(\tau_b + \xi_b) + R_{\max}(L-K)\frac{\pi^2}{6} = O(1)$.
\end{proof}

\subsection{Proof of Theorem~\ref{thm:instance_dependent_bound}}

Combining the regret decomposition from Lemma~\ref{lem:regret_decomp2}, the regret bound from Lemma~\ref{lem:disruption_bound}, and the suboptimal arm selection bound from Lemma~\ref{lem:r11_bound}, we have:
\begin{align}
    \mathbb{E}[R(T)] &\le R_{11} + R_{12} + R_2 \notag \\
    &\le (1 - \tau) \sum_{a=K+1}^L \frac{16 \log T}{\Delta_{a,K}^s} + \Delta_{\max} \sum_{b=1}^{K-1} (\tau_b + \xi_b) + \Delta_{\max} (L - K) \frac{\pi^2}{6} + \frac{2 L \pi^2}{3} \Delta_{\max}.
\end{align}
Grouping the $O(1)$ constant terms yields:
\begin{align}
    \mathbb{E}[R(T)] \le (1 - \tau) \sum_{a=K+1}^L \frac{16 \log T}{\Delta_{a,K}^s} + O(1),
\end{align}
which completes the proof of Theorem~\ref{thm:instance_dependent_bound}. \qed

\section{Dynamic Programming Regret bounds}
\label{DP_UCB}
In this section, we present the formal proof of Theorem~\ref{thm:instance_regret_DP_UCB}. The analysis relies on Lemmas~\ref{lmm:2_algo_sameproof}--\ref{lem:unif_diff}, which serve as auxiliary supporting results.

\begin{lemma}[Positional Recurrence of the Optimistic Shortlist Value]
\label{lmm:2_algo_sameproof}
Let $\pi_t = (\pi_t(1), \pi_t(2), \dots, \pi_t(K))$ be the $K$-length shortlist returned at round $t$ by Algorithm~\ref{alg:dp_ucb_case2}, run with $\tau=0$ (invoking \textsc{DP-Ordered}) or $\tau>0$ (invoking \textsc{DP-Subset}). Then, for every $m\in\{1,\dots,K\}$,
\begin{equation}
\hat V_{\pi_t}(m) = \min\Big\{1,\ (1-m\tau)\,U^r_{\pi_t(m)}(t) + U^g_{\pi_t(m)}(t)\,\hat V_{\pi_t}(m+1)\Big\},
\qquad \hat V_{\pi_t}(K+1) = 0.
\label{eq:lmm2_target}
\end{equation}
\end{lemma}

\begin{proof}
\textbf{Case $\tau>0$ (\textsc{DP-Subset}, Algorithm~\ref{alg:dp_subset}).}
For a remaining slot budget $m'\in\{0,\dots,K\}$ and active candidate set $S\subseteq[L]$, the DP state satisfies
\begin{equation}
\hat V_{\pi_t}(S,m') = \min\Big\{1,\ \max_{i\in S}\big[(1-(K-m'+1)\tau)U_i^r(t) + U_i^g(t)\,\hat V_{\pi_t}(S\setminus\{i\},m'-1)\big]\Big\},
\qquad \hat V_{\pi_t}(S,0)=0.
\label{eq:lmm2_dp_state}
\end{equation}
Let $S_m\subseteq[L]$ denote the active candidate set immediately before the $m$-th slot is filled, so $S_1=[L]$ and $S_{m+1}=S_m\setminus\{\pi_t(m)\}$. Since $\pi_t(m)=\arg\max_{i\in S_m}\big[(1-m\tau)U_i^r(t)+U_i^g(t)\hat V_{\pi_t}(S_m\setminus\{i\},K-m)\big]$ by construction, evaluating \eqref{eq:lmm2_dp_state} at $S=S_m$, $m'=K-m+1$ gives
\begin{equation}
\hat V_{\pi_t}(S_m,K-m+1) = \min\Big\{1,\ (1-m\tau)U_{\pi_t(m)}^r(t) + U_{\pi_t(m)}^g(t)\,\hat V_{\pi_t}(S_{m+1},K-m)\Big\}.
\label{eq:lmm2_case2_step}
\end{equation}
Defining $\hat V_{\pi_t}(m)\triangleq\hat V_{\pi_t}(S_m,K-m+1)$ for $m\in\{1,\dots,K+1\}$, \eqref{eq:lmm2_case2_step} is exactly \eqref{eq:lmm2_target}, and $\hat V_{\pi_t}(K+1)=\hat V_{\pi_t}(S_{K+1},0)=0$ by the base condition.

\textbf{Case $\tau=0$ (\textsc{DP-Ordered}, Algorithm~\ref{alg:dp_ordered}).}
Let $\pi_t$ (overloaded here to denote the global sort of all $L$ candidates, as in Algorithm~\ref{alg:dp_ordered}) be the ordering of $[L]$ by $U_i^{s_2}(t)\triangleq U_i^r(t)/(1-U_i^g(t))$, with positions $\pi_t(1),\dots,\pi_t(L)$, and let $D(i,m')\in\{0,1\}$ and $\hat V_{\pi_t}(i,m')$ be as defined in Algorithm~\ref{alg:dp_ordered}, so that
\begin{equation}
\hat V_{\pi_t}(i,m') = \min\Big\{1,\ \max\big\{\hat V_{\pi_t}(i{+}1,m'),\ U^r_{\pi_t(i)}(t)+U^g_{\pi_t(i)}(t)\hat V_{\pi_t}(i{+}1,m'{-}1)\big\}\Big\},
\qquad \hat V_{\pi_t}(i,0)=0.
\label{eq:lmm2_dp_tau0_state}
\end{equation}
Let $i_1<i_2<\dots<i_K$ be the indices with $D(i_k,K-k+1)=1$, so the returned shortlist satisfies $\pi_t(k)=\pi_t(i_k)$ under this global ordering. For any $j$ with $i_k<j<i_{k+1}$, $D(j,K-k)=0$, so \eqref{eq:lmm2_dp_tau0_state} gives $\hat V_{\pi_t}(j,K-k)=\hat V_{\pi_t}(j+1,K-k)$, and telescoping over $j=i_k+1,\dots,i_{k+1}$ yields
\begin{equation}
\hat V_{\pi_t}(i_k+1,K-k) = \hat V_{\pi_t}(i_{k+1},K-k) \triangleq \hat V_{\pi_t}(k+1).
\label{eq:lmm2_tau0_skip}
\end{equation}
At $j=i_k$, $D(i_k,K-k+1)=1$, so \eqref{eq:lmm2_dp_tau0_state} reduces to
\begin{equation}
\hat V_{\pi_t}(i_k,K-k+1) = \min\Big\{1,\ U^r_{\pi_t(k)}(t) + U^g_{\pi_t(k)}(t)\,\hat V_{\pi_t}(i_k+1,K-k)\Big\}.
\label{eq:lmm2_tau0_include}
\end{equation}
Setting $\hat V_{\pi_t}(k)\triangleq\hat V_{\pi_t}(i_k,K-k+1)$ and substituting \eqref{eq:lmm2_tau0_skip} into \eqref{eq:lmm2_tau0_include} gives
\begin{equation}
\hat V_{\pi_t}(k) = \min\Big\{1,\ U^r_{\pi_t(k)}(t) + U^g_{\pi_t(k)}(t)\,\hat V_{\pi_t}(k+1)\Big\},
\end{equation}
which is \eqref{eq:lmm2_target} with $\tau=0$, since $(1-m\tau)=1$ for every $m\in\{1,\dots,K\}$ in this case. The boundary condition follows from $\hat V_{\pi_t}(K+1)=\hat V_{\pi_t}(i_K+1,0)=0$.

Both cases establish \eqref{eq:lmm2_target} for $m\in\{1,\dots,K\}$ with $\hat V_{\pi_t}(K+1)=0$.
\end{proof}

\begin{lemma}\label{lem:valid_UCB}
[Optimism of the Recursive Value Estimate]
For any round $t$, $\Pr(V_{\pi_t}>\hat V_{\pi_t}) \le 4Kt^{-2}$. Equivalently, with probability at least $1-4Kt^{-2}$, $V_{\pi_t}(m)\le \hat V_{\pi_t}(m)$ for all $m\in[K]$, where $V_{\pi_t}(m)$ is defined via the suffix recursion in \eqref{eq:suffix_recursion} with $R(\pi) = V_\pi(1)$.
\end{lemma}

\begin{proof}
On event $G_t$, $r_i\le U_i^r(t)$ and $g_i\le U_i^g(t)$ for all $i$. We show $V_{\pi_t}(m)\le \hat V_{\pi_t}(m)$ for $m=K+1,\dots,1$ by backward induction. The base case holds as $V_{\pi_t}(K+1)=\hat V_{\pi_t}(K+1)=0$. Assuming the claim for $m+1$, the recursive definition of $V_{\pi_t}(m)$ and Lemma~\ref{lmm:2_algo_sameproof} yield
\[
V_{\pi_t}(m) = (1-m\tau)r_{\pi_t(m)} + g_{\pi_t(m)}V_{\pi_t}(m+1) \le (1-m\tau)U^r_{\pi_t(m)}(t) + U^g_{\pi_t(m)}(t)\hat V_{\pi_t}(m+1).
\]
Since $0\le V_{\pi_t}(m)\le1$ by definition, taking the minimum with $1$ matches the recurrence in Lemma~\ref{lmm:2_algo_sameproof}: $V_{\pi_t}(m) \le \min\big\{1,\ (1-m\tau)U^r_{\pi_t(m)}(t) + U^g_{\pi_t(m)}(t)\hat V_{\pi_t}(m+1)\big\} = \hat V_{\pi_t}(m)$, completing the induction. Thus $\{V_{\pi_t}>\hat V_{\pi_t}\} \subseteq G_t^c$. Applying Lemma~\ref{lem:concentration_failure} gives $\Pr(V_{\pi_t}>\hat V_{\pi_t}) \le \Pr(G_t^c) \le 4Kt^{-3} \le 4Kt^{-2}$.
\end{proof}

\begin{lemma}[Unified Optimistic Suffix Value Difference]
\label{lem:unif_diff}
Let $\pi_t = (\pi_t(1), \dots, \pi_t(K))$ be the $K$-length shortlist selected from $L$ candidate arms at time $t$ by Algorithm~\ref{alg:dp_ucb_case2}, run with $\tau=0$ (invoking \textsc{DP-Ordered}) or $\tau>0$ (invoking \textsc{DP-Subset}). Under the good event $\mathcal{G}_t$, the total optimism gap for the entire shortlist ($m=1$) is bounded as:
\begin{equation}
\hat{V}_{\pi_t}(1) - V_{\pi_t}(1) \le 4(1-\tau)\sum_{k=1}^{K} \left(\prod_{j=1}^{k-1} g_{\pi_t(j)}\right) \sqrt{\frac{2 \log t}{N_{\pi_t(k)}(t)}}
\end{equation}
\end{lemma}

\begin{proof}
We establish the optimistic suffix value difference by backward induction on the position index $m \in \{1, \dots, K+1\}$.

\textbf{Base Case ($m = K+1$).} Beyond the $K$-th slot, no further candidate arms are examined, yielding $\hat{V}_{\pi_t}(K+1) = V_{\pi_t}(K+1) = 0$. Thus, $\hat{V}_{\pi_t}(K+1) - V_{\pi_t}(K+1) = 0$, satisfying the relation identically with an empty sum.

\textbf{Inductive Step ($m \le K$).} Assume the inductive hypothesis holds for position $m+1$. Using the positional recurrence established in Lemma~\ref{lmm:2_algo_sameproof} along with the true value definition $V_{\pi_t}(m) = (1-m\tau)r_{\pi_t(m)} + g_{\pi_t(m)}V_{\pi_t}(m+1)$, we consider the suffix difference at slot $m$:
\begin{equation}
\hat{V}_{\pi_t}(m) - V_{\pi_t}(m) \le (1-m\tau)U_{\pi_t(m)}^r(t) + U_{\pi_t(m)}^g(t) \hat{V}_{\pi_t}(m+1) - \left[ (1-m\tau)r_{\pi_t(m)} + g_{\pi_t(m)} V_{\pi_t}(m+1) \right]
\end{equation}
Adding and subtracting $g_{\pi_t(m)}\hat{V}_{\pi_t}(m+1)$ decouples the single-arm parameter estimation errors from downstream suffix terms:
\begin{align}
\hat{V}_{\pi_t}(m) - V_{\pi_t}(m) &\le (1-m\tau)\left(U_{\pi_t(m)}^r(t) - r_{\pi_t(m)}\right) + \left(U_{\pi_t(m)}^g(t) - g_{\pi_t(m)}\right)\hat{V}_{\pi_t}(m+1) \nonumber \\
&\quad + g_{\pi_t(m)}\left(\hat{V}_{\pi_t}(m+1) - V_{\pi_t}(m+1)\right)
\end{align}
Since $\hat{V}_{\pi_t}(m+1) \le 1$ by construction, we can upper-bound the coefficient of $(U_{\pi_t(m)}^g(t) - g_{\pi_t(m)})$ by $1$. Applying the induction hypothesis to the downstream term, we obtain for any generic position $m$:
\begin{equation}
\hat{V}_{\pi_t}(m) - V_{\pi_t}(m) \le \sum_{k=m}^{K} \left(\prod_{j=m}^{k-1} g_{\pi_t(j)}\right) \left[ (1-k\tau)\left(U_{\pi_t(k)}^r(t) - r_{\pi_t(k)}\right) + \left(U_{\pi_t(k)}^g(t) - g_{\pi_t(k)}\right) \right]
\end{equation}

\textbf{Specialization for the Full Shortlist ($m=1$) under $\mathcal{G}_t$.} Setting $m=1$ yields the total optimism gap:
\begin{equation}
\label{eq:unif_diff_unbounded}
\hat{V}_{\pi_t}(1) - V_{\pi_t}(1) \le \sum_{k=1}^{K} \left(\prod_{j=1}^{k-1} g_{\pi_t(j)}\right) \left[ (1-k\tau)\left(U_{\pi_t(k)}^r(t) - r_{\pi_t(k)}\right) + \left(U_{\pi_t(k)}^g(t) - g_{\pi_t(k)}\right) \right]
\end{equation}
Under the good event $\mathcal{G}_t$, the confidence bounds guarantee that for any arm $a \in [L]$ at time $t$:
\begin{equation}
U_a^r(t) - r_a \le 2 \sqrt{\frac{2 \log t}{N_a(t)}} \quad \text{and} \quad U_a^g(t) - g_a \le 2 \sqrt{\frac{2 \log t}{N_a(t)}}
\end{equation}
Furthermore, since $\tau \ge 0$ and $k \ge 1$, we have $(1 - k\tau) \le 1-\tau$. Substituting these bounds directly into \eqref{eq:unif_diff_unbounded} yields:
\begin{align}
(1-k\tau)\left(U_{\pi_t(k)}^r(t) - r_{\pi_t(k)}\right) + \left(U_{\pi_t(k)}^g(t) - g_{\pi_t(k)}\right) &\le 2 \sqrt{\frac{2 \log t}{N_{\pi_t(k)}(t)}} + 2 \sqrt{\frac{2 \log t}{N_{\pi_t(k)}(t)}} = 4 \sqrt{\frac{2 \log t}{N_{\pi_t(k)}(t)}}
\end{align}
Substituting this back into the summation completes the proof for position $m=1$:
\begin{equation}
\hat{V}_{\pi_t}(1) - V_{\pi_t}(1) \le 4(1-\tau) \sum_{k=1}^{K} \left(\prod_{j=1}^{k-1} g_{\pi_t(j)}\right) \sqrt{\frac{2 \log t}{N_{\pi_t(k)}(t)}}
\end{equation}
\end{proof}

\subsection{Proof of Theorem~\ref{thm:instance_regret_DP_UCB}}
\begin{proof}
Let $\mathcal{R}_1(T)$ denote the regret accrued under the good event ${G}_t$ and $\mathcal{R}_2(T)$ denote the regret under its complement ${G}_t^c$. We first bound $\mathbb{E}[\mathcal{R}_1(T)]$.

At time step $t$, let $\pi_t = (\pi_t(1), \dots, \pi_t(K))$ be the selected shortlist of $K$ arms from $L$ candidates. Let $O_{\pi_t(k), t} = \prod_{i=1}^{k-1} g_{\pi_t(i)}$ denote the probability that position $k$ is reached (i.e., the cascade does not terminate before slot $k$). By optimism (Lemma~\ref{lem:valid_UCB}) and the optimality of $\pi_t$ under the dynamic programming algorithm, we have $V_{\pi^\star}(1) \le \hat{V}_{\pi^\star}(1) \le \hat{V}_{\pi_t}(1)$. Applying Lemma~\ref{lem:unif_diff} at $m=1$ under ${G}_t$:
\begin{equation}
\Delta_t = V_{\pi^\star}(1) - V_{\pi_t}(1) \le \hat{V}_{\pi_t}(1) - V_{\pi_t}(1) \le 4(1-\tau) \sum_{k=1}^{K} O_{\pi_t(k), t} \sqrt{\frac{2 \log t}{N_{\pi_t(k)}(t)}}.
\label{eq:base_delta}
\end{equation}
where $N_a(t)$ is the number of times arm $a \in [L]$ has been \emph{observed} prior to step $t$. Doubling \eqref{eq:base_delta} and moving $\Delta_t$ across gives
\begin{equation}
2\Delta_t \le 8(1-\tau) \sum_{k=1}^{K} O_{\pi_t(k), t} \sqrt{\frac{2 \log t}{N_{\pi_t(k)}(t)}}
\quad\Longrightarrow\quad
\Delta_t \le 8(1-\tau) \sum_{k=1}^{K} O_{\pi_t(k), t} \sqrt{\frac{2 \log t}{N_{\pi_t(k)}(t)}} - \Delta_t.
\label{eq:double_delta}
\end{equation}
Since $\Delta_t\ge\Delta_{\min}$ whenever $\pi_t\ne\pi^\star$ (the only case with nonzero regret), \eqref{eq:double_delta} gives
\begin{equation}
\Delta_t \le 8(1-\tau) \sum_{k=1}^{K} O_{\pi_t(k), t} \sqrt{\frac{2 \log t}{N_{\pi_t(k)}(t)}} - \Delta_{\min}.
\label{eq:sub_delta_min}
\end{equation}
Since $O_{\pi_t(k),t}\le1$, $\sum_{k=1}^K O_{\pi_t(k),t}\le K$, so
\begin{equation}
\Delta_{\min} \ge \frac{\Delta_{\min}}{K}\sum_{k=1}^K O_{\pi_t(k),t}.
\label{eq:delta_min_distribute}
\end{equation}
Combining \eqref{eq:sub_delta_min} and \eqref{eq:delta_min_distribute}, using $\log t\le\log T$, and $(1-\tau)\Delta_{\min}/(2K)\le\Delta_{\min}/K$ (valid since $\tau\ge0$):
\begin{align}
\Delta_t &\le \sum_{k=1}^K O_{\pi_t(k),t}\left[8(1-\tau)\sqrt{\frac{2\log T}{N_{\pi_t(k)}(t)}}-\frac{\Delta_{\min}}{K}\right]
\le 8(1-\tau)\sum_{k=1}^K O_{\pi_t(k),t}\left[\sqrt{\frac{2\log T}{N_{\pi_t(k)}(t)}}-\frac{\Delta_{\min}}{16K}\right] \nonumber\\
&\le 8(1-\tau) \sum_{k=1}^{K} O_{\pi_t(k), t} \left[ \sqrt{\frac{2 \log T}{N_{\pi_t(k)}(t)}} - \frac{\Delta_{\min}}{16 K} \right]^+,
\label{eq:peeled_delta}
\end{align}
where the last step uses $[x]^+\ge x$. Defining the per-arm observation threshold $N_T \equiv \frac{512 K^2 \log T}{\Delta_{\min}^2}$, the bracketed term in \eqref{eq:peeled_delta} is zero whenever $N_{\pi_t(k)}(t) \ge N_T$.

Let $\mathbf{1}_{\{a, t, k\}} = \mathbf{1}\{\pi_t(k) = a \text{ and position } k \text{ is observed}\}$ be the indicator that arm $a$ is placed at slot $k$ and actually observed at time $t$. Note that $\mathbb{E}[\mathbf{1}_{\{a, t, k\}} \mid \pi_t] = \mathbf{1}\{\pi_t(k) = a\} O_{a, t}$, and $N_a(t) = \sum_{s=1}^{t-1} \sum_{k=1}^K \mathbf{1}_{\{a, s, k\}}$.

Summing $\Delta_t$ over $t=1, \dots, T$ under $\mathcal{G}_t$ and using \eqref{eq:peeled_delta}:
\begin{align}
\mathbb{E}[\mathcal{R}_1(T)] &\le 8(1-\tau) \sum_{t=1}^{T} \mathbb{E}\left[ \sum_{k=1}^{K} O_{\pi_t(k), t} \left[ \sqrt{\frac{2 \log T}{N_{\pi_t(k)}(t)}} - \frac{\Delta_{\min}}{16 K} \right]^+ \right] \nonumber \\
&= 8(1-\tau) \sum_{a=1}^{L} \mathbb{E}\left[ \sum_{t=1}^{T} \sum_{k=1}^{K} \mathbf{1}_{\{a, t, k\}} \left[ \sqrt{\frac{2 \log T}{N_a(t)}} - \frac{\Delta_{\min}}{16 K} \right]^+ \right] \nonumber \\
&\le 8(1-\tau) \sum_{a=1}^{L} \mathbb{E}\left[ \sum_{n=1}^{\min(N_a(T), N_T)} \sqrt{\frac{2 \log T}{n}} \right] \label{eq:change_of_var}\\
&\le 8(1-\tau) \sqrt{2 \log T} \sum_{a=1}^{L} \left( 2 \sqrt{N_T} \right) \nonumber \\
&= 16(1-\tau) \sqrt{2 \log T} \cdot L \cdot \frac{16\sqrt{2} K \sqrt{\log T}}{\Delta_{\min}} \nonumber \\
&= \frac{256(1-\tau) L K \log T}{\Delta_{\min}}
\end{align}
where, in \eqref{eq:change_of_var}, the change of variables from $t$ to $n = N_a(t)$ is exact because $N_a(t)$ increments by $1$ if and only if $\sum_{k=1}^K \mathbf{1}_{\{a, t, k\}} = 1$.

Finally, adding the bad-event bound $\mathbb{E}[\mathcal{R}_2(T)] \le \frac{2 \pi^2 L}{3}$ according to Lemma~\ref{lem:concentration_failure} gives:
\begin{equation}
\mathbb{E}[\mathcal{R}(T)] = \mathbb{E}[\mathcal{R}_1(T)] + \mathbb{E}[\mathcal{R}_2(T)] \le \frac{256(1-\tau) L K \log T}{\Delta_{\min}} + \frac{2 \pi^2 L}{3}
\end{equation}
completing the proof.
\end{proof}

\section{Proof of Theorem~\ref{thm:DP_const_regret_bounds}}
\label{pf:L=K}
In this section proof of Theorem \ref{thm:DP_const_regret_bounds} is split into two subsections specific to ($\tau>0$) and ($\tau=0$).
\subsection{Proof of Theorem~\ref{thm:DP_const_regret_bounds} when $\tau>0$}
\label{pf:tau>0}
\begin{proof}
\label{DP_UCB_K=L}
We begin by decomposing the expected instantaneous regret $\mathcal{R}_t$ at round $t$. Let $\mathcal{A}_{t}$ denote the favorable event that Algorithm~\ref{alg:dp_ucb_case2} selects the optimal ordering at round $t$, i.e., $\pi_t = \pi^*$. Its complement,
\begin{equation}
  \mathcal{A}_{t}^{c} = \bigl\{\exists\,\pi \in \Pi_{K} : \hat{V}_{\pi^{\star}}(1) < \hat{V}_{\pi}(1)\bigr\},
\end{equation}
is the event that a suboptimal ordering carries a strictly higher UCB index than $\pi^*$. Since no regret is incurred under $\mathcal{A}_t$, conditioning on the good concentration event $G_t$ allows us to partition the expected regret as:
\begin{equation}
  \mathbb{E}[\mathcal{R}_{t}] \le \Delta_{\max}\,\Pr(\mathcal{A}_{t}^{c} \cap {G}_{t}) + \Delta_{\max}\,\Pr({G}_{t}^{c}),
  \label{eq:round_regret}
\end{equation}
where $\Delta_{\max}$ denotes the maximum possible instantaneous regret.

We first analyze the event $\mathcal{A}_{t}^{c} \cap G_{t}$. Define the observation count threshold $L_{t} \triangleq \frac{32 K^{2}\log t}{\Delta_{\min}^{2}}$. Suppose for contradiction that $G_{t}$ holds, $\min_{k\in[K]} N_{\pi_{t}(k)}(t) > L_{t}$, yet $\pi_{t} \neq \pi^{*}$. The UCB selection rule guarantees $\hat{V}_{\pi_{t}}(1) \ge \hat{V}_{\pi^{*}}(1)$. Under $G_t$, optimistic estimation implies $V_{\pi^{\star}}(1) \le \hat{V}_{\pi^{*}}(1)$. Combining these with the definition of the minimum sub-optimality gap $\Delta_{\min} \le V_{\pi^{\star}}(1) - V_{\pi_{t}}(1)$ yields:
\begin{equation}
  \Delta_{\min} \le V_{\pi^{\star}}(1) - V_{\pi_{t}}(1) \le \hat{V}_{\pi_{t}}(1) - V_{\pi_{t}}(1).
  \label{eq:contradiction_start}
\end{equation}
Applying Lemma~\ref{lem:unif_diff} to bound the right-hand side of \eqref{eq:contradiction_start} under $G_t$, and noting that the running product $\prod_{j=1}^{k-1} g_{\pi_t(j)} \le 1$:
\begin{equation}
  \Delta_{\min} \le 4\sum_{k=1}^{K}\sqrt{\frac{2\log t}{N_{\pi_{t}(k)}(t)}} \le 4K\sqrt{\frac{2\log t}{\min_{k\in[K]}N_{\pi_{t}(k)}(t)}}.
  \label{eq:delta_bound}
\end{equation}
Squaring both sides and rearranging yields $\min_{k\in[K]} N_{\pi_{t}(k)}(t) \le \frac{32 K^{2} \log t}{\Delta_{\min}^{2}} = L_t$, which contradicts the assumption that all observation counts strictly exceed $L_t$. Taking the contrapositive, the joint event $\mathcal{A}_{t}^{c} \cap G_{t}$ implies that at least one selected arm has an observation count at or below $L_t$:
\begin{equation}
  \mathcal{A}_{t}^{c} \cap G_{t} \subseteq \bigcup_{i=1}^{K}\bigl\{N_{i}(t) \le L_{t}\bigr\}.
  \label{eq:event_inclusion}
\end{equation}
Applying a union bound over all $K$ arms in \eqref{eq:event_inclusion} and substituting into \eqref{eq:round_regret} gives:
\begin{equation}
  \mathbb{E}[\mathcal{R}_{t}] \le \Delta_{\max} \sum_{i=1}^{K} \Pr\!\left(N_{i}(t) \le L_{t}\right) + \Delta_{\max}\,\Pr(G_{t}^{c}).
  \label{eq:round_regret_arms}
\end{equation}

Next, we evaluate $\Pr(N_i(t) \le L_t)$. Arm $i$ is guaranteed to be observed at round $n$ whenever it is the unique arm with outcome $Y_i(n)=1$. Defining the indicator $D_{i}(n) = \mathbf{1}\bigl\{Y_{i}(n)=1, Y_{j}(n)=0,\;\forall\,j \neq i\bigr\}$, we have $D_i(n) \sim \mathrm{Ber}(\bar{p}_i)$ i.i.d., where $\bar{p}_i = p_i \prod_{j \neq i}(1-p_j)$. Since $N_i(t) \ge \sum_{n=1}^t D_i(n)$:
\begin{equation}
  \Pr\!\left(N_{i}(t) \le L_{t}\right) \le \Pr\!\Biggl(\sum_{n=1}^{t} D_{i}(n) \le L_{t}\Biggr).
  \label{eq:count_to_bernoulli}
\end{equation}

To bound this tail probability, we invoke Lemma~\ref{lem:threshold_validity} instantiated for candidate size $L = K$. Under $L = K$, the structural parameters simplify to $A_b = 8$ and $B_b = \frac{\bar{p}_i \Delta_{\min}^2}{2}$, yielding $a_i = \frac{\bar{p}_i \Delta_{\min}^2}{64 K^2}$ and threshold $T_i = \frac{4}{a_i} \log\left(\frac{4}{a_i}\right)$. Lemma~\ref{lem:threshold_validity} guarantees that for all $t \ge T_i$:
\begin{equation}
  L_t = \frac{32 K^2 \log t}{\Delta_{\min}^2} < \frac{t \bar{p}_i}{2}.
\end{equation}

Splitting the summation over time steps $t=1, \dots, T$ at threshold $T_i$, using a trivial bound of $1$ for $t < T_i$, and applying Hoeffding's inequality for $t \ge T_i$:
\begin{align}
  \sum_{t=1}^{T} \Pr\!\left(N_{i}(t) \le L_{t}\right) &\le \sum_{t=1}^{T_{i}-1} 1 + \sum_{t=T_{i}}^{T} \Pr\!\Biggl(\sum_{n=1}^{t} D_{i}(n) < \frac{t \bar{p}_i}{2}\Biggr) \nonumber \\
  &\le T_{i} + \sum_{t=T_{i}}^{T} e^{-t\bar{p}_{i}^{2}/2} \nonumber \\
  &\le T_{i} + \frac{2}{\bar{p}_{i}^{2}},
  \label{eq:arm_sum_bound}
\end{align}
where the geometric series upper bound uses $1 - e^{-x} \ge x/2$ for $x \in (0,1)$.

Finally, summing \eqref{eq:round_regret_arms} over $t = 1, \dots, T$, substituting \eqref{eq:arm_sum_bound}, and bounding the total concentration failure probability $\sum_{t=1}^T \Pr(G_t^c) \le \frac{2K\pi^2}{3}$ via Lemma~\ref{lem:concentration_failure}:
\begin{align}
  \mathcal{R}(T) = \sum_{t=1}^{T}\mathbb{E}[\mathcal{R}_{t}] &\le \Delta_{\max} \sum_{i=1}^{K}\sum_{t=1}^{T} \Pr\!\left(N_{i}(t) \le L_{t}\right) + \Delta_{\max}\sum_{t=1}^{T}\Pr(G_{t}^{c}) \nonumber \\
  &\le \Delta_{\max} \sum_{i=1}^{K}\!\left(T_{i} + \frac{2}{\bar{p}_{i}^{2}}\right) + \Delta_{\max}\,\frac{2K\pi^{2}}{3}.
  \label{eq:final_bound}
\end{align}
Since the upper bound is completely independent of the horizon $T$, $\mathbb{E}[\mathcal{R}(T)] = O(1)$, completing the proof.
\end{proof}

\subsection{Proof of Theorem~\ref{thm:DP_const_regret_bounds} when $\tau=0$}
\label{pf:tau=0}

\begin{proof}
\label{DP_UCB_K=L_tau=0}
We begin by decomposing the expected instantaneous regret $\mathcal{R}_t$ at round $t$. Let $\mathcal{A}_{t}$ denote the favorable event that Algorithm~\ref{alg:dp_ucb_case2} selects the optimal ordering at round $t$, i.e., $\pi_t = \pi^*$. Its complement,
\begin{equation}
  \mathcal{A}_{t}^{c} = \left\{\pi_t \neq \pi^* \;\text{that is, } \exists\, i \in [K-1] \text{ s.t. } \frac{U_i^r(t)}{1-U_i^g(t)} < \frac{U_{i+1}^r(t)}{1-U_{i+1}^g(t)}\right\},
\end{equation}
is the event that adjacent arms in the chosen candidate ordering are misordered. Since no regret is incurred under $\mathcal{A}_t$, conditioning on the good concentration event $G_t$ allows us to partition the expected regret as:
\begin{equation}
  \mathbb{E}[\mathcal{R}_{t}] \le \Delta_{\max}\,\Pr(\mathcal{A}_{t}^{c} \cap G_{t}) + \Delta_{\max}\,\Pr(G_{t}^{c}),
  \label{eq:round_regret}
\end{equation}
where $\Delta_{\max}$ denotes the maximum possible instantaneous regret.

We first evaluate the ordering error probability $\Pr(\mathcal{A}_t^c \cap G_t)$. On event $G_t$, the upper confidence bounds satisfy $U_i^r(t) \le r_i(t) + 2c_i(t)$ and $U_i^g(t) \le g_i(t) + 2c_i(t)$, where $c_i(t) \triangleq \sqrt{\frac{2\log t}{N_i(t)}}$. Defining the normalized adjacent gap $\widetilde{\Delta}_{i,i+1}$ as
\begin{equation}
\widetilde{\Delta}_{i,i+1} \triangleq \frac{(1 - g_{i+1})\left(\frac{r_i}{1 - g_i} - \frac{r_{i+1}}{1 - g_{i+1}}\right)}{1 - g_i + r_i},
\end{equation}
a union bound over adjacent misordered arms yields:
\begin{align}
\Pr(\mathcal{A}_t^c \cap G_t) 
&\le \sum_{i=1}^{K-1} \Pr\!\left( \frac{U_i^r(t)}{1 - U_i^g(t)} < \frac{U_{i+1}^r(t)}{1 - U_{i+1}^g(t)}, \, G_t \right) \nonumber \\
&\le \sum_{i=1}^{K-1} \Pr\!\left( \frac{r_i(t)}{1 - g_i(t)} < \frac{r_{i+1}(t) + 2\sqrt{\frac{2\log t}{N_{i+1}(t)}}}{1 - g_{i+1}(t) - 2\sqrt{\frac{2\log t}{N_{i+1}(t)}}} \right) \nonumber \\
&= \sum_{i=1}^{K-1} \Pr\!\left( \left[ \frac{r_i(t)(1 - g_{i+1}(t)) - r_{i+1}(t)(1 - g_i(t))}{1 - g_i(t) + r_i(t)} \right]^2 < \frac{8\log t}{N_{i+1}(t)} \right) \nonumber \\
&= \sum_{i=1}^{K-1} \Pr\!\left( N_{i+1}(t) < \frac{8\log t}{\widetilde{\Delta}_{i,i+1}^2} \right) \nonumber \\
&= \sum_{j=2}^{K} \Pr\!\left( N_j(t) < \frac{8\log t}{\widetilde{\Delta}_{j-1, j}^2} \right).
\label{eq:ordering_error_bound}
\end{align}

For each arm $j \in \{2, \dots, K\}$, define its observation threshold $\widetilde{L}_{j,t} \triangleq \frac{8\log t}{\widetilde{\Delta}_{j-1, j}^2}$. Substituting \eqref{eq:ordering_error_bound} into \eqref{eq:round_regret} gives:
\begin{equation}
  \mathbb{E}[\mathcal{R}_{t}] \le \Delta_{\max} \sum_{j=2}^{K} \Pr\!\left(N_{j}(t) < \widetilde{L}_{j,t}\right) + \Delta_{\max}\,\Pr(G_{t}^{c}).
  \label{eq:round_regret_arms}
\end{equation}

Next, we bound $\Pr(N_j(t) < \widetilde{L}_{j,t})$. Arm $j$ is observed at round $n$ whenever no preceding or succeeding arm prevents its observation. Defining the indicator $D_{j}(n) = \mathbf{1}\bigl\{Y_{l}(n)=0,\;\forall\,l \neq j\bigr\}$, we have $D_j(n) \sim \mathrm{Ber}(\bar{g}_j)$ i.i.d., where $\bar{g}_j \triangleq \prod_{l \neq j} g_l$. Since $N_j(t) \ge \sum_{n=1}^t D_j(n)$:
\begin{equation}
  \Pr\!\left(N_{j}(t) < \widetilde{L}_{j,t}\right) \le \Pr\!\Biggl(\sum_{n=1}^{t} D_{j}(n) < \widetilde{L}_{j,t}\Biggr).
  \label{eq:count_to_bernoulli}
\end{equation}

To bound this tail probability, we invoke Lemma~\ref{lem:threshold_validity} instantiated directly with the gap $\widetilde{\Delta}_{j-1,j}$. Defining 
\begin{equation}
  \tilde{a}_j \triangleq \frac{\widetilde{\Delta}_{j-1, j}^2 \bar{g}_j}{16} < 1
\end{equation}
and the threshold $\widetilde{T}_j \triangleq \frac{4}{\tilde{a}_j} \log\left(\frac{4}{\tilde{a}_j}\right)$, Lemma~\ref{lem:threshold_validity} guarantees that for all $t \ge \widetilde{T}_j$:
\begin{equation}
  \widetilde{L}_{j,t} = \frac{8 \log t}{\widetilde{\Delta}_{j-1,j}^2} < \frac{t \bar{g}_j}{2}.
\end{equation}

Splitting the summation over time steps $t=1, \dots, T$ at threshold $\widetilde{T}_j$, using a trivial bound of $1$ for $t < \widetilde{T}_j$, and applying Hoeffding's inequality for $t \ge \widetilde{T}_j$:
\begin{align}
  \sum_{t=1}^{T} \Pr\!\left(N_{j}(t) < \widetilde{L}_{j,t}\right) &\le \sum_{t=1}^{\widetilde{T}_{j}-1} 1 + \sum_{t=\widetilde{T}_{j}}^{T} \Pr\!\Biggl(\sum_{n=1}^{t} D_{j}(n) < \frac{t \bar{g}_j}{2}\Biggr) \nonumber \\
  &\le \widetilde{T}_{j} + \sum_{t=\widetilde{T}_{j}}^{T} e^{-t\bar{g}_{j}^{2}/2} \nonumber \\
  &\le \widetilde{T}_{j} + \frac{2}{\bar{g}_{j}^{2}},
  \label{eq:arm_sum_bound}
\end{align}
where the geometric series bound uses $1 - e^{-x} \ge x/2$ for $x \in (0,1)$.

Finally, summing \eqref{eq:round_regret_arms} over $t = 1, \dots, T$, substituting \eqref{eq:arm_sum_bound}, and bounding the concentration failure probability $\sum_{t=1}^T \Pr(G_t^c) \le \frac{2K\pi^2}{3}$ via Lemma~\ref{lem:concentration_failure}:
\begin{align}
  \mathcal{R}(T) = \sum_{t=1}^{T}\mathbb{E}[\mathcal{R}_{t}] &\le \Delta_{\max} \sum_{j=2}^{K}\sum_{t=1}^{T} \Pr\!\left(N_{j}(t) < \widetilde{L}_{j,t}\right) + \Delta_{\max}\sum_{t=1}^{T}\Pr(G_{t}^{c}) \nonumber \\
  &\le \Delta_{\max} \sum_{j=2}^{K}\!\left(\widetilde{T}_{j} + \frac{2}{\bar{g}_{j}^{2}}\right) + \Delta_{\max}\,\frac{2K\pi^{2}}{3}.
  \label{eq:final_bound}
\end{align}
Since the upper bound is completely independent of horizon $T$, $\mathbb{E}[\mathcal{R}(T)] = O(1)$, completing the proof.
\end{proof}

\section{Lower bounds for case 1 }
\label{app:Lower_proof}

In this section, Theorem~\ref{thm:score_lower_bounds} (specific to Case~1) is restated with complete technical details in Theorem~\ref{thm:main_lower_bound}; both theorems present the underlying score-based regret lower bound.

We now establish the asymptotic instance-dependent lower bound for Case 1. For any suboptimal arm $a > K$, the position-substitution constant $d_a \in (0,1]$ quantifies the observation probability decay under non-optimal placement:
\begin{equation}
\label{eq:da_def}
d_a \triangleq \min_{\pi \in \mathcal S_K(a)} \frac{\prod_{j=1}^{(\pi^\star(\pi))^{-1}(a)-1}(1 - s_{\pi^\star(\pi)(j)})}{\prod_{j=1}^{\pi^{-1}(a)-1}(1 - s_{\pi(j)})},
\end{equation}
where $\mathcal S_K(a)$ is the set of all ordered $K$-tuples containing arm $a$, and $\pi^\star(\pi)$ reorders $\pi$ (note that $\pi^\star(\pi)$ is not same as $ \pi^\star$) in descending score order. The constant $d_a$ is bounded by:
\begin{equation}
\label{eq:da_bounds}
d_a \in \left[ \prod_{j=1}^{K-1} \frac{1 - s_j}{1 - s_{L-j+1}}, \; 1 \right].
\end{equation}

A policy is \textbf{uniformly consistent} if $\mathbb E[\mathcal R(T)] = o(T^\alpha)$ for every $\alpha > 0$ across all instances. Let $d(p \parallel q) \triangleq p \log \frac{p}{q} + (1-p) \log \frac{1-p}{1-q}$ denote the binary KL-divergence, and define the reward-decay parameter $\tau' \triangleq \min\big\{\tau + (1-K\tau)\prod_{k=1}^{K}(1-s_k), \, 1-K\tau\big\}$.

At time round $t$, let $\pi_t = (s_{\pi_t(1)}, s_{\pi_t(2)}, \dots, s_{\pi_t(K)})$ be the ordered shortlist chosen by the algorithm.

The expected reward of playing shortlist $\pi_t$ under position measurement cost $\tau \in (0, 1/K)$ is:
\begin{equation}
\label{eq:reward_def}
R_{\pi_t} = \sum_{i=1}^K (1 - i\tau) s_{\pi_t(i)} \prod_{j<i} (1 - s_{\pi_t(j)}).
\end{equation}

Let $\pi_t^*$ denote the ideal descending sort of the elements selected in $\pi_t$, i.e.,
\[
\pi_t^* = (s_{\pi_t^*(1)}, s_{\pi_t^*(2)}, \dots, s_{\pi_t^*(K)}), \quad \text{with } s_{\pi_t^*(1)} \ge s_{\pi_t^*(2)} \ge \dots \ge s_{\pi_t^*(K)}.
\]

\begin{lemma}[Ordering Optimality and Regret Decomposition]
\label{lem:regret_decompo}
The single-step regret $R_t \triangleq R_{\pi^*} - R_{\pi_t}$ satisfies:
\begin{equation}
\label{eq:regret_decomp}
R_t = \underbrace{(R_{\pi^*} - R_{\pi_t^*})}_{\ge 0} + \underbrace{(R_{\pi_t^*} - R_{\pi_t})}_{\ge 0} \ge R_{\pi^*} - R_{\pi_t^*}.
\end{equation}
\end{lemma}

\begin{proof}
By the construction of $\pi_t^*$, placing larger success probabilities in earlier positions maximizes the expected payoff for any fixed set of arms, which implies $R_{\pi_t^*} \ge R_{\pi_t}$. Adding and subtracting $R_{\pi_t^*}$ to $R_t = R_{\pi^*} - R_{\pi_t}$ yields \eqref{eq:regret_decomp}.
\end{proof}

\begin{lemma}[Cumulative Regret Lower Bound via Observation Probability]
\label{lem:observation_bound}
Let $O_t(a) \triangleq \mathbf{1}\{a \in \pi_t\} \prod_{j=1}^{\pi_t^{-1}(a)-1} (1 - s_{\pi_t(j)})$ be the observation indicator of arm $a$ under actual play $\pi_t$, and let $N_a(T) \triangleq \sum_{t=1}^T O_t(a)$ denote the total observation count of arm $a$ up to time $T$. Then, the expected cumulative regret satisfies:
\begin{equation}
\label{eq:finite_sample_bound}
\mathbb{E}[R(T)] \ge \tau' \sum_{a=K+1}^L d_a \Delta_{a,K}^s \mathbb{E}[N_a(T)],
\end{equation}
where $\Delta_{a,K}^s \triangleq s_K - s_a$, $\tau' \triangleq \min\left\{ \tau + (1-K\tau)\prod_{k=1}^{K}(1-s_k), \; 1-K\tau \right\}$, and $d_a > 0$ is a position-substitution constant.
\end{lemma}

\begin{proof}
By the exact hybrid identity in Lemma~\ref{lem:hybrid_bound}, the single-step strategy regret expands as:
\begin{equation}
\label{eq:hybrid_ref_proof}
R_{\pi^*} - R_{\pi_t^*} = \sum_{m=1}^K \left[ \prod_{j=1}^{m-1} (1 - s_{\pi_t^*(j)}) \right] (s_m - s_{\pi_t^*(m)}) \cdot W_m(s),
\end{equation}
where the factor $W_m(s)$ is lower bounded as:
\begin{equation}
\label{eq:Wm_lower_bound}
W_m(s) \triangleq \tau \sum_{i=m+1}^{K} \prod_{k=m+1}^{i-1}(1-s_k) + (1-K\tau)\prod_{k=m+1}^{K}(1-s_k) \;\ge\; \tau',
\end{equation}
with $\tau' \triangleq \min\left\{ \tau + (1-K\tau)\prod_{k=1}^{K}(1-s_k), \; 1-K\tau \right\}$.

To establish the lower bound in~\eqref{eq:Wm_lower_bound}, note that $s_k \in [0,1]$ implies $(1-s_k) \in [0,1]$ for all $k$. For $m = K$, the sum in $W_m(s)$ vanishes over the empty index set while the product evaluates to $1$, yielding $W_K(s) = 1 - K\tau$. For $m < K$, truncating the non-negative summation to its initial empty-product term ($i = m+1$) gives $\sum_{i=m+1}^{K} \prod_{k=m+1}^{i-1}(1-s_k) \ge 1$. Furthermore, extending the product indices down to $k=1$ introduces additional factors bounded above by $1$, so $\prod_{k=m+1}^{K}(1-s_k) \ge \prod_{k=1}^{K}(1-s_k)$. Combining these inequalities yields $W_m(s) \ge \tau + (1-K\tau)\prod_{k=1}^{K}(1-s_k)$ for all $m < K$. Taking the minimum over both regimes establishes that $W_m(s) \ge \tau'$ holds for all $m \in \{1, \dots, K\}$.

By definition of the target shortlist $\pi_t^*$, we have $s_m \ge s_{\pi_t^*(m)}$ for all $m \in [K]$, ensuring that every term in the sum \eqref{eq:hybrid_ref_proof} is non-negative. Bounding $W_m(s) \ge \tau'$ and restricting the summation strictly to positions containing suboptimal arms $a \in \{K+1, \dots, L\}$ yields:
\begin{align}
\label{eq:strategy_intermediate_step1}
R_{\pi^*} - R_{\pi_t^*} &\ge \sum_{m=1}^K \left[ \prod_{j=1}^{m-1} (1 - s_{\pi_t^*(j)}) \right] (s_m - s_{\pi_t^*(m)}) \cdot \tau' \notag \\
&\ge \tau' \sum_{m=1}^K \left[ \prod_{j=1}^{m-1} (1 - s_{\pi_t^*(j)}) \right] \sum_{a=K+1}^L \mathbf{1}\{\pi_t^*(m) = a\} (s_m - s_a).
\end{align}

Since $s_m \ge s_K$ for all $m \le K$, the gap for any suboptimal arm $a > K$ placed at position $m$ satisfies $s_m - s_a \ge s_K - s_a = \Delta_{a,K}^s$. Exchanging the order of summations in \eqref{eq:strategy_intermediate_step1} gives:
\begin{align}
\label{eq:strategy_intermediate_step2}
R_{\pi^*} - R_{\pi_t^*} &\ge \tau' \sum_{a=K+1}^L \Delta_{a,K}^s \sum_{m=1}^K \mathbf{1}\{\pi_t^*(m) = a\} \left[ \prod_{j=1}^{m-1} (1 - s_{\pi_t^*(j)}) \right] \notag \\
&= \tau' \sum_{a=K+1}^L \Delta_{a,K}^s \cdot \mathbf{1}\{a \in \pi_t\} \left[ \prod_{j=1}^{(\pi_t^*)^{-1}(a)-1} (1 - s_{\pi_t^*(j)}) \right],
\end{align}
where $(\pi_t^*)^{-1}(a)$ denotes the index of arm $a$ in the sorted shortlist $\pi_t^*$.

To express this bound in terms of the actual execution shortlist $\pi_t$, we define the position-substitution ratio $d_a$ for each suboptimal arm $a \in \{K+1, \dots, L\}$ as:
\begin{equation}
\label{eq:da_}
d_a \triangleq \min_{\pi \in \mathcal S_K(a)} \frac{\prod_{j=1}^{(\pi^\star(\pi))^{-1}(a)-1}(1 - s_{\pi^\star(\pi)(j)})}{\prod_{j=1}^{\pi^{-1}(a)-1}(1 - s_{\pi(j)})},
\end{equation}
where $\mathcal S_K(a)$ is the set of all ordered $K$-tuples containing arm $a$, and $\pi^\star(\pi)$ reorders $\pi$ (note that $\pi^\star(\pi)$ is not same as $ \pi^\star$) in descending score order.
Notice that $d_a$ is strictly positive and bounded as:
\begin{equation}
d_a \in \left[ \frac{\prod_{j=1}^{K-1}(1-s_j)}{\prod_{j=1}^{K-1}(1-s_{L-j+1})}, 1 \right].
\end{equation}

Applying \eqref{eq:da_} to compare the survival probabilities under the sorted shortlist $\pi_t^*$ versus the played shortlist $\pi_t$ yields:
\begin{align}
\label{eq:da_sub_detailed}
\mathbf{1}\{a \in \pi_t\} \prod_{j=1}^{(\pi_t^*)^{-1}(a)-1} (1 - s_{\pi_t^*(j)}) &\ge d_a \cdot \mathbf{1}\{a \in \pi_t\} \prod_{j=1}^{\pi_t^{-1}(a)-1} (1 - s_{\pi_t(j)}) \notag \\
&= d_a O_t(a).
\end{align}

Substituting \eqref{eq:da_sub_detailed} back into \eqref{eq:strategy_intermediate_step2} establishes the single-step bound:
\begin{equation}
\label{eq:single_step_final}
R_t \ge R_{\pi^*} - R_{\pi_t^*} \ge \tau' \sum_{a=K+1}^L d_a \Delta_{a,K}^s O_t(a).
\end{equation}

Finally, summing \eqref{eq:single_step_final} over time steps $t = 1, \dots, T$, taking expectations, and using $N_a(T) \triangleq \sum_{t=1}^T O_t(a)$ yields:
\begin{equation}
\mathbb{E}[R(T)] = \sum_{t=1}^T \mathbb{E}[R_t] \ge \tau' \sum_{a=K+1}^L d_a \Delta_{a,K}^s \mathbb{E}\left[ \sum_{t=1}^T O_t(a) \right] = \tau' \sum_{a=K+1}^L d_a \Delta_{a,K}^s \mathbb{E}[N_a(T)],
\end{equation}
which completes the proof of \eqref{eq:finite_sample_bound}.
\end{proof}

\begin{theorem}[Instance-Dependent Regret Lower Bound]
\label{thm:main_lower_bound}
Under Case 1, any uniformly consistent policy satisfies
\begin{equation}
\liminf_{T\to\infty} \frac{\mathbb E[\mathcal R(T)]}{\log T} \ge \tau' \sum_{a=K+1}^{L} \frac{d_a \, \Delta_{a,K}^s}{d(s_a \parallel s_K)},
\end{equation}
yielding an asymptotic lower bound of $\Omega\big((L-K)\log T\big)$.
\end{theorem}
\begin{proof}[Proof of Theorem \ref{thm:score_lower_bounds}, and \ref{thm:main_lower_bound}(case 1)]
Fix a suboptimal item $a \in \{K+1, \dots, L\}$. To lower bound the expected number of observations $\mathbb{E}_{s}[N_a(T)]$, we construct an alternative instance $s'$ that is identical to $s$ for all items $j \neq a$, but modifies item $a$ such that it becomes optimal in $s'$. Specifically, we set $s'_a = s_K + \varepsilon > s_K$ for a small $\varepsilon > 0$, while keeping $s'_j = s_j$ for all $j \neq a$.

To apply Lemma~\ref{lemma:cascaded_divergence_decomposition}, we map our parameter notation directly to the instance parametrization $(u, v)$ in the lemma by identifying $u_j = s_j = \mathbb{P}(Y_j = 1)$ and $v_j = 1$ for all $j \in [L]$. Under the alternative instance, $u'_a = s'_a$, $u'_j = u_j$ for $j \neq a$, and $v'_j = 1$ for all $j \in [L]$.

By Lemma~\ref{lemma:cascaded_divergence_decomposition}, the KL-divergence between the probability measures $\mathcal{P}_s$ and $\mathcal{P}_{s'}$ induced by policy $\mathcal{P}$ over horizon $T$ simplifies as:
\begin{align}
D\left(\mathcal{P}_{s} \parallel \mathcal{P}_{s'}\right) &= \sum_{j=1}^{L} \mathbb{E}_{s}\left[N_j(T)\right] \Big[ D(u_j \parallel u'_j) + u_j D(v_j \parallel v'_j) \Big] \nonumber \\
&= \sum_{j=1}^{L} \mathbb{E}_{s}\left[N_j(T)\right] \Big[ D(s_j \parallel s'_j) + s_j D(1 \parallel 1) \Big] \nonumber \\
&= \mathbb{E}_{s}\left[N_a(T)\right] D(s_a \parallel s'_a),
\end{align}
since $s_j = s'_j$ for all $j \neq a$, and $D(1 \parallel 1) = 0$.

By applying the change-of-measure inequality for consistent policies \cite{kaufmann2016complexity, garivier2011kl} to the alternative instance $s'$, we obtain:
\begin{equation}
\liminf_{T \to \infty} \frac{D\left(\mathcal{P}_{s} \parallel \mathcal{P}_{s'}\right)}{\log T} \ge 1.
\end{equation}

Substituting $D\left(\mathcal{P}_{s} \parallel \mathcal{P}_{s'}\right) = \mathbb{E}_{s}\left[N_a(T)\right] D(s_a \parallel s'_a)$ yields:
\begin{equation}
\label{eq:Na_lower_bound_case1}
\liminf_{T \to \infty} \frac{\mathbb{E}_{s}\left[N_a(T)\right]}{\log T} \ge \frac{1}{D(s_a \parallel s'_a)}.
\end{equation}

Now, applying the observation regret bound from Lemma~\ref{lem:observation_bound}:
\begin{equation}
\mathbb{E}_{s}[R(T)] \ge \tau' \sum_{a=K+1}^L d_a \Delta_{a,K}^s \mathbb{E}_{s}[N_a(T)].
\end{equation}

Dividing both sides by $\log T$, taking the limit inferior as $T \to \infty$, and substituting inequality~\eqref{eq:Na_lower_bound_case1} for each suboptimal arm $a \in \{K+1, \dots, L\}$ yields:
\begin{align}
\liminf_{T \to \infty} \frac{\mathbb{E}_{s}[R(T)]}{\log T} &\ge \tau' \sum_{a=K+1}^L d_a \Delta_{a,K}^s \left( \liminf_{T \to \infty} \frac{\mathbb{E}_{s}[N_a(T)]}{\log T} \right) \nonumber \\
&\ge \tau' \sum_{a=K+1}^L \frac{d_a \Delta_{a,K}^s}{D(s_a \parallel s'_a)}.
\end{align}

Finally, taking $\varepsilon \to 0$ gives $s'_a \to s_K$, which replaces $D(s_a \parallel s'_a)$ with the Bernoulli KL-divergence $d(s_a \parallel s_K)$, completing the proof.
\end{proof}

\section{Lower Bounds in Shared channel regime}
\label{lb:lower_bound}
In this section, Theorem~\ref{thm:score_lower_bounds} (specific to shared channel regime) is restated with complete technical details in Theorem~\ref{thm:cascaded_mab_lower_bound}; both theorems present the underlying score-based regret lower bound.
 we derive the asymptotic regret lower bound specialized to the shared-channel regime of Remark~\ref{rem:shared-channel} ($\epsilon_{0,i} = \epsilon_0$, $\epsilon_{1,i} = \epsilon_1$ for all $i \in [L]$, with $\epsilon_0 + \epsilon_1 < 1$). Let $u_i \triangleq \mathbb{P}(Y_i = 1)$ and $v_i \triangleq \mathbb{P}(X_i = 1 \mid Y_i = 1)$ denote the underlying parameters for arm $i \in \{1, \dots, L\}$, summarized by the parameter vectors $u = (u_1, u_2, \dots, u_L)$ and $v = (v_1, v_2, \dots, v_L)$. Under this setting, $ u_i = 1 - g_i$ and $v_i = r_i / (1 - g_i)$, so $u_i$, $v_i$, and $\mu_i$ all serve as valid score representations yielding identical arm rankings, as discussed in Remark~\ref{rem:shared-channel}.

Re-index arms such that $\mu_1 > \dots > \mu_L$. For any suboptimal arm $a > K$, define the gap $\Delta_{a,K}^\mu \triangleq \mu_K - \mu_a$. Under this shared-channel regime, we define:
\begin{equation}
\label{eq:tau_bar_def}
\begin{aligned}
\bar{\tau} \triangleq \min \Bigg\{ &\left(\frac{1-\varepsilon_1}{1-\varepsilon_0-\varepsilon_1}\right) \left[\tau + (1-K\tau) \prod_{k=1}^{K}(1-u_k)\right], \\
&\qquad (1-\varepsilon_1)(1-K\tau) \Bigg\}.
\end{aligned}
\end{equation}

\begin{lemma}[Cumulative Regret Lower Bound, Shared Channel]
\label{lem:observation_bound_shared}
Under the setup and definitions of Theorem~\ref{thm:score-characterization}, the expected cumulative regret satisfies
\begin{equation}
\label{eq:cumulative_regret_bound_shared}
\mathbb E\bigl[R(T)\bigr]\;\ge\;\bar{\tau}\sum_{a=K+1}^{L}d_a\,\Delta_a^\mu\,\mathbb E\bigl[N_a(T)\bigr],
\qquad \Delta_a^\mu\triangleq \mu_K - \mu_a,
\end{equation}
where $N_a(T)\triangleq\sum_{t=1}^{T}O_t(a)$ is the total observation count of arm $a$ under the executed slates, and the constants $\bar{\tau} > 0$ and $d_a > 0$ are defined as in Theorem~\ref{thm:cascaded_mab_lower_bound}.
\end{lemma}

\begin{proof}
Let $\pi_t^\star$ denote the ideal descending sort of the arms selected in $\pi_t$ according to
their mean parameters $\mu_i$, so that $\mu_{\pi_t^\star(1)}\ge\dots\ge\mu_{\pi_t^\star(K)}$. Since
$\pi_t^\star$ orders the same set of arms according to the score $\mu_i$, we have
$R_{\pi_t^\star}\ge R_{\pi_t}$. Adding and subtracting $R_{\pi_t^\star}$ yields the single-step
reduction
\begin{equation}
\label{eq:reduction_shared}
R_t = R_{\pi^\star}-R_{\pi_t}
=\underbrace{\bigl(R_{\pi^\star}-R_{\pi_t^\star}\bigr)}_{\ge 0}
+\underbrace{\bigl(R_{\pi_t^\star}-R_{\pi_t}\bigr)}_{\ge 0}
\;\ge\; R_{\pi^\star}-R_{\pi_t^\star}.
\end{equation}

Applying Lemma~\ref{lem:regret_decomp} to the sorted slate $\pi_t^\star$ decomposes the single-step
regret into position-wise components,
\begin{equation}
\label{eq:sorted_decomp_shared}
R_{\pi^\star}-R_{\pi_t^\star}=\sum_{m=1}^{K}S_{m,\pi_t^\star}\,\Delta_m\bigl(\pi_t^\star(m)\bigr),
\qquad
S_{m,\pi_t^\star}=\prod_{j=1}^{m-1}\bigl(1-u_{\pi_t^\star(j)}\bigr),
\end{equation}
with $S_{1,\pi_t^\star}=1$. Since
$r_{\pi^\star(m)}-r_a=(1-\varepsilon_1)(\mu_m^\star-\mu_a)$ and
$g_{\pi^\star(m)}-g_a=-(1-\varepsilon_0-\varepsilon_1)(\mu_m-\mu_a)$, the single-position gap
factors directly as
\begin{equation}
\label{eq:delta_m_shared}
\Delta_m(a)=(\mu_m-\mu_a)\Bigl[(1-\varepsilon_1)(1-m\tau)-R_{\pi^\star,m+1}\Bigr].
\end{equation}

For $m<K$, we bound the tail reward $R_{\pi^\star,m+1}$ using
$u_j-\varepsilon_0=(1-\varepsilon_0-\varepsilon_1)\mu_j\le u_j$ and
$1-j\tau\le 1-(m+1)\tau$ for all $j\ge m+1$:
\begin{align}
R_{\pi^\star,m+1}
&=(1-\varepsilon_1)\sum_{j=m+1}^{K}(1-j\tau)\,\mu_j\prod_{l=m+1}^{j-1}\bigl(1-u_l\bigr)\notag\\
&=\left(\frac{1-\varepsilon_1}{1-\varepsilon_0-\varepsilon_1}\right)
\sum_{j=m+1}^{K}(1-j\tau)\bigl(u_j-\varepsilon_0\bigr)\prod_{l=m+1}^{j-1}\bigl(1-u_l\bigr)\notag\\
&\le\left(\frac{1-\varepsilon_1}{1-\varepsilon_0-\varepsilon_1}\right)\bigl(1-(m+1)\tau\bigr)
\left(1-\prod_{l=m+1}^{K}\bigl(1-u_l\bigr)\right).
\end{align}
Substituting this upper bound into \eqref{eq:delta_m_shared} gives, for every $m<K$,
\begin{align}
(1-\varepsilon_1)(1-m\tau)-R_{\pi^\star,m+1}
&\ge\left(\frac{1-\varepsilon_1}{1-\varepsilon_0-\varepsilon_1}\right)
\left[\tau+\bigl(1-(m+1)\tau\bigr)\prod_{l=m+1}^{K}\bigl(1-u_l\bigr)\right]\notag\\
&\ge\left(\frac{1-\varepsilon_1}{1-\varepsilon_0-\varepsilon_1}\right)
\left[\tau+(1-K\tau)\prod_{k=1}^{K}\bigl(1-u_k\bigr)\right],
\end{align}
while for $m=K$ we have $R_{\pi^\star,K+1}=0$, so the bracket in \eqref{eq:delta_m_shared} equals $(1-\varepsilon_1)(1-K\tau)$. Taking the minimum of these two expressions gives exactly $\bar{\tau}$:
\begin{equation}
\label{eq:tau_bar_shared_explicit}
\bar{\tau} \;\triangleq\; \min\left\{
\left(\frac{1-\varepsilon_1}{1-\varepsilon_0-\varepsilon_1}\right)
\left[\tau+(1-K\tau)\prod_{k=1}^{K}\bigl(1-u_k\bigr)\right],
\;(1-\varepsilon_1)(1-K\tau)\right\} > 0,
\end{equation}
and yields the uniform gap bound
\begin{equation}
\label{eq:delta_ge_taubar}
\Delta_m(a)\;\ge\;\bar{\tau}\,(\mu_m-\mu_a)
\qquad\text{whenever }\mu_m\ge\mu_a.
\end{equation}

Because $\pi_t^\star$ is ordered in descending order of $\mu$, we have
$\mu_m\ge\mu_{\pi_t^\star(m)}$ for all $m\in[K]$, so every summand in
\eqref{eq:sorted_decomp_shared} is nonnegative. Discarding the positions occupied by optimal arms
and applying $\mu_m-\mu_a\ge\mu_K-\mu_a$ for $m\le K$,
\begin{align}
R_{\pi^\star}-R_{\pi_t^\star}
&\ge\bar{\tau}\sum_{m=1}^{K}S_{m,\pi_t^\star}\sum_{a=K+1}^{L}
\mathbb I\{\pi_t^\star(m)=a\}\,(\mu_m-\mu_a)\notag\\
&\ge\bar{\tau}\sum_{a=K+1}^{L}\bigl(\mu_K-\mu_a\bigr)\,
\mathbb I\{a\in\pi_t\}\prod_{j=1}^{(\pi_t^\star)^{-1}(a)-1}\bigl(1-u_{\pi_t^\star(j)}\bigr).
\end{align}

Finally, relating the observation probabilities under the sorted slate $\pi_t^\star$ back to the executed slate $\pi_t$ through the position-substitution constant
\begin{equation}
\label{eq:da__}
d_a \triangleq \min_{\pi \in \mathcal S_K(a)} \frac{\prod_{j=1}^{(\pi^\star(\pi))^{-1}(a)-1}(1 - u_{\pi^\star(\pi)(j)})}{\prod_{j=1}^{\pi^{-1}(a)-1}(1 - u_{\pi(j)})},
\end{equation}
where $\mathcal S_K(a)$ is the set of all ordered $K$-tuples containing arm $a$, and $\pi^\star(\pi)$ reorders $\pi$ (note that $\pi^\star(\pi)$ is not same as $ \pi^\star$) in descending score order.
we obtain
\begin{equation}
\mathbb I\{a\in\pi_t\}\prod_{j=1}^{(\pi_t^\star)^{-1}(a)-1}\bigl(1-u_{\pi_t^\star(j)}\bigr)
\;\ge\;d_a\,\mathbb I\{a\in\pi_t\}\prod_{j=1}^{\pi_t^{-1}(a)-1}\bigl(1-u_{\pi_t(j)}\bigr)
=d_a\,O_t(a).
\end{equation}
Combining this with the single-step bound, summing over $t=1,\dots,T$, taking expectations, and substituting $N_a(T)=\sum_{t=1}^{T}O_t(a)$ completes the proof.
\end{proof}
\begin{theorem}[Regret Lower Bound Under a Shared Channel]
\label{thm:cascaded_mab_lower_bound}
Suppose $\varepsilon_{0,i}=\varepsilon_0$ and $\varepsilon_{1,i}=\varepsilon_1$ for every arm $i\in[L]$ with $\varepsilon_0+\varepsilon_1<1$, Any uniformly consistent policy $\mathcal{P}$ satisfies:
\begin{align}
\label{eq:asymptotic_lower_bound}
\liminf_{T \to \infty} \frac{\mathbb{E}_{(u,v)}[\mathcal R(T)]}{\log T} &\ge \bar{\tau} \sum_{a=K+1}^L d_a \Delta_{a,K}^\mu \times \nonumber \\
&\quad  \Big[ d(u_a \parallel u_K) + u_a \, d(v_a \parallel v_K) \Big]^{-1}.
\end{align}
where $d_a$ in equation~\ref{eq:da_def} with $u_i$ substituted for $s_i$ and $\bar{\tau}$ is given by \eqref{eq:tau_bar_def}.
\end{theorem}
\begin{proof}[Proof of Theorem \ref{thm:score_lower_bounds} (shared channel regime )and~\ref{thm:cascaded_mab_lower_bound}]
Fix a suboptimal item $a \in \{K+1, \dots, L\}$. To bound the expected number of observations $\mathbb{E}_{(u,v)}[N_a(T)]$, we construct an alternative instance $(u', v')$ that is identical to $(u, v)$ for all items $j \neq a$, but simultaneously modifies item $a$ such that it becomes optimal in $(u', v')$. Specifically, we set $u'_a = u_K + \delta_1 > u_K$ and $v'_a = v_K + \delta_2 > v_K$ for small $\delta_1, \delta_2 > 0$, while keeping $(u'_j, v'_j) = (u_j, v_j)$ for all $j \neq a$.

By Lemma~\ref{lemma:cascaded_divergence_decomposition}, the KL-divergence between the probability measures $\mathcal{P}_{(u,v)}$ and $\mathcal{P}_{(u',v')}$ induced by policy $\mathcal{P}$ over horizon $T$ simplifies to:
\begin{align}
D\left(\mathcal{P}_{(u,v)} \parallel \mathcal{P}_{(u',v')}\right) &= \sum_{j=1}^{L} \mathbb{E}_{(u,v)}\left[N_j(T)\right] \Big[ d(u_j \parallel u'_j) + u_j d(v_j \parallel v'_j) \Big] \nonumber \\
&= \mathbb{E}_{(u,v)}\left[N_a(T)\right] \Big[ d(u_a \parallel u'_a) + u_a d(v_a \parallel v'_a) \Big],
\end{align}
since $(u_j, v_j) = (u'_j, v'_j)$ for all $j \neq a$, causing all other terms in the summation to evaluate to zero.

By applying the change-of-measure inequality for consistent policies \cite{kaufmann2016complexity, garivier2011kl} to the alternative instance $(u', v')$, we directly obtain:
\begin{equation}
\liminf_{T \to \infty} \frac{D\left(\mathcal{P}_{(u,v)} \parallel \mathcal{P}_{(u',v')}\right)}{\log T} \ge 1.
\end{equation}

Substituting the divergence decomposition from Lemma~\ref{lemma:cascaded_divergence_decomposition} into the inequality yields:
\begin{equation}
\label{eq:Na_lower_bound}
\liminf_{T \to \infty} \frac{\mathbb{E}_{(u,v)}\left[N_a(T)\right]}{\log T} \ge \frac{1}{d(u_a \parallel u'_a) + u_a d(v_a \parallel v'_a)}.
\end{equation}

Now, applying the single-step observation regret bound from Lemma~\ref{lem:observation_bound_shared}:
\begin{equation}
\mathbb{E}_{(u,v)}[R(T)] \ge \bar{\tau} \sum_{a=K+1}^L d_a \Delta_{a,K}^\mu \mathbb{E}_{(u,v)}[N_a(T)].
\end{equation}

Dividing both sides by $\log T$, taking the limit inferior as $T \to \infty$, and substituting inequality~\eqref{eq:Na_lower_bound} for each suboptimal arm $a \in \{K+1, \dots, L\}$ yields:
\begin{align}
\liminf_{T \to \infty} \frac{\mathbb{E}_{(u,v)}[R(T)]}{\log T} &\ge \bar{\tau} \sum_{a=K+1}^L d_a \Delta_{a,K}^\mu \left( \liminf_{T \to \infty} \frac{\mathbb{E}_{(u,v)}[N_a(T)]}{\log T} \right) \nonumber \\
&\ge \bar{\tau} \sum_{a=K+1}^L \frac{d_a \Delta_{a,K}^\mu}{d(u_a \parallel u'_a) + u_a d(v_a \parallel v'_a)}.
\end{align}
Taking $\delta_1 \to 0$ and $\delta_2 \to 0$ gives $u'_a \to u_K$ and $v'_a \to v_K$, completing the proof of Theorem~\ref{thm:cascaded_mab_lower_bound}.
\end{proof}

\section{Lower Bound for the General Regime}
\label{sec:general-lb}

Beyond the case 1 and shared-channel setting, we consider the instance-dependent regret lower bound for the general no-score cascading model. Let $R_{(u, v)}(\pi)$ denote the expected reward of shortlist $\pi$ according to \eqref{eq:reward_rg} with $r_i = u_i v_i$ and $g_i = 1 - u_i$.

The KL-divergence of shortlist $\pi$, denoted $D_\pi((u, v) \parallel (u', v'))$, is established in Lemma~\ref{for_perm} (detailed in Appendix~\ref{llowerb}):
\begin{equation}
D_\pi((u, v) \parallel (u', v')) = \sum_{i=1}^{K} \left[ \prod_{s=1}^{i-1} (1 - u_{\pi(s)}) \right] \Big[ d(u_{\pi(i)} \parallel u'_{\pi(i)}) + u_{\pi(i)} d(v_{\pi(i)} \parallel v'_{\pi(i)}) \Big].
\end{equation}

We define the set of alternative parameters $B(u, v)$ as those parameter configurations $(u', v')$ where a suboptimal shortlist becomes strictly superior to $\pi^*$, while remaining indistinguishable under $\pi^*$:
\begin{equation}
B(u, v) \triangleq \Big\{ (u', v') : D_{\pi^*}\big((u, v) \parallel (u', v')\big) = 0 \quad \text{and} \quad \exists \, \pi \neq \pi^* \text{ s.t. } R_{(u', v')}(\pi) > R_{(u, v)}(\pi^*) \Big\}.
\end{equation}

By the lower bound framework of \cite{lai1985asymptotically}, any uniformly consistent policy must satisfy:
\begin{equation}
\liminf_{T \to \infty} \frac{\mathbb{E}_{(u,v)}[\mathcal R(T)]}{\log T} \ge c(u, v),
\end{equation}
where $c(u, v)$ is the optimal value of the following linear optimization problem:
\begin{align}
c(u, v) = \inf_{\{c_\pi \ge 0\}} \quad & \sum_{\pi \neq \pi^*} c_\pi \big(R_{(u, v)}(\pi^*) - R_{(u, v)}(\pi)\big) \label{eq:gen_lb_lp} \\
\text{s.t.} \quad & \sum_{\pi \neq \pi^*} c_\pi D_\pi\big((u, v) \parallel (u', v')\big) \ge 1, \quad \forall \, (u', v') \in B(u, v). \nonumber
\end{align}

Obtaining an explicit lower bound in the general regime requires solving the optimization problem in \eqref{eq:gen_lb_lp} over all valid permutations $\pi$. Due to the non-linear coupling between $u_i$ and $v_i$ across arbitrary noise configurations, expressing the regret decomposition directly in terms of $\mathbb{E}[N_a(T)]$ as in \eqref{eq:cumulative_regret_bound_shared} is difficult, rendering an exact analytical solution to \eqref{eq:gen_lb_lp} for general parameters $(u, v)$ challenging to establish.

\section{Baseline Algorithms for Experiments}
\label{exps}
In this section, we provide detailed descriptions of the baseline algorithms used for performance comparison against our proposed methods.
\subsection{Naive Shortlist Bandit (Perm-UCB)}
A naive approach reduces the problem to a standard multi-armed bandit with action space $\mathcal{P}_{L,K}$ of size $|\mathcal{P}_{L,K}| = \frac{L!}{(L-K)!}$. Algorithm~\ref{perm_algo} applies UCB1 \cite{auer2002finite} directly over these meta-arms. We define $w_t \in [0, 1]$ as the instantaneous reward obtained at round $t$ by playing shortlist $\pi_t$ (with expected reward $w_\pi \triangleq \mathbb{E}[w_t \mid \pi_t = \pi]$), and $\hat{w}_\pi(t) \triangleq \frac{\sum_{\tau=1}^{t-1} w_\tau \mathbb{I}(\pi_\tau = \pi)}{\max(1, N_\pi(t))}$ as its empirical mean reward prior to round $t$.

\begin{algorithm}[H]
\caption{Naive UCB over Shortlists (Perm-UCB)} \label{perm_algo}
\begin{algorithmic}[1]
\State \textbf{Input:} Candidate arms $L$, shortlist length $K \le L$, action space $\mathcal{P}_{L,K}$
\State Initialize $\hat{w}_\pi(1)=0$ and $N_\pi(1)=0$ for all $\pi \in \mathcal{P}_{L,K}$
\For{$t = 1$ to $T$}
    \State Select $\pi_t = \arg\max_{\pi \in \mathcal{P}_{L,K}} \left( \hat{w}_\pi(t) + \sqrt{\frac{2 \log t}{\max(1, N_\pi(t))}} \right)$
    \State Play shortlist $\pi_t$, observe feedback $r_t$, and receive instantaneous reward $w_t$
    \State Update $N_{\pi_t}(t+1) = N_{\pi_t}(t) + 1$ and $\hat{w}_{\pi_t}(t+1) = \frac{N_{\pi_t}(t)\hat{w}_{\pi_t}(t) + w_t}{N_{\pi_t}(t+1)}$
    \State Carry forward $N_\pi(t+1) = N_\pi(t)$ and $\hat{w}_\pi(t+1) = \hat{w}_\pi(t)$ for all $\pi \neq \pi_t$
\EndFor
\end{algorithmic}
\end{algorithm}

Treating shortlists independently yields a gap-independent regret bound of $\mathcal{O}\left(\sqrt{\frac{L!}{(L-K)!} T \log T}\right)$, which scales poorly for large $L$ and $K$.

\subsection{Cascaded Correlated UCB (Corr-UCB)}
\label{corr}
To exploit overlap across shortlists, we adapt the C-UCB framework \cite{gupta2021multi}. Let $\mathcal{Z}_t = \{(X_i, Y_i)\}_{i=1}^{L}$ denote the latent realization at round $t$. Executing $\pi_t$ yields partial feedback $r_t = r(\pi_t, \mathcal{Z}_t)$. The consistency set of compatible latent states is:
\[
\Omega(r_t \mid \pi_t) = \{ \mathcal{Z} : r(\pi_t, \mathcal{Z}) = r_t \}.
\]
For any target shortlist $\pi'$, the optimistic conditional upper bound is $\gamma_{\pi', \pi_t, r_t} \triangleq \max_{\mathcal{Z} \in \Omega(r_t \mid \pi_t)} w(\pi', \mathcal{Z})$. Computationally, $\gamma_{\pi, \pi_t, r_t}$ is evaluated by scanning $\pi = (a_1, \dots, a_K)$ left-to-right until the first arm $a_{i^*}$ that is either unobserved in $r_t$ (assume $Y=1, X=1$) or terminates with $Y=1$, yielding $1 - i^*\tau$ (or $0$ if $X=0$).

\begin{example}[Evaluation of Optimistic Bound $\gamma_{\pi, \pi_t, r_t}$]
Consider a candidate pool $L=4$ with items $\{1, 2, 3, 4\}$, shortlist size $K=3$, and lookahead penalty $\tau = 0.1$. Suppose at round $t$, the executed shortlist is $\pi_t = (1, 2, 3)$ and the observed feedback vector is $r_t = \{(1, Y_1=0), (2, Y_2=0), (3, Y_3=1, X_3=0)\}$. This implies items $1$ and $2$ were skipped, while item $3$ was evaluated ($Y_3=1$) but yielded no reward ($X_3=0$). Item $4$ was unobserved.
\begin{itemize}
    \item \textbf{Target shortlist $\pi_A = (1, 2, 4)$:} Items $1$ and $2$ are known to have $Y_1=Y_2=0$. Item $4$ is unobserved, so the optimistic state sets $Y_4=1, X_4=1$. Scanning $\pi_A$ skips positions 1 and 2, evaluating item $4$ at position $i^*=3$, yielding $\gamma_{\pi_A, \pi_t, r_t} = 1 - 3(0.1) = 0.7$.
    \item \textbf{Target shortlist $\pi_B = (4, 1, 2)$:} Item $4$ appears at position $1$. Under the optimistic state ($Y_4=1, X_4=1$), the cascade terminates immediately at $i^*=1$, yielding $\gamma_{\pi_B, \pi_t, r_t} = 1 - 1(0.1) = 0.9$.
    \item \textbf{Target shortlist $\pi_C = (1, 2, 3)$:} Items $1$ and $2$ are skipped, and item $3$ at position $i^*=3$ has $X_3=0$. Thus, $\gamma_{\pi_C, \pi_t, r_t} = 0$.
\end{itemize}
\end{example}

The expected upper bound is $\psi_{\pi, \pi'} \triangleq \mathbb{E}_{\pi'}[\gamma_{\pi, \pi', r_t}]$, where $w_\pi \le \psi_{\pi, \pi'}$.

\begin{algorithm}[H]
\caption{Cascaded Correlated Upper Confidence Bound (Corr-UCB)}
\label{alg:c_ucb_shortlist}
\begin{algorithmic}[1]
\State \textbf{Input:} Action space $\mathcal{P}_{L,K}$, lookahead cost $\tau$
\State Initialize $N_{\pi}(1) = 0$, $\hat{w}_{\pi}(1) = 0$, and $\hat{\psi}_{\pi,\pi'}(1) = 0$ for all $\pi, \pi' \in \mathcal{P}_{L,K}$
\For{$t=1,\ldots,T$}
    \State Compute $U_{\pi}(t) = \hat{w}_{\pi}(t) + \sqrt{\frac{2 \log t}{\max(1, N_{\pi}(t))}}$ for all $\pi \in \mathcal{P}_{L,K}$
    \State Define explored set $S_t = \{ \pi \in \mathcal{P}_{L,K} : N_{\pi}(t) \geq \frac{t-1}{|\mathcal{P}_{L,K}|} \}$ and empirical leader $\pi_{\text{emp}}(t) = \arg\max_{\pi \in S_t} \hat{w}_{\pi}(t)$
    \State Form empirical competitive set $A_t = \{ \pi \in \mathcal{P}_{L,K} : \min_{\ell \in S_t} \hat{\psi}_{\pi, \ell}(t) \geq \hat{w}_{\pi_{\text{emp}}}(t) \}$
    \State Select and execute shortlist $\pi_t = \arg\max_{\pi \in A_t \cup \{\pi_{\text{emp}}(t)\}} U_{\pi}(t)$
    \State Observe feedback $r_t$ and instantaneous reward $w_t$; update $N_{\pi_t}(t+1)$ and $\hat{w}_{\pi_t}(t+1)$
    \State Update bounds $\hat{\psi}_{\pi, \pi_t}(t+1) = \frac{1}{N_{\pi_t}(t+1)}\sum_{\tau=1}^{t} \mathbb{I}(\pi_\tau = \pi_t) \gamma_{\pi, \pi_\tau , r_\tau}$ for all $\pi \in \mathcal{P}_{L,K}$
\EndFor
\end{algorithmic}
\end{algorithm}

\subsection{Structural Bottleneck and Regret Analysis}
\label{corr1}

Following \citet{gupta2021multi}, an action $\pi$ is \textbf{non-competitive} if $\delta_{\pi, \pi^*} \triangleq w_{\pi^*} - \psi_{\pi, \pi^*} > 0$, and \textbf{competitive} if $\delta_{\pi, \pi^*} \le 0$. Standard C-UCB achieves $\mathcal{O}(C \log T)$ regret, where $C$ is the number of competitive actions. However, Theorem~\ref{competitive proof} demonstrates that correlated feedback provides no pruning power in cascade models.

\begin{theorem}\label{competitive proof}[Universal Competitiveness of Shortlists]
Let $\pi^* = (a_1, \ldots, a_K) \in \mathcal{P}_{L,K}$ denote the optimal shortlist. For each $m \in \{1, \ldots, K\}$, let $p_m \triangleq (\prod_{t=1}^{m-1} g_{a_t}) r_{a_m}$, so that $w_{\pi^*} = \sum_{m=1}^{K} p_m(1 - m\tau)$. For any shortlist $\pi \in \mathcal{P}_{L,K}$, let $j(\pi, m) \triangleq \min ( \{ k \in \{1, \dots, K\} : \pi_k \notin \{a_1, \ldots, a_{m-1}\} \} \cup \{K+1\} )$. Then, $\psi_{\pi, \pi^*} \ge w_{\pi^*}$.
\end{theorem}

\begin{proof}
Fix $\pi \in \mathcal{P}_{L,K}$. Conditioning on the stopping outcome of $\pi^*$ gives:
\[
\psi_{\pi, \pi^*} = \mathbb{E}\left[\gamma_{\pi,\pi^*,r}\right] \ge \sum_{m=1}^{K} \mathbb{E}\left[ \gamma_{\pi,\pi^*,r} \;\middle|\; Y_{a_1}=0, \dots, Y_{a_{m-1}}=0, X_{a_m}=1, Y_{a_m}=1 \right] p_m,
\]
where non-negative zero-reward terms are dropped. Under this event, the optimistic evaluation scans $\pi$ left-to-right and stops at position $j(\pi, m) \le K$ outside $F_m = \{a_1, \ldots, a_{m-1}\}$, yielding reward $1 - j(\pi, m)\tau$. 

Since $|F_m| = m-1$, the first $m$ positions of $\pi$ cannot all lie in $F_m$ by the pigeonhole principle. Hence, $j(\pi, m) \le m$, implying $1 - j(\pi, m)\tau \ge 1 - m\tau$. Weighting by $p_m \ge 0$ and summing over $m$ yields:
\[
\psi_{\pi, \pi^*} \ge \sum_{m=1}^{K} p_m \bigl(1 - j(\pi, m)\tau\bigr) \ge \sum_{m=1}^{K} p_m (1 - m\tau) = w_{\pi^*},
\]
completing the proof.
\end{proof}

\begin{remark}[Collapse of Pruning]
Theorem~\ref{competitive proof} establishes $\delta_{\pi, \pi^*} \le 0$ for all $\pi \in \mathcal{P}_{L,K}$. Thus, \textit{every} shortlist is competitive ($C = |\mathcal{P}_{L,K}| = \frac{L!}{(L-K)!}$). Direct mapping of cascade feedback onto correlated multi-armed bandits yields zero structural pruning, leading to an expanded regret bound of $\mathcal{O}\left(\frac{L!}{(L-K)!} \log T\right)$.
\end{remark}

\subsection{Score-based Evaluation}
\begin{algorithm}[H]
\label{score_algo}
\caption{Top-$K$-$P(X=1|Y=1)$-UCB: Indexing via $\frac{U_i^r(t)}{1 - U_i^g(t)}$}
\label{alg:top_k_ratio_ucb}
\begin{algorithmic}[1]
\State \textbf{Input:} Candidate arms $L$, shortlist size $K \le L$, small constant $\delta > 0$
\State Initialize $N_{i}(1) = 0$, $\hat{r}_i(1) = 0$, and $\hat{g}_i(1) = 0$ for all $i \in [L]$
\For{$t = 1$ to $T$}
    \State Compute $U_i^r(t)$ and $U_i^g(t)$ via \eqref{eq:ucb_rg} for all $i \in [L]$
    \State Form shortlist $\pi_t = (a_1, \dots, a_K)$ sorting top-$K$ indices $I_i(t) = \frac{U_i^r(t)}{1 - U_i^g(t)}$ such that $I_{a_1}(t) \ge \dots \ge I_{a_K}(t)$
    \State Play $\pi_t$, observe cascade feedback up to evaluated item $a_{m^*}$ (or end of list $K$), and update $N_i(t+1)$
    \State Update empirical estimates $\hat{r}_i(t+1)$ and $\hat{g}_i(t+1)$ via \eqref{eq:emp_rg} for observed arms $i$; carry forward unobserved arms
\EndFor
\end{algorithmic}
\end{algorithm}

\section{Additional Experimental Setup \& Supplementary Results}
\subsection{Synthetic supplementary:}
\label{app:extra_experiments}

\begin{figure*}[htbp]
    \centering
    \begin{subfigure}[b]{0.32\textwidth}
        \centering
        \includegraphics[width=\textwidth]{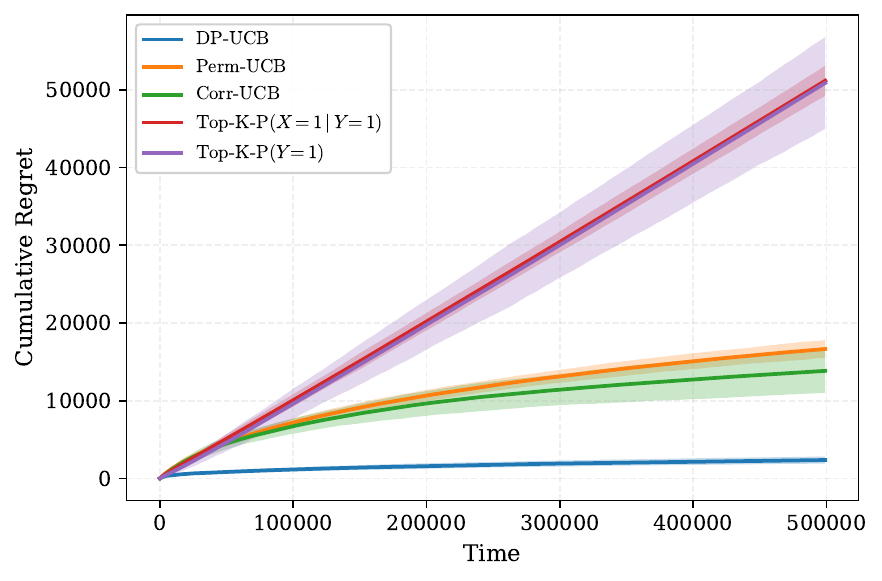}
        \caption{Case 2: $\tau = 0, K = 3 < L, \pi^* = (2, 4, 5)$}
        \label{fig:app_case2_tau0}
    \end{subfigure}
    \hfill
    \begin{subfigure}[b]{0.32\textwidth}
        \centering
        \includegraphics[width=\textwidth]{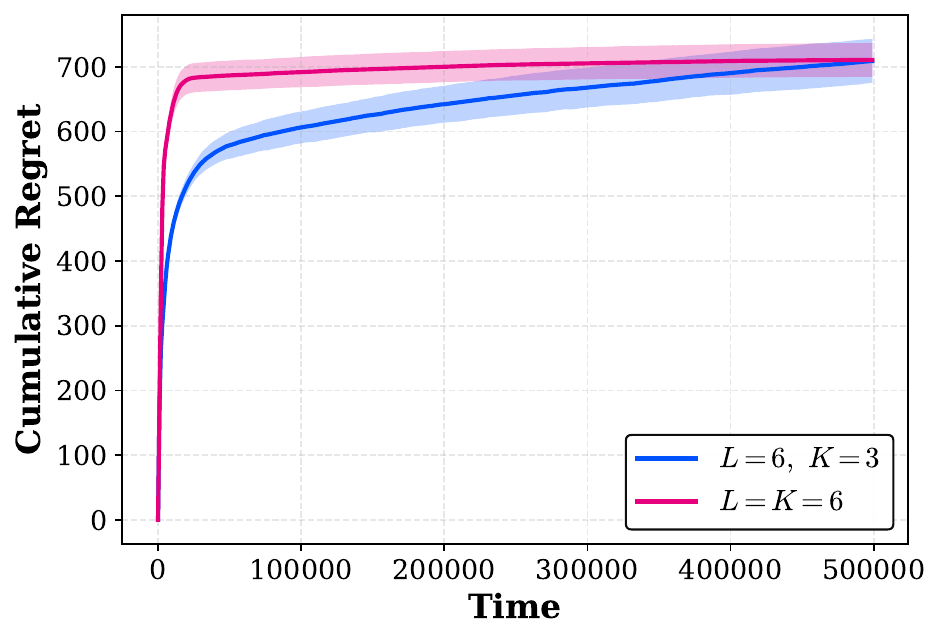}
        \caption{Case 2: \textsc{DP-UCB} ($K < L$ vs. $K = L$)}
        \label{fig:app_case2_scaling}
    \end{subfigure}
    \hfill
    \begin{subfigure}[b]{0.32\textwidth}
        \centering
        \includegraphics[width=\textwidth]{regret_comparison_aistats_commom_epsilon.pdf}
        \caption{Shared-Channel: $\tau = 0, \pi^* = (4, 2, 3)$}
        \label{fig:app_common_eps_tau0}
    \end{subfigure}

   \caption{cumulative regret on MovieLens-1M . \textbf{(a)} Zero-cost regime ($\tau=0, K=3$). \textbf{(b)} \textsc{DP-UCB} regret scaling ($K=3$ vs.\ $K=6$). \textbf{(c)} Shared-channel zero-cost setting ($\tau=0, \boldsymbol{\epsilon}_1=0.3\mathbf{1}_6, \boldsymbol{\epsilon}_0=0.2\mathbf{1}_6$).}
    \label{fig:appendix_excluded_results}
\end{figure*}

\subsection*{Experimental Parameters synthetic }
All experiments are conducted with $L=6$ arms over a horizon of $T=500{,}000$ rounds, averaged across 100 independent trials. Table~\ref{tab:exp_parameters} specifies the exact baseline true reward probabilities $\boldsymbol{\mu} \in (0,1)^L$, false-negative error rates $\boldsymbol{\epsilon}_1 \in [0,0.5)^L$, and false-positive error rates $\boldsymbol{\epsilon}_0 \in [0,0.5)^L$ corresponding to all sub-figures in the main text (Figure~\ref{fig:unified_empirical_results}) and this appendix (Figure~\ref{fig:appendix_excluded_results}).
Parameters are chosen arbitrarily.
\begin{table}[htbp]
\centering
\caption{Parameter configurations across all empirical evaluation regimes ($L=6$).}
\label{tab:exp_parameters}
\small
\begin{tabular}{@{}llcc@{}}
\toprule
\textbf{Regime} & \textbf{Vector} & \textbf{Parameter Values} & \textbf{Figure Reference} \\ \midrule
\textbf{Case 1} & $\boldsymbol{\mu}$ & $[0.5595, 0.9104, 0.2653, 0.9089, 0.3995, 0.4887]^\top$ & Fig.~\ref{fig:unified_empirical_results}(a)--(c) \\
(Score) & $\boldsymbol{\epsilon}_1$ & $[0.3987, 0.4253, 0.3167, 0.1038, 0.1080, 0.4177]^\top$ & \\
 & $\boldsymbol{\epsilon}_0$ & $[0.0000, 0.0000, 0.0000, 0.0000, 0.0000, 0.0000]^\top$ & \\ \midrule
\textbf{Case 2} & $\boldsymbol{\mu}$ & $[0.5595, 0.9104, 0.2653, 0.9089, 0.3995, 0.4887]^\top$ & Fig.~\ref{fig:unified_empirical_results}(d)--(e), \\
(General) & $\boldsymbol{\epsilon}_1$ & $[0.3987, 0.4253, 0.3167, 0.1038, 0.1080, 0.4177]^\top$ & Fig.~\ref{fig:appendix_excluded_results}(a)--(b) \\
 & $\boldsymbol{\epsilon}_0$ & $[0.1266, 0.1573, 0.4879, 0.2698, 0.4306, 0.3222]^\top$ & \\ \midrule
\textbf{Shared-} & $\boldsymbol{\mu}$ & $[0.6000, 0.8000, 0.7000, 0.9000, 0.4000, 0.5000]^\top$ & Fig.~\ref{fig:unified_empirical_results}(f), \\
\textbf{Channel} & $\boldsymbol{\epsilon}_1$ & $0.3 \cdot \mathbf{1}_6 = [0.3, 0.3, 0.3, 0.3, 0.3, 0.3]^\top$ & Fig.~\ref{fig:appendix_excluded_results}(c) \\
 & $\boldsymbol{\epsilon}_0$ & $0.2 \cdot \mathbf{1}_6 = [0.2, 0.2, 0.2, 0.2, 0.2, 0.2]^\top$ & \\ \bottomrule
\end{tabular}
\end{table}

\subsection*{Discussion of Supplementary Results}
Under zero position cost ($\tau=0$), non-zero false positives cause greedy score heuristics to select suboptimal orderings and incur linear regret, whereas \textsc{DP-UCB} recovers the optimal shortlist (Fig.~\ref{fig:app_case2_tau0}). Figure~\ref{fig:app_case2_scaling} isolates \textsc{DP-UCB} for clear visualization, highlighting its sub-linear regret when $K < L$ compared to significantly lower bounded regret when $K = L$. Finally, under uniform channel noise with $\tau=0$, \textsc{Top-K-P(Y=1)-UCB}, \textsc{Top-K-P(X=1|Y=1)-UCB}, and \textsc{DP-UCB} all three outperform the remaining two baselines (Fig.~\ref{fig:app_common_eps_tau0}).

\subsection{Supplementary MovieLens-1M Results \& Parameters}
\label{_app:movielens_details}
\subsection*{Dataset details}
We validate our framework using a recommendation setup derived from the MovieLens-1M benchmark dataset \cite{harper2015movielens}, which contains $1{,}000{,}000$ ratings across $3{,}883$ movies from $6{,}040$ users (on a scale of 1 to 5). Using this dataset, our goal is to recommend the best movie genres to a target user population.

To map each movie to a single category, we select one genre uniformly at random from the set of genres associated with that movie. We model each genre as an arm $i \in \{1, \dots, L\}$, selecting the top $L = 10$ most top rated genres as our candidate arm pool. Ratings within each genre are quantized into binary outcomes $X_i \in \{0, 1\}$ ($X_i = 1$ for ratings $\ge 4$, and $X_i = 0$ otherwise), and the mean 
$\mu_i = \mathbb{P}(X_i = 1)$ for arm $i$ is computed as the empirical average of these quantized ratings across all user interactions in that genre. Synthetic false-negative ($\boldsymbol{\epsilon}_1$) and false-positive ($\boldsymbol{\epsilon}_0$) noise rates are then introduced to model noisy feedback channels for each genre recommendation arm.

\subsection*{Additional experiments and Parameters}
Figure~\ref{fig:app_movielens_k5} presents the supplementary evaluation on the MovieLens-1M dataset corresponding to case 2 (no score) $K=5$.

\begin{figure}[htbp]
    \centering
    \includegraphics[width=0.48\textwidth]{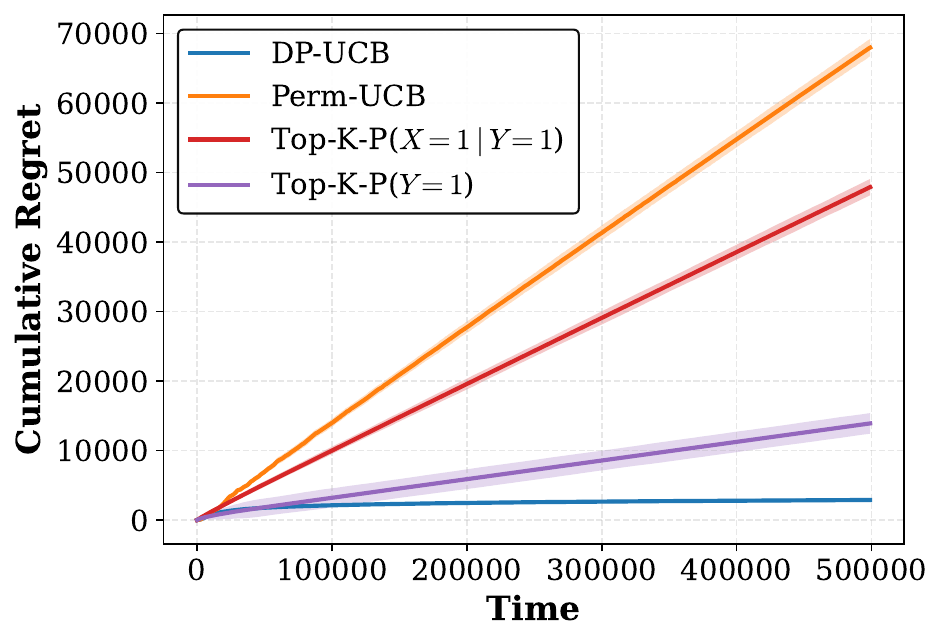}
    \caption{Empirical cumulative regret on MovieLens-1M under a larger cascade slate ($L=10, K=5, \tau=0.15, T=500{,}000$, 100 trials). Heuristic rankings based on $P(Y_i=1)$ or $P(X_i=1 \mid Y_i=1)$ suffer linear regret, whereas \textsc{DP-UCB} scales efficiently and maintains sub-linear regret. }
    \label{fig:app_movielens_k5}
\end{figure}

\paragraph{Discussion of Figure~\ref{fig:app_movielens_k5}.}
When expanding the slate size to $K=5$, \textsc{Corr-UCB} is omitted due to exponential computational complexity over $\binom{10}{5} \times 5!$ shortlists and the overhead of updating pseudo-rewards over the expanded search space. Even within this larger search space, \textsc{DP-UCB} scales efficiently, consistently recovering the optimal cascade order while maintaining sub-linear regret.

\begin{table}[]
\centering
\caption{MovieLens-1M parameter configurations ($L=10$). Note that $\boldsymbol{\epsilon}_0 = \mathbf{0}_{10}$ rest of the parameters being same is used for Figure~\ref{fig:movielens_results_combined}(a) (score-based regime).}
\label{tab:movielens_parameters}
\small
\begin{tabular}{@{}lccc@{}}
\toprule
\textbf{Genre (Arm $i$)} & \textbf{Mean ($\mu_i$)} & \textbf{False Negative ($\epsilon_{1,i}$)} & \textbf{False Positive ($\epsilon_{0,i}$)} \\ \midrule
Film-Noir   & $0.7250$ & $0.1289$ & $0.1780$ \\
Mystery     & $0.6560$ & $0.3979$ & $0.4692$ \\
Crime       & $0.6250$ & $0.4771$ & $0.0202$ \\
Documentary & $0.5930$ & $0.4290$ & $0.0170$ \\
Drama       & $0.5550$ & $0.1072$ & $0.0684$ \\
Musical     & $0.5330$ & $0.1154$ & $0.4110$ \\
War         & $0.5330$ & $0.4594$ & $0.4003$ \\
Adventure   & $0.5130$ & $0.3838$ & $0.4338$ \\
Western     & $0.5100$ & $0.1310$ & $0.0642$ \\
Animation   & $0.5040$ & $0.3960$ & $0.3558$ \\ \bottomrule
\end{tabular}
\end{table}

\paragraph{Parameters.}
The candidate pool consists of $L=10$ top genres selected from MovieLens-1M. Table~\ref{tab:movielens_parameters} details the empirical mean reward vector $\boldsymbol{\mu} \in (0,1)^{10}$, false-negative error rates $\boldsymbol{\epsilon}_1 \in [0,0.5)^{10}$, and false-positive error rates $\boldsymbol{\epsilon}_0 \in [0,0.5)^{10}$ for each arm. The arm reward means $\boldsymbol{\mu}$ reflect empirical performance on the benchmark dataset. In contrast, the noise probabilities $\boldsymbol{\epsilon}_0$ and $\boldsymbol{\epsilon}_1$ were selected arbitrarily to evaluate the algorithm across arbitrary observation channels.

\subsection{LLM experiments supplement}

\label{app:LLM_data}
To demonstrate the practical, real-world application of our proposed setup to automated LLM output routing via model cascading \cite{chen2023frugalgpt}, we evaluate it on a financial headlines sentiment benchmark \cite{sinha2021impact}. The dataset consists of $11{,}412$ human-annotated headlines scraped from financial news outlets and aggregators (2000--2019), labeled with gold-standard sentiment tags (\textsc{Up}, \textsc{Down}, \textsc{Neutral}, \textsc{None}) indicating the implied trajectory of gold prices.

In this setup, the input is a financial news headline, which is sent through a cascade of LLM generators (arms) $i \in \{1, \dots, L\}$ (where $L = 5$). Each arm $i$ is an autoregressive LLM generator—chosen from SmolLM2 \cite{allal2025smollm2}, Qwen2.5 \cite{qwen2025qwen25technicalreport}, Phi-2 \cite{abdin2024phi3}, or StableLM-2 \cite{bellagente2024stablelm}—paired with an LLM judge acting as a noisy observation channel, selected from BERT \cite{devlin2018bert}, RoBERTa \cite{liu2019roberta}, ELECTRA \cite{clark2020electra}, ALBERT \cite{lan2019albert}, or DeBERTa-v3 \cite{he2021debertav3}.

In each round, a financial headline is sampled uniformly at random. When queried, the LLM generator (arm $i$) predicts one of the sentiment labels, which we map to a binary state $X_i \in \{0, 1\}$, where $X_i = 1$ if the predicted label matches the gold-standard target and $0$ otherwise. The paired LLM judge evaluates the prediction and produces a binary confidence signal $Y_i \in \{0, 1\}$. Here, $Y_i = 0$ means the LLM judge indicates that the LLM generator is incorrect or could not satisfy the learner's requirement.

The core idea behind this setup is that observation noise originates directly from judge imperfectness: sometimes the LLM generator genuinely satisfies the learner's requirement ($X_i = 1$), but the LLM judge is inaccurate and rejects it ($Y_i = 0$), whereas other times the LLM generator fails to satisfy the requirement ($X_i = 0$), but the LLM judge misjudges and accepts it ($Y_i = 1$).

The learner can only observe the output of the LLM judge ($Y_i$). If $Y_i = 0$, the request cascades to arm $i + 1$. Whenever the LLM judge outputs $1$ ($Y_i = 1$), the cascading halts at position $i$, revealing $X_i$, and the learner receives a net reward of $(1 - i\tau)X_i$, where $\tau$ represents the per-step cost penalty. 

\paragraph{Parameter Extraction \& Ground Truth Verification.}
The empirical generator accuracies $\boldsymbol{\mu}$, false-negative rates $\boldsymbol{\epsilon}_1$, and false-positive rates $\boldsymbol{\epsilon}_0$ computed as empirical averages over the entire dataset are detailed in Table~\ref{tab:llm_cascade_params}. Using these parameters, the theoretical optimal sequence $\pi^*$ was computed via Eq.~\eqref{eq:reward_rg}. To cross-verify this ground truth empirically, we evaluated every possible candidate shortlist across $100{,}000$ independent trials. In each trial, a headline was sampled uniformly at random from the dataset to run the cascade experiment, and the shortlist yielding the highest empirical average reward was identified as the oracle strategy, matching the theoretical optimum perfectly.
\end{document}